\documentclass{article}

\usepackage[utf8]{inputenc} 
\usepackage[T1]{fontenc}   
\usepackage{lmodern}
\usepackage{url}            %
\usepackage{booktabs}       %
\usepackage{amsfonts}       %
\usepackage{nicefrac}       %
\usepackage{microtype}      %
\usepackage{xcolor}         %
\usepackage{mathtools}

\usepackage[top=1.3in, left=1.2in, bottom=1.3in, right=1.2in]{geometry}

\usepackage{xcolor}
\definecolor{CColor}{rgb}{0.01,0.31,0.59}
\definecolor{GGray}{rgb}{0.80,0.90,1}
\definecolor{Shady}{rgb}{0.9,0.9,0.9}
\definecolor{kaistblue}{RGB}{20,135,200}
\definecolor{kaistdarkblue}{RGB}{0,65,145}
\definecolor{urbanablue}{RGB}{19,41,75}
\definecolor{urbanaorange}{RGB}{232,74,39}
\definecolor{drp}{rgb}{0.53,0.15,0.34}
\usepackage[plainpages=false,pdfpagelabels,colorlinks=true,linkcolor=CColor,citecolor=CColor]{hyperref}
 
\usepackage{graphicx}
\usepackage{amsmath}
\usepackage{amssymb}
\usepackage{mathtools}
\usepackage{amsthm}
\usepackage[capitalize,noabbrev]{cleveref}
\usepackage[textsize=tiny]{todonotes}
\usepackage{booktabs}
\usepackage{subcaption}

\usepackage{cleveref} %
\usepackage{subcaption} %
\usepackage{mathtools}
\usepackage{amsmath}
\usepackage{stmaryrd}
\usepackage{etoolbox} %
\RequirePackage{paralist}
\usepackage{enumitem}
\usepackage{multirow} %
\usepackage{array}
\usepackage{graphicx}
\usepackage{etoc}
\usepackage{listings}
\usepackage{eqparbox}
\usepackage{bbm}

\AtEndPreamble{%
  \newtheorem{remark}[theorem]{Remark}
  \newtheorem{assumption}{Assumption}
  \newtheorem{property}{Property}
  \theoremstyle{acmplain}
  \newenvironment{definition*}
                 {\pushQED{\qed}\definition}
                 {\popQED\enddefinition}  
  \newenvironment{example*}
                 {\pushQED{\qed}\example}
                 {\popQED\endremark}  
  \newenvironment{remark*}
                 {\pushQED{\qed}\remark}
                 {\popQED\endremark}  
  \newenvironment{assumption*}
                 {\pushQED{\qed}\assumption}
                 {\popQED\endremark}  
  \newenvironment{property*}
                 {\pushQED{\qed}\property}
                 {\popQED\endremark}  
}

\crefname{assumption}{Assumption}{Assumptions}
\crefname{figure}{Fig{.}}{Figs{.}}%
\crefname{table}{Table}{Tables}
\crefname{definition}{Definition}{Definitions}
\crefname{theorem}{Theorem}{Theorems}
\crefname{lemma}{Lemma}{Lemmas}
\crefname{proposition}{Proposition}{Propositions}
\crefname{corollary}{Corollary}{Corollaries}
\crefname{problem}{Problem}{Problems}
\crefname{example}{Example}{Examples}
\crefname{fact}{Fact}{Facts}
\crefname{conjecture}{Conjecture}{Conjectures}
\crefname{remark}{Remark}{Remarks}
\crefname{condition}{Condition}{Conditions}
\crefname{requirement}{Requirement}{Requirements}
\crefname{enumi}{}{}
\crefname{equation}{Eq{.}}{Eqs{.}}
\crefname{section}{Section}{Sections}

\newcommand{\mathboldcommand}[1]{\mathbb{#1}}

\newcommand{\bbF}{\mathboldcommand{F}}

\newcommand{\bbN}{\mathboldcommand{N}}

\newcommand{\bbR}{\mathboldcommand{R}}

\newcommand{\bbZ}{\mathboldcommand{Z}}

\newcommand{\bfa}{\mathbf{a}}
\newcommand{\bfb}{\mathbf{b}}
\newcommand{\bfc}{\mathbf{c}}

\newcommand{\bfw}{\mathbf{w}}
\newcommand{\bfx}{\mathbf{x}}
\newcommand{\bfy}{\mathbf{y}}
\newcommand{\bfz}{\mathbf{z}}

\usepackage{euscript}
\newcommand{\mathcalcommand}[1]{\mathcal{#1}}

\newcommand{\mcF}{\mathcalcommand{F}}

\newcommand{\mcI}{\mathcalcommand{I}}

\newcommand{\mcX}{\mathcalcommand{X}}
\newcommand{\mcY}{\mathcalcommand{Y}}

\DeclareMathAlphabet{\mathpzc}{T1}{pzc}{m}{it}

\allowdisplaybreaks

\newcommand{\tcb}[1]{{{\color{blue}{#1}}}}

\newcommand*{\commentout}[1]{}

\newlength{\parskiptrue}
\definecolor{lred}{rgb}{1.0, 0.5, 0.5}
\definecolor{lorange}{rgb}{1.00, 0.90, 0.20}
\definecolor{lgreen}{rgb}{0.35, 0.95, 0.35}
\definecolor{lime}{rgb}{0.9, 1.0, 0.6}
\definecolor{lblue}{rgb}{1.0, 0.85, 0.75}

\makeatletter

\newcommand*\wthelper[2]{%
        \hbox{\dimen@\accentfontxheight#1%
                \accentfontxheight#11.1\dimen@
                $\m@th#1\widetilde{#2}$%
                \accentfontxheight#1\dimen@
        }%
}
\newcommand*\accentfontxheight[1]{%
        \fontdimen5\ifx#1\displaystyle
                \textfont
        \else\ifx#1\textstyle
                \textfont
        \else\ifx#1\scriptstyle
                \scriptfont
        \else
                \scriptscriptfont
        \fi\fi\fi3
}

\newcommand*\whhelper[2]{%
        \hbox{\dimen@\accentfontxheight#1%
                \accentfontxheight#11.2\dimen@
                $\m@th#1\widehat{#2}$%
                \accentfontxheight#1\dimen@
        }%
}
\makeatother

\makeatletter
\newcommand{\oset}[3][0ex]{%
  \mathrel{\mathop{#3}\limits^{
    \vbox to#1{\kern-3\ex@
    \hbox{$\scriptstyle#2$}\vss}}}}
\makeatother

\newcommand*{\defeq}{\triangleq}

\newcommand*{\relu}{\mathrm{ReLU}}
\newcommand*{\step}{\mathrm{Step}}

\DeclareMathOperator*{\argmin}{arg\,min}

\newcommand*{\round}[1]{ \lceil{#1}\rfloor }
\newcommand*{\indc}[1]{\mathbbm{1}\!\left[{#1}\right]}

\newcommand*{\aff}{{\mathrm{aff}}}

\newcommand*{\ceilZ}[1]{ \lceil{#1}\rceil_{\mathbb{Z}} }

\makeatletter
\def\moverlay{\mathpalette\mov@rlay}
\def\mov@rlay#1#2{\leavevmode\vtop{%
   \baselineskip\z@skip \lineskiplimit-\maxdimen
   \ialign{\hfil$\m@th#1##$\hfil\cr#2\crcr}}}
\newcommand{\charfusion}[3][\mathord]{
    #1{\ifx#1\mathop\vphantom{#2}\fi
        \mathpalette\mov@rlay{#2\cr#3}
      }
    \ifx#1\mathop\expandafter\displaylimits\fi}
\makeatother

\newcommand{\bigcupdot}{\charfusion[\mathop]{\bigcup}{\cdot}}

\renewcommand{\paragraph}[1]{{#1}}
\setlist{noitemsep, topsep=5pt}

 \AtBeginDocument{%
    \crefname{section}{Section}{Sections}%
    \crefname{appendix}{Appendix}{Appendices}%
    \crefname{subsection}{Section}{Sections}%
    \crefname{figure}{Figure}{Figures}%
}

\newtheorem{theorem}{Theorem}

\newtheorem{lemma}[theorem]{Lemma}

\newtheorem{corollary}[theorem]{Corollary}
\newtheorem{definition}[theorem]{Definition}
\newtheorem{example}[theorem]{Example}

\usepackage{letltxmacro}
\LetLtxMacro\oldttfamily\ttfamily
\DeclareRobustCommand{\ttfamily}{\oldttfamily\csname ttsize\endcsname}
\newcommand{\setttsize}[1]{\def\ttsize{#1}}%
\setttsize{\small}

\newcommand{\listingsttfamily}{\small\oldttfamily}
\lstdefinestyle{prettycode}{
  basicstyle=\listingsttfamily,
  keywordstyle=,%
  morekeywords={prog, func, if, else, return},
  keepspaces=true,
  mathescape=true,
}
\newcommand{\Mod}[1]{\ (\mathrm{mod}\ #1)}

\usepackage{algorithm}
\usepackage{algorithmic}
\usepackage{enumitem}

\title{Expressive Power of ReLU and Step Networks \\ under Floating-Point Operations}

\author{Yeachan Park\thanks{equal contribution}\ \ \ \ \ \
Geonho Hwang$^*$\ \ \ \ \ \
Wonyeol Lee\ \ \ \ \ \
Sejun Park\thanks{corresponding author}}
\date{}

\begin{document}
\maketitle

\begin{abstract}
The study of the expressive power of neural networks has investigated the fundamental limits of neural networks. Most existing results assume real-valued inputs and parameters as well as exact operations during the evaluation of neural networks. However, neural networks are typically executed on computers that can only represent a tiny subset of the reals and apply inexact operations, i.e., most existing results do not apply to neural networks used in practice. In this work, we analyze the expressive power of neural networks under a more realistic setup: when we use floating-point numbers and operations as in practice. Our first set of results assumes floating-point operations where the significand of a float is represented by finite bits but its exponent can take any integer value. Under this setup, we show that neural networks using a binary threshold unit or $\relu$ can memorize any finite input/output pairs and can approximate any continuous function within an arbitrary error. 
In particular, the number of parameters in our constructions for universal approximation and memorization coincides with that in classical results assuming exact mathematical operations.
We also show similar results on memorization and universal approximation when floating-point operations use finite bits for both significand and exponent; these results are applicable to many popular floating-point formats such as those defined in the IEEE 754 standard (e.g., 32-bit single-precision format) and bfloat16. 
\end{abstract}

\section{Introduction}\label{sec:intro}
Identifying the expressive power of neural networks is an important problem in the theory of deep learning. 
Theoretical results represented by the universal approximation theorem have shown that neural networks with sufficiently large width and depth are able to approximate a continuous function on a compact domain within an arbitrary error \cite{cybenko89, Hornik89, pinkus99, Lu17, park21}.
Another line of work on memory capacity has studied regression problems and showed that shallow networks with $O(n)$ parameters can fit arbitrary $n$ input/output pairs \cite{baum88, huang98, vershynin20, yun19}, while $o(n)$ parameters are sufficient for deep ones, i.e., $\omega(1)$ layers \cite{park21b, vardi22}.
However, since these results assume exact mathematical operations, they %
do not apply to neural networks executed on computers that can only represent a tiny subset of the reals (e.g., floating-point numbers) and perform inexact operations (e.g., floating-point operations)  \cite{wray95,puheim14}.

Several works studied the expressive power under machine-representable parameters.
For example, \cite{ding18} showed that for approximating a continuous function within $\varepsilon$ error using quantized weights, $O(K\log^5\varepsilon)$ parameters are sufficient where $K$ denotes the number of parameters for approximation using real-valued parameters.
The memory capacity of networks with quantized weights was also studied in
\cite{park21b,vardi22}.
However, these works also assume exact mathematical operations, and hence, what neural networks can/cannot do on actual computers has remained unknown. %

Most neural networks used in practice operate under floating-point arithmetic. %
Typically, a floating-point number $x$ can be written as $x=s_x\times a_x\times 2^{e_x}$,
where $s_x$, $a_x$, and $e_x$ are called the \emph{sign}, \emph{significand}, and \emph{exponent} of $x$, respectively. %
Here, to express $x$ using a finite memory, $s_x$, $a_x$, and $e_x$ must have finite-bit representations:
e.g., in 32-bit single-precision floats, %
$s_x$ uses $1$ bit, $a_x$ uses $23$ bits, and $e_x$ uses $8$ bits \cite{ieee19}.
If we apply floating-point addition/multiplication to two floating-point numbers, the result is rounded to a nearest\footnote{{There are other rounding modes (e.g., ``round towards $0$'') as well, but this paper assumes the ``round to nearest (ties to even)'' mode since it is the most commonly used one (e.g., it is the default rounding mode in the IEEE 754 standard~\cite{ieee19}).}} floating-point number; hence, floating-point operations may incur some rounding error. %

Although the relative error of each floating-point operation is usually small, the relative error incurred by a composition of such operations can be arbitrarily large, even with a small number of operations. %
For example, let $\delta$ be some small floating-point number such that $1/\delta$ is also a floating-point number, and consider the expression $(1/\delta)\times(1+\delta-1)$. %
If floating-point operations are applied to this expression and $\delta$ is small enough, %
then $1$ is returned for the expression $1+\delta$. %
This implies that $0$ is returned for $(1+\delta-1)$
as well as for the entire expression.
Thus, the final output $0$
has a relative error one compared to the true output $1$ (which is computed under exact operations). %

\subsection{Summary of contribution}\label{sec:contribution}
Unlike prior works considering exact mathematical operations that typical neural network implementations do not perform, in this work, we investigate the expressive power of neural networks under floating-point operations, and $\step(x)=\indc{x\ge0}$ and $\relu(x)=\max\{x,0\}$ activation functions. %
To our knowledge, this is the first work that considers practical neural networks typically implemented by computers using floating-point numbers (e.g., inputs, parameters, and intermediate values) and operations (e.g., addition and multiplication).

Our first set of results is for $\bbF_p$, a set of floating-point numbers with \emph{$p$-bit significand and unbounded exponent} (i.e., the exponent can be any integer), and an input dimension {$d\le2^p$}.
We note that $\bbF_p$ have been widely used in the floating-point literature since floating-point operations in $\bbF_p$ do not incur any overflow and underflow \cite{BoldoM11, Boldo15, Jeannerod15, JeannerodLMP16, JeannerodR18}, i.e., $\bbF_p$ is often easier to analyze compared to the bounded exponent case.

\begin{itemize}[leftmargin=15pt]
\item As in the classical results,  we show for $\step$ networks that $O(n)$ parameters are sufficient for memorizing arbitrary $n$ input/output pairs. \cref{thm:memstepfp-temp} states that there exists a $\step$ network $f_\theta:\bbF_p^d\to\bbF_p$ of $O(n)$ parameters that is parameterized by $\theta$, uses only floating-point operations, and satisfies the following: %
for any $\mathcal D=\{(\bfz_1,y_1),\dots,(\bfz_n,y_n)\}\subset\bbF_p^d\times\bbF_p$ such that $\bfz_i\ne\bfz_j$ for all $i\ne j$, there exists $\theta_{\mathcal D}$ satisfying
$$f_{\theta_{\mathcal D}}(\bfz_i)=y_i~\text{for all}~i.$$
\item We next show that $\step$ networks can also universally approximate continuous functions on a unit cube. \cref{thm:univstepfp-temp} states that for any continuous $f^*:[0,1]^d\to\bbR$ and $\varepsilon>0$, there exists a floating-point $\step$ network $f:[0,1]^d\cap\bbF_p^d\to\bbF_p$ such that
$$|f(\bfx)-f^*(\bfx)|\le|f^*(\bfx)-\round{f^*(\bfx)}|+\varepsilon~\text{for all}~\bfx\in[0,1]^d\cap\bbF_p^d,$$
where $\round{f^*(\bfx)}$ denotes a floating-point number in $\bbF_p$ closest to $f^*(\bfx)$.
The first term in the error bound denotes an intrinsic error arising from the nature of floating-point arithmetic, i.e., one cannot achieve a smaller error than this term.
\item By carefully analyzing floating-point operations, we further extend these results to $\relu$ networks in \cref{thm:memrelufp-temp,thm:univrelufp-temp}.
\end{itemize}

Our next set of results considers a more realistic class of floating-point numbers $\bbF_{p,q}$ with \emph{$p$-bit significand and $q$-bit exponent}, i.e., each number in $\bbF_{p,q}$ can be represented by a finite number of bits.
In particular, we focus on $p,q$ satisfying $q\ge5$ and $4\le p\le2^{q-2}+2$.
This condition on $p,q$ is met by many practical floating-point formats such as those defined in the IEEE 754 standard (e.g., 32-bit single-precision format) \cite{ieee19} and bfloat16 \cite{tensorflow};
see \cref{sec:float} for details. %
We note that operations in $\bbF_{p,q}$ may incur overflow or underflow unlike in $\bbF_p$.

\begin{itemize}[leftmargin=15pt]
\item \cref{thm:memstepfpqr-temp,thm:univstepfpqr-temp} show that memorization and universal approximation using $\step$ networks are possible under $\bbF_{p,q}$ with a similar number of parameters for the $\bbF_p$ case. This shows that $\step$ networks are expressive even under realistic floating-point operations.
\item Showing similar memorization and universal approximation results for $\relu$ networks is more challenging since 
the output of an activation function can be very large unlike in $\step$ networks;
this can incur overflow if the output is multiplied by a large weight in the next layer. %
Nevertheless, by carefully analyzing floating-point operations under $\bbF_{p,q}$, we show in \cref{thm:memrelufpqr-temp} that %
$O(n)$ parameters are sufficient for memorizing arbitrary $n$ input/output pairs $(\bfz_1,y_1),\dots,(\bfz_n,y_n)$,
under a mild condition on inputs.
In addition, \cref{thm:univrelufpqr-temp} shows that universal approximation using $\relu$ networks is possible under $\bbF_{p,q}$. %
\end{itemize}

We lastly note that the numbers of parameters 
used in our results to show memorization and universal approximation have the same order as the necessary and sufficient numbers of parameters shown in corresponding prior results under exact operations.

\subsection{Organization}
We introduce our problem setup and notations in \cref{sec:setup}. %
We then formally describe our results on the expressive power of $\step$ networks and $\relu$ networks under $\bbF_p$ and $\bbF_{p,q}$ in \cref{sec:fp,sec:fpqr}. 
We present the proofs of these results in \cref{sec:pfresults-fp,sec:pfresults-fpqr}, %
and conclude the paper in \cref{sec:conclusion}. %

\section{Problem setup and notations}\label{sec:setup}
\subsection{Notations}
We first introduce the notations used in this paper. For $n\in\bbN$, we use $[n]\defeq\{1,\dots,n\}$. We often use lower-case alphabets $a,b,c,\dots$ for denoting scalar values and 
bold lower-case alphabets $\bfa,\bfb,\bfc,\dots$ for denoting column vectors where $f,g,h$ are reserved for denoting a scalar-valued functions.
For $n\in\bbN$, we use $\mathbf{1}_n$ to denote the $n$-dimensional vector consisting of ones.
For a vector $\bfx$, we use $x_i$ to denote its $i$-th coordinate.
For $\mathcal S\subset\bbR$ and $x\in\bbR$, we use $x^{(\ge;\mathcal S)}\defeq\inf_{v \in \mathcal S :v \ge x} v $, $x^{(\le;\mathcal S)}\defeq\sup_{v \in \mathcal S :v\le x} v $, $x^{(>;\mathcal S)}\defeq\inf_{v \in \mathcal S :v > x} v $, and $x^{(<;\mathcal S)}\defeq\sup_{v \in \mathcal S :v< x} v $. Likewise, for $\bfx=(x_1,\dots,x_d)$, we define $\bfx^{(\le;\mathcal S)}\defeq (x_1^{(\le;\mathcal S)},\dots,x_d^{(\le;\mathcal S)})$ and use $\bfx^{(\ge;\mathcal S)}$, $\bfx^{(<;\mathcal S)}$, and $\bfx^{(>;\mathcal S)}$ similarly.
For $\bfx,\bfy\in\bbR^d$, we use $[\bfx,\bfy]\defeq[x_1,y_1]\times\cdots\times[x_d,y_d]$; we also use $[\bfx,\bfy)$, $(\bfx,\bfy]$, and $(\bfx,\bfy)$ similarly.
We use $\bigcupdot$ to denote the disjoint union.    

We use $C(\mathcal X,\mathcal Y)$ to denote the set of all continuous functions from $\mathcal X$ to $\mathcal Y$.
Given a continuous function $f:\mathcal X\to\bbR$ for some compact $\mathcal X\subset\bbR^d$, we use $\omega_f:\bbR_{\ge0}\to\bbR_{\ge0}$ to denote the modulus of continuity of $f$, defined as
$$\omega_f(\delta)\defeq\sup_{\bfx,\bfx'\in\mathcal X:\|\bfx-\bfx'\|_\infty\le \delta}|f(\bfx)-f(\bfx')|,$$
and we use $\omega_f^{-1}:\bbR_{\ge0}\to\bbR_{\ge0} \cup \{\infty\}$ to denote its inverse, defined as
$$\omega_f^{-1}(\varepsilon)\defeq\sup\{\delta\ge0:\omega_f(\delta)\le\varepsilon\}.$$
Throughout this paper, we treat the input dimension $d$ as a constant and often hide it in the big-O notation $O(\cdot)$.

\subsection{Floating-point numbers}
\label{sec:float}
We consider two types of floating-point numbers: 
\begin{align}
\nonumber
\bbF_{p}&\defeq\big\{s\times (1.a_1\cdots a_p) \times 2^{b}:s\in\{-1,1\},a_1,\dots,a_p\in\{0,1\},b\in\bbZ\big\} \cup {\big\{0\big\}},
\\
\nonumber
\bbF_{p,q}&\defeq\big\{s\times (1.a_1\cdots a_p) \times 2^{b-1+e_{min}}:s\in\{-1,1\},a_1,\dots,a_p\in\{0,1\},b\in[2^q-{2}]\big\}
\\
\label{eq:float-def}
&\qquad\cup\big\{s\times (0.a_1\cdots a_p) \times 2^{e_{min}} %
:s\in\{-1,1\},a_1,\dots,a_p\in\{0,1\}\big\},
\end{align}
where $1.a_1\cdots a_p$ and $0.a_1\cdots a_p$ denote binary representations, $e_{min} \defeq -2^{q-1}+2$, and $e_{max} \defeq 2^{q-1}-1$.
In \Cref{eq:float-def}, the first factor (i.e., $s$) of each floating-point number is called the \emph{sign}, the second factor (e.g., $1.a_1 \cdots a_p$) called the \emph{significand}, and $\log_2$ of the third factor (e.g., $b$) called the \emph{exponent}.
For instance, the sign, significand, and exponent of any $x \in \bbF_{p,q}$ are always in $\{-1,1\}$, $[0,2)$, and $[e_{min}, e_{max}]$, respectively.
Many floating-point formats used in practice can also represent three special values: $+\infty$, $-\infty$, and NaN (not-a-number). We do not include these values in $\bbF_p$ and $\bbF_{p,q}$, yet we do consider them in our results (including proofs).
For $\bbF_{p,q}$, we assume that $p$ and $q$ satisfy $q \ge 5$ and $4 \le p \le 2^{q-2}+2$.

As we introduced in \cref{sec:contribution}, floating-point numbers with unbounded exponent $\bbF_p$ have been widely used since it does not incur any underflow and overflow (i.e., rounding to $0$ or $\pm\infty$) unlike those over $\bbF_{p,q}$ \cite{BoldoM11, Boldo15, Jeannerod15, JeannerodLMP16, JeannerodR18}.
On the other hand, each number in $\bbF_{p,q}$ can be represented by a finite number of bits, and $\bbF_{p,q}$ covers many practical floating-point formats. %
For example, the following floating-point formats defined in the IEEE 754 standard \cite{ieee19} are all instances of $\bbF_{p,q}$:
the 16-bit half-precision format has $(p,q)=(10,5)$;
the 32-bit single-precision format has $(p,q)=(23,8)$;
the 64-bit double-precision format has $(p,q)=(52,11)$; and
the 128-bit quadruple-precision format has $(p,q)=(112,15)$.
Also, the bfloat16 format \cite{tensorflow}, frequently used in machine learning these days, is also an instance of $\bbF_{p,q}$ with $(p,q)=(7,8)$.

We say a non-zero number $x \in \bbF_{p,q}$ is \emph{normal} if $|x| \ge 2^{e_{min}}$ (i.e., significand is at least $1$), and is \emph{subnormal} if $|x| < 2^{e_{min}}$ (i.e., significand is smaller than $1$). In $\bbF_{p,q}$, we denote the unit round-off as $u \defeq 2^{-p}$, the largest number as $\Omega \defeq (2 - u) \times 2^{e_{max}}$, and the smallest positive number as $\eta \defeq 2^{-p+e_{min}}$. We also use $u \defeq 2^{-p}$ for $\bbF_p$.
We often use $\bbF$ to denote either $\bbF_{p}$ or $\bbF_{p,q}$. %

\subsection{Floating-point operations}
\label{sec:float-op}

For both $\bbF_p$ and $\bbF_{p,q}$, we consider the rounding mode called ``round to nearest (ties to even)''.
Roughly, we round $x\in \bbR$ to a floating-point number $y \in \bbF$ that is closest to $x$,
where ties are broken by choosing a floating-point number that has $0$ as the $p$-th binary digit of its significand (i.e., $a_p=0$ in \Cref{eq:float-def}).
Formally, the rounding of $x \in \bbR$ in $\bbF$ is defined as
\begin{align*}
    \round{x}_{\bbF_p} &\defeq \argmin_{y \in \bbF_p} |x-y|,
    &
    \round{x}_{\bbF_{p,q}} &\defeq 
    \begin{cases}
        \argmin_{y \in \bbF_{p,q}} |x-y| & \text{if } x \in (-\Omega',\Omega'),
        \\
        +\infty & \text{if } x \in [\Omega', \infty),
        \\
        -\infty & \text{if } x \in (-\infty, -\Omega'],
    \end{cases}    
\end{align*}
where ties are broken as explained above and $\Omega' = \Omega + 2^{e_{max}-p-1}$ \cite[Chapter 2.1]{BoldoJMM23}.
If $\bbF$ is clear from context, we use $\round{x}$ to denote $\round{x}_\bbF$.
For $x\in\bbF$, we use $x^-\defeq\sup_{z\in\bbF:z<x}z$ and 
$x^+\defeq\inf_{z\in\bbF:z>x}z$ if exist: e.g., $0^+$ and $0^-$ do not exist in $\bbF_p$.
We note that all network constructions in our results do not generate infinities and NaN during all operations including intermediate ones.

We consider the following floating-point operations for $\bbF$: $a\oplus_\bbF b\defeq\round{a+b}_\bbF$, $a\ominus_\bbF b\defeq a\oplus(-b)$, and $a\otimes_\bbF b \defeq \round{a\times b}_\bbF$.
We use $a\oplus b$, $a\ominus b$, $a\otimes b$ to denote $a\oplus_\bbF b$, $a\oplus_\bbF b$, $a\otimes_\bbF b$ if it is clear from the context.
Since the addition between floating-point numbers is not associative, the ordering of a summation becomes important. 
We define the summation operation $\bigoplus$ of a sequence of floating-point numbers $x_1,x_2,\dots$ inductively as follows: for $\bigoplus_{i=1}^{1}x_i\defeq x_1$, 
$$\bigoplus_{i=1}^n x_i \defeq \left(\bigoplus_{i=1}^{n-1}x_i\right) \oplus   x_n.$$
For $y,x_1,x_2,\dots,\in\bbF$, we also define the order of operation as
$$y \oplus \bigoplus_{i=1}^n x_i \defeq \left( y \oplus \bigoplus_{i=1}^{n-1} x_i\right) \oplus x_n,$$
and
$$y \ominus \bigoplus_{i=1}^n x_i \defeq \left( y \ominus \bigoplus_{i=1}^{n-1} x_i\right) \ominus x_n.$$
Using $\bigoplus$, we define an affine transformation under floating-point arithmetic as follows:
for $\bfx=(x_1,\dots,x_n)\in\bbF^n$, %
$\bfw=(w_1,\dots,w_k)\in\bbF^k$, %
and $\mathcal I=\{i_1,\dots,i_k\}\subset[n+1]$ with $i_1<\cdots<i_k$,
$$\aff_\bbF(\bfx,\bfw,\mathcal I)\defeq\bigoplus_{j=1}^k(w_{j}\otimes z_{i_j})$$
where $(z_1,\dots,z_{n+1})=(x_1,\dots,x_n,1)$.

\subsection{Neural networks}
\label{subsection:nn}

Let $\bbF$ be a set of floating-point numbers, $L\in\bbN$ be the number of layers, $N_0=d$ and $N_L=1$ be the input and output dimensions, $N_1,\dots,N_{L-1}\in\bbN$ be numbers of hidden neurons, and $\mathcal I_{l,i}=\{\iota_{l,i,1},\dots,\iota_{l,i,|\mathcal I_{l,i}|}\}\subset[N_{l-1}+1]$ be a set of indices that characterizes the affine map used for computing an input to the $i$-th hidden neuron in the layer $l$.
Let $N=\sum_{i=1}^{L-1}N_l$ be the total number of hidden neurons and $I=\sum_{l=1}^L\sum_{i=1}^{N_l}|\mcI_{l,i}|$ be the total number of free parameters.
Under this setup, we define a neural network $f_{\theta, \mathcal{I}}(\,\cdot\,;\bbF):\bbF^d\to\bbF$ parameterized by $\theta\in\bbF^I$ with  $\mathcal{I}\defeq(\mathcal I_{1,1},\dots,\mathcal I_{1,N_l},\dots,\mathcal I_{L,1},\dots,\mathcal I_{1,N_L})$ via the following recursive relationship: for each $l\in[L]$,

\begin{align}
f_{\theta,\mathcal{I}}(\bfx;\bbF)&\defeq\phi_L(\bfx;\bbF),\notag\\
\phi_l(\bfx;\bbF)&\defeq\big(\phi_{l,1}(\bfx;\bbF),\dots,\phi_{l,N_l}(\bfx;\bbF)\big),\notag\\
\phi_{l,i}(\bfx;\bbF)&\defeq  \;  \aff_\bbF\big(\psi_{l-1}(\bfx;\bbF),\bfw_{l,i},\mathcal I_{l,i}\big) = \bigoplus_{j =1 }^{|\mathcal I_{l,i}|}(w_{l,i,j} \otimes \psi_{l-1}(\bfx;\bbF)_{\iota_{l,i,j}}), \; \notag\\
\psi_l(\bfx;\bbF)&\defeq\big(\sigma(\phi_{l,1}(\bfx;\bbF)),\dots,\sigma(\phi_{l,N_l}(\bfx;\bbF))\big), \label{eq:def-nn}
\end{align}
where $\psi_0(\bfx;\bbF)\defeq\bfx$, $\sigma:\bbF\to\bbF$ denotes a pointwise activation function, and $\theta$ denote the concatenations of all $w_{l,i,k}$, i.e., $\theta\in\bbF^I$. %
In the definition of neural networks, we note that $\phi_{l,i}(\bfx;\bbF)$ denotes the mapping of the input $\bfx$ to the feature of the $i$-th neuron at the $l$-th layer \emph{before} applying the activation function of the $l$-th layer, and $\psi_{l,i}(\bfx;\bbF)$ denotes the mapping of the input $\bfx$ to the feature of the $i$-th neuron at the $l$-th layer \emph{after} applying the activation function of the $l$-th layer. %
See \cref{figure:nn} for an illustration of the neural network example.

\begin{figure}
\centering
\includegraphics[width=\textwidth]{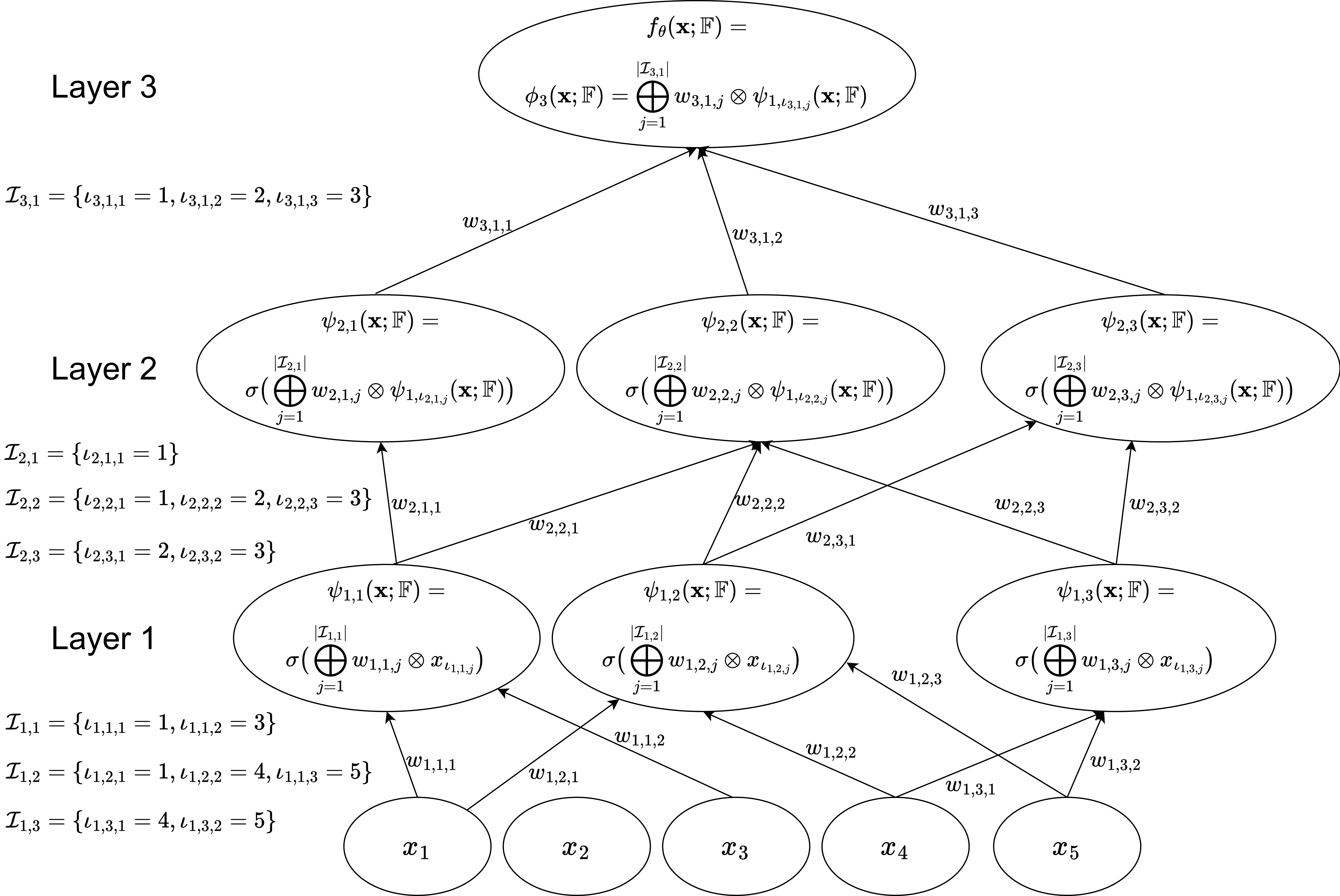}
\caption{An illustration of a neural network is presented in \cref{eq:def-nn}.  In this example, we set the parameters as follows: $d=5$, $L=3$, and $N_1=N_2=3$. 
}
\label{figure:nn}
\end{figure}

Namely, such a network has $L$ layers, $N$ (hidden) neurons, and $I$ parameters, where each layer $l$ has $\sum_{i=1}^{N_l}|\mathcal I_{l,i}|$ free parameters. Note that if $\mathcal{I}_{l,i} = [N_{l-1}+1]$ for all $l,i$, then the network is fully-connected.
For notational simplicity, we often omit $\mathcal{I}$ and use $f_{\theta}$. %
We say a neural network $f_\theta$ is a \emph{$\sigma$ network} if $f_\theta$ uses $\sigma$ as its activation function in \cref{eq:def-nn} for all layers $l\in[L]$. Note that within our network architecture, every input, output, parameter of the network, and number occurring during the intermediate calculation is in $\bbF$. Furthermore, all operations performed adhere to floating-point arithmetic.
\subsection{Memorization}
For a set of floating-point numbers $\bbF$, $\mathcal Z\subset\bbF^d$, and $\mathcal Y\subset\bbF$, we say a neural network $f_\theta(\,\cdot\,;\bbF):\bbF^d\to\bbF$ of $I$ parameters can \emph{memorize any set of $n$ pairs in $\mathcal Z\times\mathcal Y$} if for any 
$\{(\bfz_1,y_1),\dots,(\bfz_n,y_n)\}\subset\mathcal Z\times\mathcal Y$ with $\bfz_i\ne \bfz_j$ for all $i\ne j$, there exists a parameter configuration $\theta\in\bbF^I$ such that
$f_\theta(\bfz_i;\bbF)=y_i$
for all $i\in[n]$.
\subsection{Universal approximation}
For $\bbF\in\mcF$, a compact set $\mcX\subset\bbR^d$, and $\mcY\subset\bbR$, we say $\sigma$ networks can \emph{universally approximate} $C(\mcX,\mcY)$ under $\bbF$ if for each $f^*\in C(\mcX,\mcY)$ and $\varepsilon>0$, there exists a $\sigma$ network $f_\theta(\,\cdot\,;\bbF):\bbF^d\to\bbF$ such that
\begin{equation*}
|f_\theta(\bfx;\bbF)-f^*(\bfx)|\le|f^*(\bfx)-\round{f^*(\bfx)}_{\bbF}|+\varepsilon
\end{equation*}
for all $\bfx\in\mathcal X\cap\bbF^d$.
As we described in \cref{sec:contribution}, the term $|f^*(\bfx)-\round{f^*(\bfx)}_{\bbF}|$ is an intrinsic error from the representation of floating-point numbers; one cannot obtain a smaller error than this.

\section{Expressive power of neural networks under $\bbF_p$}\label{sec:fp}

We are now ready to present our main results on the universal approximation and memorization properties of neural networks under floating-point inputs, parameters, and operations.
In particular, we first introduce our results for $\bbF_p$ in \cref{sec:fp}, and then extend these results to $\bbF_{p,q}$ in \cref{sec:fpqr}. %
We defer the proofs of all results in this section to \cref{sec:pfresults-fp}.
\subsection{$\step$ network results}\label{sec:stepfp}%
To show our memorization and universal approximation results for $\step$ network under $\bbF_p$, we explicitly construct $\step$ networks that can memorize arbitrary $n$ pairs or universally approximate a target function.
Here, our $\step$ network constructions consist of indicator functions for $d$-dimensional cubes and affine transformations, where such multi-dimensional indicator functions can be implemented by a single $\step$ network architecture as in the following lemma. 
The proof of \cref{lem:multidim_step-temp} is presented in \cref{sec:pflem:multidim_step-temp}.
\begin{lemma}\label{lem:multidim_step-temp}
For any $p\in\bbN$ and $d\in[2^{p}]$, there exists a $\step$ network $f_{\theta}(\,\cdot\,;\bbF_p):\bbF_p^d\to\bbF_p$ of $3$ layers and $6d+2$ parameters that satisfies the following: 
for any $\alpha=(\alpha_1,\dots,\alpha_d),\beta=(\beta_1,\dots,\beta_d)\in\bbF_p^d$
with $\alpha_i<\beta_i$ for all $i\in[d]$, there exists $\theta_{\alpha,\beta}\in\bbF_p^{6d+2}$ such that
$$f_{\theta_{\alpha,\beta}}(\bfx;\bbF_p)=\indc{\bfx\in\prod_{i=1}^d[\alpha_i,\beta_i]}$$
for all $\bfx\in[0,1]^d$. 
\end{lemma}
\cref{lem:multidim_step-temp} shows that if the input dimension $d$ smaller than or equal to $2^{p}$, a three-layer $\step$ network of $O(1)$ parameters can represent any indicator function for a $d$-dimensional cube under $\bbF_p$. 
We note that the parameter $\theta_{\alpha,\beta}$ of this network is a function of the cube $\Pi_{i=1}^d[\alpha_i,\beta_i]$ in the indicator function.

Using \cref{lem:multidim_step-temp}, we show the existence of a $\step$ network $f_\theta(\,\cdot\,;\bbF_p):\bbF_p^d\to\bbF_p$ %
such that for any
$\mathcal D=\{(\bfz_1,y_1),\dots,(\bfz_n,y_n)\}\subset\bbF_p^d\times\bbF_p$ with $\bfz_i\ne\bfz_j$ for all $i\ne j$, there exists $\theta_{\mathcal D}$ satisfying
$$f_{\theta_{\mathcal D}}(\bfx;\bbF_p)=\sum_{i=1}^ny_i\otimes \indc{\bfx\in\prod_{j=1}^d[z_{ij},z_{ij}]}.$$
Such a result is formally stated in the following theorem; its proof is presented in \cref{sec:pfthm:memstepfp-temp}.
\begin{theorem}\label{thm:memstepfp-temp}
For any $p\in\bbN$ and $d\in[2^{p}]$, there exists a $\step$ network $f_{\theta}(\,\cdot\,;\bbF_p):\bbF_p^d\to\bbF_p$ of $3$ layers and {$6dn+2n$} parameters satisfying the following: for any $\mathcal D=\{(\bfz_1,y_1),\dots,(\bfz_n,y_n)\}\subset\bbF_p^d\times\bbF_p$ with $\bfz_i\ne\bfz_j$ for all $i\ne j$, there exists $\theta_{\mathcal D}$ such that
$$f_{\theta_{\mathcal D}}(\bfz_i)=y_i\quad\forall i\in[n].$$
\end{theorem}
\cref{thm:memstepfp-temp} shows that even under floating-point operations, $\step$ networks can successfully memorize finite datasets represented by floating-point numbers in $\bbF_p$. Furthermore, 
it states that $O(n)$ parameters are sufficient for $\step$ networks to memorize arbitrary $n$ pairs under $\bbF_p$; this coincides with the real-valued parameters and exact operation case \cite{baum88, huang98, vershynin20, yun19} that are known to be tight up to a logarithmic multiplicative factor for networks of $O(1)$ layers using piecewise linear activation functions \cite{bartlett19}.
Namely, the required number of parameters for memorization under $\bbF_p$ does not decrease compared to existing results assuming real operations.

We next show that $\step$ networks can universally approximate $C([0,1]^d,\bbR)$ under $\bbF_p$. %
See \cref{sec:pfthm:univstepfp-temp} for the proof of \cref{thm:univstepfp-temp}.
\begin{theorem}\label{thm:univstepfp-temp}
For any $p\in\bbN$, $d\in[2^{p}]$,  $f^*\in C([0,1]^d,\bbR)$, and $\varepsilon>0$, there exists a $\step$ network $f_{\theta}(\,\cdot\,;\bbF_p):\bbF_p^d\to\bbF_p$ of {$3$} layers and at most {$(6d+2)K^d$} parameters %
where $K=\min\{k\in\bbN:(\omega_{f^*}^{-1}(\varepsilon))^{-1}\le k\}$ such that 
$$|f_\theta(\bfx;\bbF_p)-f^*(x)|\le|f^*(\bfx;\bbF_p)-\round{f^*(x)}|+\varepsilon\quad\forall\bfx\in[0,1]^d\cap\bbF_p^d.$$
\end{theorem}
\cref{thm:univstepfp-temp} states that given a target continuous function $f^*:[0,1]^d\to\bbR$ and $\varepsilon>0$, $O(\omega_{f^*}^{-1}(\varepsilon)^{-d})$ parameters are sufficient for approximating $f^*$ in $|f^*(x;\bbF_p)-\round{f^*(x)}|+\varepsilon$ error; this result easily generalizes to an arbitrary compact domain in $\bbR^d$ instead of $[0,1]^d$. The number of parameters in the network in \cref{thm:univstepfp-temp} is similar to existing results under real-valued parameters and the exact operations for $\relu$ networks of $O(1)$ layers: $(\omega_{f^*}^{-1}(\Theta(\varepsilon)))^{-d}$ parameters are necessary and sufficient for approximation in $\varepsilon$ error \cite{yarotsky18}.

The error bound in \cref{thm:univstepfp-temp} contains an additional term $|f^*(x;\bbF_p)-\round{f^*(x)}|$ compared to the classical universal approximation results. This term corresponds to an intrinsic error arising from the floating-point representation; one cannot achieve a smaller error than this.
However, if the domain of the target function contains only finite numbers in $\bbF_p^d$, then we can approximate this function in the intrinsic error without an additional error term $\varepsilon$ as in the following corollary of our memorization result (\cref{thm:memstepfp-temp}). We also note that the domain $[0,1]^d\cap\bbF_p^d$ considered in \cref{thm:univstepfp-temp} is an infinite set, and hence, cannot directly apply this corollary.
\begin{corollary}\label{cor:univstepfp-temp}
For any $p\in\bbN$, $d\in[2^{p}]$, a compact set $\mathcal K\subset\bbR^d$ such that $|\mathcal K\cap\bbF_p^d|<\infty$, $f^*\in C(\mathcal K,\bbR)$, and $\varepsilon>0$, there exists a $\step$ network $f(\,\cdot\,;\bbF_p):\bbF_p^d\to\bbF_p$ of {$3$} layers and {$6d|\mathcal K\cap\bbF_p^d|+2|\mathcal K\cap\bbF_p^d|$} parameters %
such that 
$$|f_\theta(\bfx;\bbF_p)-f^*(\bfx)|=|f^*(\bfx;\bbF_p)-\round{f^*(\bfx)}|\quad\forall\bfx\in[0,1]^d\cap\bbF_p^d.$$
\end{corollary}

\subsection{$\relu$ network results}\label{sec:relufp}
We next analyze the expressive power of $\relu$ networks under $\bbF_p$.
Given the $\step$ network results, a na\"ive approach can be implementing $\step$ function using two $\relu$ activation functions. For example, for any two consecutive numbers $z^-<z$ in $\bbF_p$,
\begin{align}
\indc{x\ge z}=\relu\left(\frac1{z-z^-}(x-z^-)\right)-\relu\left(\frac1{z-z^-}(x-z)\right)\label{eq:naiverelu}
\end{align}
for all $x\in\bbF_p$ under \emph{exact operations}.
However, if we consider floating-point operations in $\bbF_p$, RHS in \cref{eq:naiverelu} does not represent the indicator function anymore. 
For example, consider the indicator function $\indc{x \ge 1}$, which can be exactly implemented under the exact mathematical operations for all inputs $x\in\bbF_p$ by 
\[ f(x ; \bbR) = \relu\big( (x \times 2^p) - (1-u) \times 2^p \big) -  \relu\big((x \times 2^p) -2^p \big)  \]
where $u=2^{-p}$.
Under operations in $\bbF_p$, however, $f(x ; \bbF_p)$ with the following form cannot represent $\indc{x \ge 1}$ for some $x\in\bbF_p$: %
\begin{align*}
    f(x ; \bbF_p) &= \relu\big( (x \times 2^p) \ominus ((1-u) \times 2^p) \big) \ominus  \relu\big((x \times 2^p) \ominus 2^p \big)  . 
\end{align*}
Specifically, the output values of $f(x ; \bbF_p)$ can be exactly characterized as follows: for $x\in\bbF_p$,
\begin{align}
    f(x ; \bbF_p) &=  \begin{cases}
    0 \; &\text{if} \;  x < 1,    \\ 
    1 \; &\text{if} \;  1 \le x < 1.1 \times 2^1, \\ 
    1 + (-1)^{n_x+1} \; &\text{if} \;  1.1 \times 2^1  \le x < 1.01 \times 2^2 , \\ 
    0 \; &\text{if} \;  x \ge 1.01 \times 2^2, 
\end{cases} \label{eq:ex_indc}
\end{align}
where $n_x = (x-1.1 \times 2^1)\times 2^{p-1} \in \bbN$ for $1.1 \times 2^1  \le x < 1.01 \times 2^2$.
Here, one can observe that $f(x;\bbF_p)$ becomes $0$ for $x \ge 1.01 \times 2^2$; for $1.1 \times 2^1  \le x < 1.01 \times 2^2$, it oscillates between $0$ and $2$.
Hence, unlike the exact operation case, deriving the $\relu$ network results from $\step$ network results is non-trivial under floating-point operations. 
We present the formal derivation of \cref{eq:ex_indc} in \cref{proof:ex_indc}.

Nevertheless, we observe that \cref{eq:naiverelu} under $\bbF_p$ behaves like an indicator function for $x$ close to $z$. Using this property, we implement an indicator function as follows: using a $\relu$ network, we first map an input $x$ to $x'$ in some local neighborhood of $z$ such that $x'\ge z$ if and only if $x\ge z$; we then apply \cref{eq:naiverelu} to $x'$.
This construction of an indicator function via $\relu$ networks enables us to show memorization and universal approximation properties of $\relu$ networks as in the following theorems. The proofs of \cref{thm:memrelufp-temp,thm:univrelufp-temp} are presented in \cref{sec:pfthm:memrelufp-temp,sec:pfthm:univrelufp-temp}, respectively. %
\begin{theorem}\label{thm:memrelufp-temp}
For any $p\in\bbN$ and $d\in[2^{p}]$, there exists a $\relu$ network $f_{\theta}(\,\cdot\,;\bbF_p):\bbF_p^d\to\bbF_p$ of {$4$ layers and $20dn+2n$ parameters} satisfying the following: for any $\mathcal D=\{(\bfz_1,y_1),\dots,(\bfz_n,y_n)\}\subset\bbF_p^d\times\bbF_p$ with $\bfz_i\ne\bfz_j$ for all $i\ne j$, there exists $\theta_{\mathcal D}$ such that
$$f_{\theta_{\mathcal D}}(\bfz_i)=y_i\quad\forall i\in[n].$$
\end{theorem}
\begin{theorem}\label{thm:univrelufp-temp}
For any $p\in\bbN$, $d\in[2^{p}]$, $f^*\in C([0,1]^d,\bbR)$, and $\varepsilon>0$, there exists a $\step$ network $f(\,\cdot\,;\bbF_p):\bbF_p^d\to\bbF_p$ of $4$ layers and $(20d+2)K^d$ parameters %
where  $K=\min\{k\in\bbN:(\omega_{f^*}^{-1}(\varepsilon))^{-1}\le k\}$ such that 
$$|f_\theta(x;\bbF_p)-f^*(x)|\le|f^*(x;\bbF_p)-\round{f^*(x)}|+\varepsilon.$$
\end{theorem}
We note that the numbers of parameters in networks in \cref{thm:memrelufp-temp,thm:univrelufp-temp} coincide with that of $\step$ network results (\cref{thm:memstepfp-temp,thm:univstepfp-temp}) up to constant multiplicative factors, i.e., they also correspond to the necessary and sufficient number of parameters in classical results under exact operations.
Further, the following analogue of \cref{cor:univstepfp-temp} also holds for $\relu$ networks.
\begin{corollary}\label{cor:univrelufp-temp}
For any $p\in\bbN$, $d\in[2^{p}]$, a compact set $\mathcal K\subset\bbR^d$ such that $|\mathcal K\cap\bbF_p^d|<\infty$, $f^*\in C(\mathcal K,\bbR)$, and $\varepsilon>0$, there exists a $\relu$ network $f(\,\cdot\,;\bbF_p):\bbF_p^d\to\bbF_p$ of {$4$} layers and {$20d|\mathcal K\cap\bbF_p^d|+2|\mathcal K\cap\bbF_p^d|$} parameters %
such that 
$$|f_\theta(\bfx;\bbF_p)-f^*(\bfx)|=|f^*(\bfx;\bbF_p)-\round{f^*(\bfx)}|\quad\forall\bfx\in[0,1]^d\cap\bbF_p^d.$$
\end{corollary}

\section{Expressive power of neural networks under $\bbF_{p,q}$}\label{sec:fpqr}
We now consider a more practical setup: when all computations are done in $\bbF_{p,q}$ (i.e., the set of floating-point numbers with $p$-bit significand and $q$-bit exponent),
where $p,q\in\bbN$ satisfy $q\ge5$ and $4\le p\le 2^{q-2}+2$.
As noted in \cref{sec:float}, this class of $\bbF_{p,q}$ covers many popular floating-point formats such as those defined in the IEEE 754 standard \cite{ieee19} and bfloat16 \cite{tensorflow}.
Specifically, we demonstrate that $\step$ networks and $\relu$ networks under $\bbF_{p,q}$ also exhibit the memorization and universal approximation properties with the same number of parameters as in the $\bbF_p$ case. The proofs of all results in this section are presented in \cref{sec:pfresults-fpqr}.

\subsection{$\step$ network results} For $\step$ networks, we can prove similar results as in the $\bbF_p$ case; we construct an indicator function for a high-dimensional cube and derive the following memorization and universal approximation results. 
Namely, $\step$ networks are expressive under realistic floating-point operations.
The proofs of \cref{thm:memstepfpqr-temp,thm:univstepfpqr-temp} are presented in \cref{sec:pfthm:memstepfpqr-temp,sec:pfthm:univstepfpqr-temp}, respectively.
\begin{theorem}\label{thm:memstepfpqr-temp}
For any $d\in[2^{p}]$, there exists a $\step$ network $f_\theta(\,\cdot\,;\bbF_{p,q}):\bbF_{p,q}^d\to\bbF_{p,q}$ of $3$ layers and {$6dn+2n$} parameters 
satisfying the following: for any $\mathcal D=\{(\bfz_1,y_1),\dots,(\bfz_n,y_n)\}\subset\bbF_{p,q}^d\times\bbF_{p,q}$ with $\bfz_i\ne\bfz_j$ for all $i\ne j$, there exists $\theta_{\mathcal D}$ such that
$$f_{\theta_{\mathcal D}}(\bfz_i)=y_i\quad\forall i\in[n].$$
\end{theorem}

\begin{theorem}\label{thm:univstepfpqr-temp}
For any $d\in[2^{p}]$,  $f^*\in C([0,1]^d,\bbR)$, and $\varepsilon\ge0$, there exists a $\step$ network $f(\,\cdot\,;\bbF_{p,q}):\bbF_{p,q}^d\to\bbF_{p,q}$ of {$3$} layers and at most {$(6d+2)K^d$} parameters 
where $K=\min\{k\in\bbN:\delta ^{-1}\le k\}$, $\delta = \max \{ \eta,  \omega_{f^\ast}^{-1}(\varepsilon) \}$ such that 
$$|f_\theta(\bfx;\bbF_{p,q})-f^*(x)|\le|f^*(\bfx;\bbF_{p,q})-\round{f^*(x)}|+\varepsilon\quad\forall\bfx\in[0,1]^d\cap\bbF_{p,q}^d.$$
Here, $\eta\defeq 2^{-p+e_{min}}$ denotes the smallest positive floating-point number as we defined in \cref{sec:float}.

\end{theorem}
Unlike \cref{thm:univstepfp-temp} under $\bbF_p$, we allow $\varepsilon=0$ in \cref{thm:univstepfpqr-temp}. This is because $[0,1]^d\cap\bbF_{p,q}^d$ is always finite, and hence, we can achieve $\varepsilon=0$ using the memorization result (\cref{thm:memstepfpqr-temp}) with a finite number of parameters as in \cref{cor:univstepfp-temp}. %
We note that the number of parameters in \cref{thm:memstepfpqr-temp,thm:univstepfpqr-temp} coincide with results under $\bbF_p$ up to a constant multiplicative factor.

\subsection{$\relu$ network results}
While $\step$ network results for $\bbF_{p,q}$ (\cref{thm:memstepfpqr-temp,thm:univstepfpqr-temp}) can be shown as in the $\bbF_p$ case, $\relu$ network results do not naturally follow from $\bbF_p$ to $\bbF_{p,q}$ due to the overflow (and underflow) in $\bbF_{p,q}$. 
For example, recall \cref{eq:naiverelu} 
\begin{align*}
\indc{x\ge z}=\relu\left(\frac1{z-z^-}(x-z^-)\right)-\relu\left(\frac1{z-z^-}(x-z)\right)
\end{align*}
and consider operations under $\bbF_{p,q}$.
If $z$ is small enough (e.g., $z=2^{e_{min}}$) and $x$ is large enough (e.g., $x=2^{e_{max}}$). Then, the computation of $\frac1{z-z^-}(x-z^-)$ under $\bbF_{p,q}$ can incur an overflow.
However, by carefully analyzing floating-point operations, we can implement the indicator function using a $\relu$ network of $O(1)$ parameters (see \cref{lem:reluindcfpqr}); using this we can also show that memorization and universal approximation are possible under $\bbF_{p,q}$.
We present the proofs of \cref{thm:memrelufpqr-temp,thm:univrelufpqr-temp} in \cref{sec:pfthm:memrelufpqr-temp,sec:pfthm:univrelufpqr-temp}, respectively.
\begin{theorem}\label{thm:memrelufpqr-temp}
For any $d\in[2^p]$ and for $\kappa=(2-u) \times 2^{-2-p+e_{max}}$, there exists a $\relu$ network $f_\theta(\,\cdot\,;\bbF_{p,q}):(\bbF_{p,q}\cap[-\kappa,\kappa])^d\to\bbF_{p,q}$ of $4$ layers and {$20dn+5n$} parameters 
satisfying the following: for any $\mathcal D=\{(\bfz_1,y_1),\dots,(\bfz_n,y_n)\}\subset(\bbF_{p,q}\cap[-\kappa,\kappa])^d\times\bbF_{p,q}$ with $\bfz_i\ne\bfz_j$ for all $i\ne j$, there exists $\theta_{\mathcal D}$ such that
$$f_{\theta_{\mathcal D}}(\bfz_i)=y_i\quad\forall i\in[n].$$
Furthermore, computing $f_\theta(\bfx;\bbF_{p,q})$ does not incur overflow for all $\bfx\in\bbF_{p,q}^d$ satisfying $\|\bold{x}\|_{\infty} \le (2-u) \times 2^{-3+2^{q-2}}$. 
\end{theorem}

\begin{theorem}\label{thm:univrelufpqr-temp}
For any $d\in[2^{p}]$, $f^*\in C([0,1]^d,\bbR)$, and $\varepsilon\ge0$, there exists a $\relu$ network $f(\,\cdot\,;\bbF_{p,q}):\bbF_{p,q}^d\to\bbF_{p,q}$ of 4 layers and at most $(20d+2)K^d$ parameters where  $K=\min\{k\in\bbN:\delta^{-1}\le k\}$, $\delta = \max \{ \eta,  \omega_{f^\ast}^{-1}(\varepsilon) \}$ such that
$$|f(\bold{x};\bbF_{p,q})-f^*(\bold{x})|\le| f^*(\bold{x};\bbF_{p,q})-\round{f^*(x)} | +\varepsilon\quad\forall\bold{x}\in[0,1]^d\cap\bbF_{p,q}^d.$$
Here, $\eta\defeq 2^{-p+e_{min}}$ denotes the smallest positive floating-point number as we defined in \cref{sec:float}.
\end{theorem}
In \cref{thm:memrelufpqr-temp}, the supremum norm of an input to memorize is bounded by $\kappa=(2-u) \times 2^{-2-p+e_{max}}$; we use this condition to bypass the overflow during intermediate computations in our network construction.

\begin{remark}
To check whether the condition in \cref{thm:memrelufpqr-temp} can be satisfied in practice, we evaluate the $B_x\defeq(2-u) \times 2^{-3+2^{q-2}}$ and $\kappa \defeq (2-u) \times 2^{-2-p+e_{max}}$ for popular floating-point formats including IEEE 754 formats and bfloat16. For the 16-bit half-precision format, we have $B_x=63.969$ and $\kappa=15.992$. For the 32-bit single-precision format, we have $B_x=4.6117 \times 10^{18}$ and $\kappa=1.01141 \times 10^{31}$. For the 64-bit double-precision format, we have $B_x=3.3520\times 10^{153}$ and $\kappa=9.9980 \times 10^{291}$. For the 128-bit quadruple-precision format, we have $B_x=1.3634\times 10^{2465}$ and $\kappa=2.8642 \times 10^{4897}$. For the bfloat16 format, we have $B_x=4.5938\times 10^{18}$ and $\kappa=6.6202 \times 10^{35}$. This implies that the condition in \cref{thm:memrelufpqr-temp} can be often satisfied in many floating-point number formats and inputs except for half-precision. In the 16-bit half-precision format, we need to carefully bound the input data to have the universal approximation property using \cref{thm:memrelufpqr-temp}. On the other hand, bfloat16 format has larger $B_x$ and $\kappa$, i.e., bfloat16 can be better than the half-precision format in terms of the universal approximation.
\end{remark}

In \cref{thm:univrelufpqr-temp}, we also allow $\varepsilon=0$ as in \cref{thm:univstepfpqr-temp}.
As in our previous results, \cref{thm:memrelufpqr-temp,thm:univrelufpqr-temp} use $O(n)$ parameters and $O((\omega_{f^*}^{-1}(\varepsilon))^{-1})$ parameters, respectively.

\section{Proofs of results under $\bbF_p$}\label{sec:pfresults-fp}
\subsection{Technical lemmas}
\label{subsec:techincal_fp}
We present several technical lemmas regarding the computation of operations in $\bbF_p$. \cref{lem:sterbenz} is the well-known Sterbenz's lemma \cite[Theorem 4.3.1]{sterbenz73} (also in \cite[Lemma 4.1]{muller18}). 
In particular, \cref{lem:representable_fp,lem:exact_fp,lem:integerexact_fp,lem:ignore_fp} have 
$\bbF_{p,q}$ versions analogous in \cref{sec:pfresults-fpqr}.

For $x \in \bbF_{p}$, we represent $x$ as 
$$x = s_x \times a_x \times 2^{e_x}, \; s_x \in \{-1,1\},  a_x = 1.x_1\cdots x_p.  $$ 
For $x \in \bbF_p$, we define $\mu(x)$ as 
$$ \mu(x) \defeq  \inf \{  m \in \bbZ : x \times 2^{- m }  \in \bbZ \}.  $$
Note that if $x \ne 0$, we can represent $x$ as
\begin{align}
    x= (n_x \times 2^{-e_x+\mu(x)} ) \times 2^{\frak{e}_x}, \label{eq:murepr_fp}
\end{align}
 for $n_x  = x \times 2^{-\mu(x)} \in \bbN$ with $ 2^{e_x- \mu(x)} \le n_x < 2^{1+ e_x- \mu(x) }$. \\ 
For $x \in \bbR$, we define the ceiling function $\ceilZ{x}$ as 
$$ \ceilZ{x} = \min\{ n \in \bbZ : n \ge x \}. $$

We now formally introduce the statements of technical lemmas used for proving our results under $\bbF_p$.
\begin{lemma}[Sterbenz's lemma] \label{lem:sterbenz}
    Let $x,y \in \bbF_p$. If  $0 \leq y/2 \leq x \leq y$, then $x \ominus y$ and $y \ominus x$ are exact.
\end{lemma}

\begin{lemma}\label{lem:representable_fp}
    Let $x = n \times 2^{m}$ for some $n \in \bbN, \; m \in \bbZ$.  If 
    $ 0 < n < 2^{1+p}, $ then 
    $x$ is representable by $\bbF_{p}$.
\end{lemma}
\begin{proof}
 Since $ 0 \le n < 2^{1+p}$, there exists $c_0 \in  \{ 0 \} \cup [p]$ such that $2^{p-c_0} \le n < 2^{1+p-c_0}$.   
Note that $x$ has the following representation in $\bbF_{p}$,
\begin{align*}
     &x =  (n \times 2^{-p+c_0}) \times 2^{p-c_0+m}, \; 1 \le n \times 2^{-p+c_0} <  2,  
\end{align*}
Then  we express $n \times 2^{-p+c_0} $ as 
$$  n \times 2^{-p+c_0} = 1.\underbrace{w_{1} \cdots w_{p-c_0}}_{p-c_0 \; \text{times}},  $$ for some $w_1 , \dots , w_{p-c_0} \in \{0,1\}. $ 
Therefore $x$ is representable by $\bbF_{p}$. 
\end{proof}

\begin{lemma}\label{lem:exact_fp}
Let $x,y \in \bbF_p$ and $s_x=s_y$ and $e_x \le e_y$.
If $\mu(x) \ge e_y-p $, then $x \ominus y$ and $y \ominus x$ are exact.
In addition, if $ |x + y| \le 2^{1+e_y }$, then $x \oplus y$ is exact. 
\end{lemma}
\begin{proof}

Let $ k \defeq \mu(x)- e_y+p \ge 0$.
As described in \cref{eq:murepr_fp}, we can represent $x$ and $y$ as 
\begin{align*}
    &x = (n_x \times 2^{- e_x + \mu(x)})  \times 2^{e_x} = n_x \times 2^{ \mu(x)} ,\;  2^{e_x - \mu(x)} \le n_x < 2^{1+e_x - \mu(x)}, \\
    &y = (n_y \times 2^{-p+c_y}) \times 2^{e_y} = n_y \times 2^{-p+e_y} ,\;  2^{p-c_y} \le n_y < 2^{1+p},
\end{align*}
for some $n_x,n_y \in \bbN$. 
Since $ k \ge 0$, we have 
$$x = (2^k n_x) \times 2^{-p + e_y} , \;   2^{p-(e_y -e_x+c_y) } \le 2^k n_x < 2^{1+p-(e_y -e_x) }.$$
Therefore for $n'  = n_y - 2^k n_x \in \bbN$, we have
$$ y - x = n' \times 2^{-p + e_y}, \;  2^{p} - 2^{1+p-(e_y -e_x) } <  n' <  2^{1+p} - 2^{p-(e_y -e_x) }$$
which leads to 
$$ -2^{p} <  n' <  2^{p+1} - 1. $$
Since $|n'| =0$ or $|n'| <   2^{1+p} - 1$, by \cref{lem:representable_fp}, $ y - x =n' \times  2^{-p + e_y}$ is representable by $\bbF_{p}$. This ensures that $y \ominus x$ is exact. \\
Now suppose $| x + y | \le 2^{1+e_y}$. Then for $n'' = n_x + n_y \in \bbN$, we have 
$$  x + y  = n'' \times 2^{-p + e_y}, \;  2^{p-(e_y -e_x ) } + 2^{p}  \le n'' <   2^{1+p-(e_y -e_x) } +2^{1+p}. $$
Since $| x + y | \le 2^{1+e_y}$, we have $n'' \le 2^{1+p}$. If $|n''| =0$ or $|n''| <   2^{1+p} $, by \cref{lem:representable_fp}, $ x + y =n'' \times  2^{-p + e_y}$ is representable by $\bbF_{p}$. If $|n''| = 2^{1+p}$, $x+y = 2^{1+e_y}$ is obviously representable by $\bbF_{p}$. Therefore, $x \oplus y$ is exact.
\end{proof}
\begin{lemma}\label{lem:integerexact_fp}
    In $\bbF_p$, if $x , y \in [2^{1+p}]$, $ x \ominus y$ and $y \ominus x$ are exact. In addition if $x,y \in [2^{p}]$, then $x \oplus y$ is exact. 
\end{lemma}
\begin{proof}
Without loss of generality, suppose $x \le y$ for $x,y \in [2^{1+p}]$. Since $y - x < 2^{1+p}$, $ y - x$ is representable by $\bbF_{p}$ by \cref{lem:representable_fp}, ensuring that $ x \ominus y$ and $y \ominus x$ are exact. \\
In addition, suppose $x ,y \in [2^p]$ and $x \le y$. Since $\mu(x) \ge 0$ and $e_y \le p$, we have $\mu(x) \ge -p +e_y $. Since $ | x + y | < 2^{1+p}$, $x \oplus y$ is exact by \cref{lem:exact_fp}.
\end{proof}

\begin{lemma}\label{lem:ignore_fp}
    Let $x,y \in \bbF_p$ and $e_x \le e_y$.
If $\mu(x) \le e_y-p-2$, then $y \oplus x $  = $ y \oplus x = y$.  
\end{lemma}
\begin{proof}
Since $|x| \le 2^{-2-p+e_y}$, we have $\round{x+y}=\round{-x+y}=y$. 
\end{proof}

\begin{lemma}\label{lem:seqconv}
For any $d,p\in\bbN$, let $\alpha_1,\beta_1,\dots,\alpha_d,\beta_d\in\bbF_p$ such that $\alpha_i\le\beta_i$ for all $i\in[d]$. 
Then, for any sequence $(\bfx_i)_{i\in\bbN}$ in $([\alpha_1,\beta_1]\times\cdots\times[\alpha_d,\beta_d])\cap\bbF_p^d$, there exists a subsequence of $(\bfx_i)_{i\in\bbN}$ that converges to some $\bfz^*\in([\alpha_1,\beta_1]\times\cdots\times[\alpha_d,\beta_d])\cap\bbF_p^d$.
\end{lemma}
\begin{proof}
Let $\mathcal S=[\alpha_1,\beta_1]\times\cdots\times[\alpha_d,\beta_d]$. 
Since $\mathcal S$ is compact, there exists a convergent subsequence $(\bfz_i)_{i\in\bbN}$ of $(\bfx_i)_{i\in\bbN}$ and $(\bfz_i)_{i\in\bbN}$ converges to some point, say $\bfz^*=(z_1^*,\dots,z_d^*)$, in $\mathcal S$. 
However, if $z_j^*\notin\bbF_p$ for some $j\in[d]$, then $\inf_{y\in\bbF_p}|z_j^*-y|>0$
by the definition of $\bbF_p$, i.e., $\bfz_1,\bfz_2,\dots\in\bbF_p^d$ cannot converge to $\bfx^*$.
Hence, $\bfz^*\in\mathcal S\cap\bbF_p^d$ and this completes the proof.
\end{proof}

\begin{lemma}\label{lem:bestrounding}
For any $d,p\in\bbN$, let $\alpha_1,\beta_1,\dots,\alpha_d,\beta_d\in\bbF_p$ such that $\alpha_i\le\beta_i$ for all $i\in[d]$. 
Then, for any continuous function $f:[\alpha_1,\beta_1]\times\cdots\times[\alpha_d,\beta_d]\to\bbR$, there exists $\bfx^*\in([\alpha_1,\beta_1]\times\cdots\times[\alpha_d,\beta_d])\cap\bbF_p^d$ such that
$$|f(\bfx^*)-\round{f(\bfx^*)}|=\inf_{\bfx\in([\alpha_1,\beta_1]\times\cdots\times[\alpha_d,\beta_d])\cap\bbF_p^d}|f(\bfx)-\round{f(\bfx)}|.$$
\end{lemma}
\begin{proof}
Let $\mathcal S=([\alpha_1,\beta_1]\times\cdots\times[\alpha_d,\beta_d])\cap\bbF_p^d$.
Suppose for a contradiction that such $\bfx^*$ does not exist, i.e., there is no $\bfx\in\mathcal S$ such that $f(\bfx)\in\bbF_p$.
Then, there exists a sequence $(\bfx_i)_{i\in\bbN}$ in $\mathcal S$ such that $(|f(\bfx_i)-\round{f(\bfx_i)}|)_{i\in\bbN}$ is strictly decreasing and 
\begin{align}
|f(\bfx_i)-\round{f(\bfx_i)}|\to\inf_{\bfx\in\mathcal S}|f(\bfx)-\round{f(\bfx)}|\label{eq:pflem:bestrounding}
\end{align}
as $i\to\infty$.
By \cref{lem:seqconv}, there exists a subsequence $(\bfz_i)_{i\in\bbN}$ of $(\bfx_i)_{i\in\bbN}$ that converges to some point $\bfz^*\in\mathcal S$. Since $f(\bfz^*)\notin\bbF_p$, by the definition of $\bbF_p$, we have 
$$f(\bfz^*)^{(\le,\bbF_p)}<f(\bfz^*)<f(\bfz^*)^{(\ge,\bbF_p)}.$$

Let $\varepsilon=\min\{f(\bfz^*)-f(\bfz^*)^{(\le,\bbF_p)},f(\bfz^*)^{(\ge,\bbF_p)}-f(\bfz^*)\}$. We consider two cases to show the contradiction: $\varepsilon=(f(\bfz^*)^{(\ge,\bbF_p)}-f(\bfz^*)^{(\le,\bbF_p)})/2$ and $\varepsilon\ne(f(\bfz^*)^{(\ge,\bbF_p)}-f(\bfz^*)^{(\le,\bbF_p)})/2$.
Suppose $\varepsilon=(f(\bfz^*)^{(\ge,\bbF_p)}-f(\bfz^*)^{(\le,\bbF_p)})/2$, i.e., $f(\bfz^*)=(f(\bfz^*)^{(\ge,\bbF_p)}+f(\bfz^*)^{(\le,\bbF_p)})/2$. Let $(\bfy_i)_{i\in\bbN}$ be a subsequence of $(\bfz_i)_{i\in\bbN}$ such that $|f(\bfz^*)-f(\bfy_i)|\le\varepsilon$ for all $i\in\bbN$ and $(|f(\bfz^*)-f(\bfy_i)|)_{i\in\bbN}$ is strictly decreasing; such a subsequence always exist due to the continuity of $f$ and our choice of $(\bfz_i)_{i\in\bbN}$. 
Then, one can observe that $(|f(\bfy_i)-\round{f(\bfy_i)}|)_{i\in\bbN}$ is monotonically increasing as $\round{f(\bfy_i)} \in\{f(\bfz^*)^{(\le,\bbF_p)},f(\bfz^*)^{(\ge,\bbF_p)}\}$ and $f(\bfy_i)\to f(\bfz^*)=(f(\bfz^*)^{(\le,\bbF_p)}+f(\bfz^*)^{(\ge,\bbF_p)})/2$. 
Namely, $(|f(\bfy_i)-\round{f(\bfy_i)}|)_{i\in\bbN}$ is not strictly decreasing although $(\bfy_i)_{i\in\bbN}$ is a subsequence of $(\bfx_i)_{i\in\bbN}$. 
This contradicts the assumption that $(|f(\bfx_i)-\round{f(\bfx_i)}|)_{i\in\bbN}$ is strictly decreasing.

We now consider the case that $\varepsilon\ne(f(\bfz^*)^{(\ge,\bbF_p)}-f(\bfz^*)^{(\le,\bbF_p)})/2$.
Without loss of generality, suppose $f(\bfz^*)-f(\bfz^*)^{(\le,\bbF_p)}<f(\bfz^*)^{(\ge,\bbF_p)}-f(\bfz^*)$.
Let 
$$\delta=\frac{f(\bfz^*)^{(\ge,\bbF_p)}-f(\bfz^*)^{(\le,\bbF_p)}}2-\varepsilon = \frac{f(\bfz^*)^{(\ge,\bbF_p)}+f(\bfz^*)^{(\le,\bbF_p)}}2 - f(\bfz^*) >0$$
and choose a subsequence $(\bfy'_i)_{i\in\bbN}$ of $(\bfz_i)_{i\in\bbN}$ such that $|f(\bfz^*)-f(\bfy'_i)|<\delta$ for all $i\in\bbN$ and $(|f(\bfz^*)-f(\bfy'_i)|)_{i\in\bbN}$ is strictly decreasing. Then, since $f(\bfz^\ast)$ and $f(\bfy_i')$ are closer to $f(\bfz^*)^{(\le,\bbF_p)}$ than $f(\bfz^*)^{(\ge,\bbF_p)}$,  we have $\round{f(\bfz^\ast)} = \round{f(\bfy_i')} =  f(\bfz^*)^{(\le,\bbF_p)}$. 
Namely, 
we must have $f(\bfy'_i)\ge f(\bfz^*)$ for all $i\in\bbN$ to have strictly decreasing $(|f(\bfy'_i)-\round{f(\bfy_i)}|)_{i\in\bbN}$. However, this implies that 
$$|f(\bfy'_i)-\round{f(\bfy_i)}|\ge|f(\bfz^*)-\round{f(\bfz^*)}|$$
for all $i\in\bbN$, which implies
$$|f(\bfz^*)-\round{f(\bfz^*)}|=\inf_{\bfx\in\mathcal S}|f(\bfx)-\round{f(\bfx)}|$$
from our choice of $(\bfx_i)_{i\in\bbN}$ and \cref{eq:pflem:bestrounding}.
This contradicts the assumption that $\bfx^*$ in the statement of the lemma does not exist and completes the proof.
\end{proof}

\begin{lemma}\label{lem:reluindc-temp}  
There exist $\relu$ networks $\psi_\theta(\,\cdot\,;\bbF_p):\bbF_p\to\bbF_p$ of $3$ layers and {$5$} parameters that satisfies the following: 
for any $z\in\bbF_p\setminus\{0\}$, there exist $\theta_{1,1,z},\theta_{1,2,z},\theta_{2,1,z},\theta_{2,2,z}\in\bbF_p^{5}$ such that for any $x\in\bbF_p$,
\begin{align*}
\psi_{\theta_{1,1,z}}(x;\bbF_p)\oplus \psi_{\theta_{1,2,z}}(x;\bbF_p)&=\indc{x\ge z},\\
\psi_{\theta_{2,1,z}}(x;\bbF_p)\oplus \psi_{\theta_{2,2,z}}(x;\bbF_p)&=\indc{x\le z}.
\end{align*}
In addition, $\psi_{\theta_{i,1,z}}(x;\bbF_p), -\psi_{\theta_{i,2,z}}(x;\bbF_p) \in \{0 \} \cup [2^p]$ for all $i\in\{1,2\}$.
\end{lemma}
\begin{proof}
We first consider $z > 0$. 
To construct $\indc{x \ge z}$, we define a three-layer $\relu$ network $f_1(x;\bbF_p)$ as follows:
\begin{align}
    & f_1(x;\bbF_p) = \psi_{\theta_{1,1,z}}(x;\bbF_p) \oplus \psi_{\theta_{1,2,z}}(x;\bbF_p), \nonumber\\
    & \psi_{\theta_{1,1,z}}(x;\bbF_p) = \phi_{\theta_{1,1,z}}(x;\bbF_p), \;  \psi_{\theta_{1,2,z}}(x;\bbF_p) = \phi_{\theta_{1,2,z}}(x;\bbF_p), \nonumber\\ 
    &\phi_{\theta_{1,1,z}}(x;\bbF_p)= 2^{\tilde{c}} \otimes \relu\big( (-2^{p-e_z} \otimes g(x) )  \oplus (2-a_z - \tilde{u}) \times 2^p   \big), \nonumber \\
&\phi_{\theta_{1,2,z}}(x;\bbF_p)= -2^{\tilde{c}} \otimes \relu\big( (-2^{p-e_z} \otimes g(x) )  \oplus (2-a_z - u) \times 2^p   \big),  \nonumber
\end{align}
where 
\begin{align*}
g(x)&= \relu( -x \oplus (2-u) \times 2^{e_z} ),\\
\tilde{u}&= \begin{cases}
    0 \; &\text{if}\; a_z \ne 1, \\
    2^{-1-p} \; &\text{if}\; a_z = 1, \\
\end{cases},\\
\tilde{c} &= \begin{cases}
    0 \; &\text{if}\; a_z \ne 1, \\
    1 \; &\text{if}\; a_z = 1. \\
\end{cases} 
\end{align*}
Note that $z^- = (a_z - u + \tilde{u}) \times 2^{e_z}$.

If $(2-u)\times   2^{-1+e_z} \le x \le (2-2u) \times 2^{e_z}$, then $-x \oplus (2-u) \times 2^{e_z}$ is exact by \cref{lem:sterbenz}. Therefore, we have
\begin{align*}
   g(x)
   \begin{cases}
         \ge  (2-u) \times 2^{-1+e_z} \; &\text{if}\;  x \le (2-2u) \times 2^{-1+e_z}, \\
        = -x +  (2-u)\times 2^{e_z}  \; &\text{if}\; (2-u)\times   2^{-1+e_z} \le x \le (2-2u) \times 2^{e_z}, \\
        =0 \; &\text{if}\; x \ge  (2-u) \times 2^{e_z}.
    \end{cases}
\end{align*}
If $x \le (2-2u) \times 2^{-1+e_z}$, since we have 
\begin{align*}
    &(2-a_z-\tilde{u}) \le (1-2^{-1-p}) = (2-u) \times 2^{-1},
\end{align*}
and
\begin{align*}
    (-2^{p-e_z} \otimes g(x) ) &\le  (-2^{p-e_z} \times (2-u) \times 2^{-1+e_z} )\\
    &= - (2-u)\times 2^{-1+p} \le - (2-a_z -\tilde{u}) \times 2^p, 
\end{align*}
it holds that $\phi_{\theta_{1,1,z}}(x;\bbF_p) = \phi_{\theta_{1,2,z}}(x;\bbF_p)=0$. 

If $(2-u)\times   2^{-1+e_z} \le x \le (2-2u) \times 2^{e_z}$, we have $2^{p-e_z} \otimes g(x)  = n_1  $,  where $ n_1 = 1,2, ... , 2^{p} -1$, leading $ 2^{p-e_z} \otimes g(x) \in [2^p]$. Since 
$(2-a_z - \tilde{u}) \times 2^p, (2-a_z - u) \times 2^p \in [2^p]$,  by \cref{lem:integerexact_fp}, all operations in $\phi_{\theta_{1,1,z}}(x;\bbF_p)$ and $\phi_{\theta_{1,2,z}}(x;\bbF_p)$ are exact. Hence
\begin{align*}
    \phi_{\theta_{1,1,z}}(x;\bbF_p) &= 2^{\tilde{c}} \times \relu(  2^{p-e_z} x - (a_z-u+\tilde{u}) \times 2^p ), \\
    \phi_{\theta_{1,2,z}}(x;\bbF_p) &= -  2^{\tilde{c}} \times \relu(  2^{p-e_z} x - a_z \times 2^p ), 
\end{align*}
with $\phi_{\theta_{1,1,z}}(x;\bbF_p), -\phi_{\theta_{1,2,z}}(x;\bbF_p) \in \{0 \} \cup [2^p]$.  
Therefore
$$ f_1(x;\bbF_p) = \begin{cases}
    0 \; &\text{if} \; (2-u)\times   2^{-1+e_z} \le x \le z^{-},\\
    1 \; &\text{if} \; z \le x \le (2-2u) \times 2^{e_z}.
\end{cases} $$

If $x \ge ( 2-u) \times 2^{e_z} $, we have 
\begin{align*}
    &\phi_{\theta_{1,1,z}}(x;\bbF_p)  = 2^{\tilde{c}} \big( (2-a_z - \tilde{u}) \times 2^p \big) \\
    &\phi_{\theta_{1,2,z}}(x;\bbF_p)  = -2^{\tilde{c}} \big( (2-a_z - u ) \times 2^p \big) \\
    &f_1(x;\bbF_p) = \phi_{\theta_{1,1,z}}(x;\bbF_p) \oplus \phi_{\theta_{1,2,z}}(x;\bbF_p) = 2^{\tilde{c}}( u - \tilde{u}) = 1,
\end{align*}
with $ \phi_{\theta_{1,1,z}}(x;\bbF_p), -\phi_{\theta_{1,2,z}}(x;\bbF_p) \in \{0 \} \cup [2^p] $.  
Therefore we conclude 
\begin{align*}
    f_1(x;\bbF_p) = \indc{x \ge z},
\end{align*}
for all $x\in\bbF_p$.

We also construct  $\indc{x \le z}$ in a similar way. To this end, we define $f_2(x;\bbF_p)$ as 
\begin{align}
     &f_2(x;\bbF_p) = \psi_{\theta_{2,1,z}}(x;\bbF_p) \oplus \psi_{\theta_{2,2,z}}(x;\bbF_p), \nonumber \\ %
     & \psi_{\theta_{2,1,z}}(x;\bbF_p) = \phi_{\theta_{2,1,z}}(x;\bbF_p), \;  \psi_{\theta_{2,2,z}}(x;\bbF_p) = \phi_{\theta_{2,2,z}}(x;\bbF_p), \nonumber\\ 
 &\phi_{\theta_{2,1,z}}(x;\bbF_p)=1 \otimes \relu \big( ( -2^{p - e_z} \otimes \relu(x \ominus 2^{e_z}))  \oplus ( (a_z-1 +u) \times 2^{p}  ) \big) , \nonumber \\ 
     & \phi_{\theta_{2,2,z}}(x;\bbF_p) = -1 \otimes \relu \big( (-2^{p-e_z} \otimes \relu(x \ominus 2^{e_z}) )  \oplus (a_{z}-1) \times 2^{p}   \big). \nonumber   
\end{align}
If $2^{e_z} <  x < 2^{1+e_z}$, $x \ominus 2^{e_z}$ is exact by \cref{lem:sterbenz}. Therefore we have
\begin{align*}
    \relu(x \ominus 2^{e_z})  \begin{cases}
        =0 \; &\text{if}\; x < 2^{e_z} \\
        =x-2^{e_z} \; &\text{if}\; 2^{e_z} \le  x \le 2^{1+e_z}. \\
        \ge  2^{e_z}  \; &\text{if}\;   x >  2^{1+e_z}.
    \end{cases}
\end{align*}
 If $ x < 2^{e_z}$, we have 
 \begin{align*}
     &\phi_{\theta_{2,1,z}}(x;\bbF_p)  = \relu( (a_z -1+u) \times 2^{p} ) \\
     &\phi_{\theta_{2,2,z}}(x;\bbF_p) = - \relu( (a_{z}-1) \times 2^{p} ) \\
     &f_2(x;\bbF_p) =  \phi_{\theta_{2,1,z}}(x;\bbF_p) + \phi_{\theta_{2,2,z}}(x;\bbF_p) = 1,
 \end{align*}
with $  \phi_{\theta_{2,1,z}}(x;\bbF_p), -\phi_{\theta_{2,2,z}}(x;\bbF_p) \in \{0 \} \cup [2^p].$ \\
 If $2^{e_z} \le x \le 2^{1+e_z}$, we have $   \relu(x-2^{e_z})  = n_2 \times 2^{-p+e_z}$ where $n_2=0,1,\dots,2^{p}$. Since $2^{p-e_z} \otimes \relu(x \ominus 2^{e_z} ) \in [2^p]$ and $(a_z-1 +u) \times 2^{p}, (a_z-1) \times 2^{p} \in [2^p]$, by \cref{lem:integerexact_fp}, all operations in $\phi_{\theta_{2,1,z}}(x;\bbF_p)$ and $\phi_{\theta_{2,2,z}}(x;\bbF_p)$ are exact. Hence 
 \begin{align*}
     \phi_{\theta_{2,1,z}}(x;\bbF_p) &= \relu( -2^{p-e_z}x + (a_z+u) \times 2^p), \\ 
  \phi_{\theta_{2,2,z}}(x;\bbF_p) &= -\relu( -2^{p-e_z}x + a_z \times 2^p),
 \end{align*}
 with $  \phi_{\theta_{2,1,z}}(x;\bbF_p), -\phi_{\theta_{2,2,z}}(x;\bbF_p) \in \{0 \} \cup [2^p].$ 
Therefore, 
$$ f_2(x;\bbF_p) = \begin{cases}
    1 \; &\text{if} \; 2^{e_z} \le x \le z,\\
    0 \; &\text{if} \; z^{+} \le x \le 2^{1+e_z}.
\end{cases} $$
 If $x >  2^{1+e_z} $, we have 
 \begin{align*}
     &1 \ge a_z -1 + u,  \\ 
     &-2^{p-e_z} \otimes \relu(x \ominus 2^{e_z}) \le - 2^p  \le -(a_z-1+u) \times 2^p,
 \end{align*}
which leads to $ \phi_{\theta_{2,1,z}}(x;\bbF_p)=\phi_{\theta_{2,2,z}}(x;\bbF_p)=0$.
Therefore we conclude 
\begin{align*}
    f_2(x;\bbF_p) = \indc{x \le z}.
\end{align*}
Now, we consider $z < 0$, we define
\begin{align*}
&f_1(x;\bbF_{p})=\psi_{1,1,z}(x;\bbF_p) \oplus \psi_{1,2,z}(x;\bbF_p), \; f_2(x;\bbF_{p})=\psi_{2,1,z}(x;\bbF_p) \oplus \psi_{2,2,z}(x;\bbF_p), \\
&\psi_{1,1,z}(x;\bbF_p) = \phi_{2,1,-z}(-x;\bbF_p), \;\psi_{1,2,z}(x;\bbF_p) = \phi_{2,2,-z}(-x;\bbF_p), \\
&\psi_{2,1,z}(x;\bbF_p) = \phi_{1,1,-z}(-x;\bbF_p), \; \psi_{2,2,z}(x;\bbF_p) = \phi_{1,2,-z}(-x;\bbF_p).
\end{align*}
Then we have 
\begin{align*}
  &f_1(x;\bbF_{p})=\phi_{\theta_{2,1,-z}}(-x;\bbF_p) \oplus  \phi_{\theta_{2,2,-z}}(-x;\bbF_p) =  \indc{-x \le -z} = \indc{x \ge z}, \\
&f_2(x;\bbF_{p})=\phi_{\theta_{1,1,-z}}(-x;\bbF_p) \oplus  \phi_{\theta_{1,2,-z}}(-x;\bbF_p) =  \indc{-x \ge -z} = \indc{x \le z}.  
\end{align*}
Finally, we have $\psi_{\theta_{i,1,z}}(x;\bbF_p),  \psi_{\theta_{i,2,z}}(x;\bbF_p) \in \{0 \} \cup [2^p]$ for all $i\in\{1,2\}$, and
\begin{align*}
    &f_1(x;\bbF_{p}) = \indc{x \ge z}, \; f_2(x;\bbF_{p}) = \indc{x \le z}, 
\end{align*}
for all $z\in\bbF_p\setminus\{0\}$.

Lastly, it is easy to observe that $\psi_{\theta_{i,j,z}}$ for all $i,j\in\{1,2\}$ and $z\in\bbF_p\setminus\{0\}$ share the same network architecture of $3$ layers and $5$ parameters. This completes the proof.
\end{proof}
\subsection{Proof of \cref{lem:multidim_step-temp}}\label{sec:pflem:multidim_step-temp}
Let $\bfx=(x_1,\dots,x_d)$.
For each $i\in[d]$, we define $g_{2j-1}(x_j;\bbF_p)=\indc{-x_j\oplus\beta_j\ge0}$ and $g_{2j}(x_j;\bbF_p)=\indc{x_j\ominus\alpha_j\ge0}$,  and define $f_{\theta_{\bold{\alpha},\bold{\beta}}}$ as follows:

\begin{align*}
f_{\theta_{\bold{\alpha},\bold{\beta}}}(\bfx;\bbF_p)=\indc{\left(\bigoplus_{i=1}^{2d}g_i(x_{\ceilZ{i/2}};\bbF_p)\right)\ominus2d\ge0}. \label{eq:lem1}
\end{align*}

From the definition of $g_i$, one can observe that 
$$g_{2j-1}(x_j;\bbF_p)+g_{2j}(x_j;\bbF_p)=\begin{cases}
    2~&\text{if}~x_j\in[\alpha_j,\beta_j],\\
    1~&\text{if}~x_j\notin[\alpha_j,\beta_j].
\end{cases}$$ %
Since $\bigoplus_{i=1}^{m}g_i(x_{\ceilZ{i/2}};\bbF_p) \le m  \le 2^{1+p} - 1 $ for $1 \le m \le (2d - 1) $, and $|g_i(x_{\ceilZ{i/2}};\bbF_p)| = 0,1$, by \cref{lem:integerexact_fp}, $\bigoplus_{i=1}^{2d}g_i(x_{\ceilZ{i/2}};\bbF_p)$ is exact.
Therefore, we have $\bigoplus_{i=1}^{2d}g_i(x_{\ceilZ{i/2}};\bbF_p)=2d$ if $\bfx\in\prod_{i=1}^d[\alpha_i,\beta_i]$ and $\bigoplus_{i=1}^{2d}g_i(x_{\ceilZ{i/2}};\bbF_p)<2d$ otherwise, i.e.,
$$f_{\theta_{\alpha,\beta}}(\bfx,\bbF_p)=\indc{\bfx\in\prod_{j=1}^d[\alpha_j,\beta_j]}.$$
Here, $f_{\theta_{\bold{\alpha},\bold{\beta}}}$ can be constructed via a three-layer $\step$ network of $6d+2$ parameters ($4d$ parameters for the first layer, $2d+1$ parameters for the second layer, and $1$ parameter for the last layer.) 
This completes the proof.

\subsection{Proof of \cref{thm:memstepfp-temp}}\label{sec:pfthm:memstepfp-temp}
We use \cref{lem:multidim_step-temp} to prove \cref{thm:memstepfp-temp}.
Consider 
$$f_{\theta_{\bfz_i,\bfz_i}}(\bfx;\bbF_p)=\indc{\bfx\in\prod_{j=1}^d[\bfz_{i,j},\bfz_{i,j}]}=\indc{\bfx=\bfz_i},$$ and 
$$f_{\theta_{\mathcal{D}}}(\bfx;\bbF_p)=\bigoplus_{i=1}^n\left(y_i\otimes f_{\theta_{\bfz_i,\bfz_i}}(\bfx, \bbF_p)\right),$$
where $f_{\theta_{\bfz_i,\bfz_i}}(\bfx, \bbF_p)$ is from  \cref{eq:lem1} in \cref{lem:multidim_step-temp}.
Then, $f_{\theta_{\mathcal{D}}}(\bfz_i;\bbF_p)=y_i$ for all $i\in[n]$ and $f$ can be implemented by a three-layer $\step$ network of $6dn+2n$ parameters ($4dn$ parameters for the first layer, $2dn+n$ parameters for the second layer, and $n$ parameters for the last layer).

\subsection{Proof of \cref{eq:ex_indc}} \label{proof:ex_indc}

If $x < 1$, since $x \le (1-u)$, we have $ x \times 2^p \le (1-u) \times 2^p  < 2^p$. Therefore, $f(x;\bbF_p) = 0$. 
If $1 \le x \le 1.1 \times 2^1$, we have $(x \times 2^p)  \in [2^{1+p}]$,  $(1-u) \times 2^p , 2^p \in [2^p]$.
By \cref{lem:integerexact_fp}, $(x \times 2^p) \ominus ((1-u) \times 2^p)$ and $(x \times 2^p) \ominus 2^p$ are exact. Therefore, 
$$f(x;\bbF_p) =  \big( (x \times 2^p) - ((1-u) \times 2^p) \big) - \big( (x \times 2^p) - 2^p \big)  = 1.$$ 
If $ 1.1 \times 2^1 \le x < 1.01 \times 2^2$, let $n_x = (x-1.1 \times 2^1)\times 2^{p-1} \in \bbN$. Then we have $ 0 \le n_x < 2^{p}$ and  
\begin{align*}
(x \times 2^p) \ominus  ((1-u) \times 2^p) &= \round{ 2^{p+1} + 2n_x + 1 }, \\
(x \times 2^p) \ominus  2^p  &= \round{ 2^{p+1} + 2n_x }.
\end{align*}
For $0 \le n < 2^{p+1}$, since %
    \begin{align*}
        \round{2^{p+1} + n} = \begin{cases}
            2^{p+1} + n \; &\text{if} \; n\equiv 0,2  \Mod{4}, \\ 
            2^{p+1} + n - 1 \; &\text{if} \; n \equiv 1  \Mod{4}, \\ 
            2^{p+1} + n + 1 \; &\text{if} \; n \equiv 3   \Mod{4},  
        \end{cases}
    \end{align*}
we have 
\begin{align*}
     f(x;\bbF_p) = \begin{cases}
        0 \; &\text{if} \; n_x \equiv 0  \Mod{2}, \\ 
        2 \; &\text{if} \; n_x \equiv 1 \Mod{2}. \\ 
    \end{cases}
\end{align*}
If $x \ge 1.01 \times 2^2$, let $x' = x - 1$. Then we have $x' \ge 2^2$ and 
\begin{align*}
(x \times 2^p) \ominus ((1-u) \times 2^p) &= \round{ (x-1+u) \times 2^p  } = \round{x' + u} \times 2^p, \\ 
(x \times 2^p) \ominus  2^p = \round{ (x-1) \times 2^p  } &= x' \times 2^p.
\end{align*}
Since $\round{x' + u}=x'$ by \cref{lem:ignore_fp}, we have $ f(x;\bbF_p) =0$.

\subsection{Proof of \cref{thm:univstepfp-temp}}\label{sec:pfthm:univstepfp-temp}

If $\omega_{f^*}^{-1}(\varepsilon)=\infty$, let $\delta =1$, and we have $K=1$. 
If $\omega_{f^*}^{-1}(\varepsilon)\ne \infty$, let $\delta = (\omega_{f^*}^{-1}(\varepsilon))^{-1}$.
For each $i\in[K]$, we define
\begin{align*}
    \alpha_i&=\begin{cases}
        i\delta~&\text{if}~i\in\{0,1,\dots,K-1\},\\
        1~&\text{if}~i=K,
    \end{cases}\\
    \mathcal I_i&=\begin{cases}
        [\alpha_{i-1}^{(\ge,\bbF_p)},\alpha_i^{(<,\bbF_p)}]\cap\bbF_p~&\text{if}~i\in[K-1],\\
        [\alpha_{K-1}^{(\ge,\bbF_p)},\alpha_K^{(\le,\bbF_p)}]\cap\bbF_p~&\text{if}~i=K.
    \end{cases}
\end{align*}
Without loss of generality, we assume that $\mathcal I_i\ne\emptyset$ for all $i\in[K]$; otherwise, we remove empty $\mathcal I_j$, decrease $K$, and re-index $\mathcal I_i$ so that $\mathcal I_i$ if non-empty for all $i\in[K]$.
We note that since $0,1\in\bbF_p$, there is at least one non-empty $\mathcal I_i$ and $K\ge1$.
Then, by the above definitions, it holds that $\sup\mathcal I_i-\inf\mathcal I_i\le\delta$. Since $\alpha_i^{(<,\bbF_p)} < \alpha_i^{(\ge,\bbF_p)}$, $\mathcal I_{i_1}$ and $\mathcal I_{i_2}$ are disjoint if $i_1\ne i_2$, and $\bigcupdot_{i\in[K]}\mathcal I_i=[0,1]\cap\bbF_p$.
For each $\boldsymbol{\iota}=(\iota_1,\dots,\iota_d)\in[K]^d$, we also define 
$$\boldsymbol{\gamma}_{\boldsymbol{\iota}}=\argmin_{\bfx\in\mathcal I_{\iota_1}\times\cdots\times\mathcal I_{\iota_d}}|f^*(\bfx)-\round{f^*(\bfx)}|,$$
which is well-defined by \cref{lem:bestrounding}.

We are now ready to introduce our $\step$ network construction $f_\theta$:
\begin{align*}
    f_{\theta}(\bold{x} ; \bbF_p ) =  \bigoplus_{{\boldsymbol{\iota}}\in [K]^d } \round{f^*(\boldsymbol{\gamma}_{\boldsymbol{\iota}})} \otimes \indc{\left(\bigoplus_{j=1}^{2d}h_{{\boldsymbol{\iota}},j}(x_{\ceilZ{j/2}})\right) \ominus d \ge 0}
\end{align*}
where for each $j\in[d]$ and $\boldsymbol{\iota}=(\iota_1,\dots,\iota_d)\in[K]^d$, %
 \begin{align*}
     h_{{\boldsymbol{\iota}},2j-1}(x) &= \begin{cases}
    \indc{ x - \alpha_{j-1}^{(\ge,\bbF_p)} \ge 0  }\; &\text{if} \; \iota_j\in\{2,\dots,K\}, \\
    \indc{ x + \alpha_{0}^{(\ge,\bbF_p)} \ge 0  }\; &\text{if} \; \iota_j=1,
\end{cases}\\
h_{{\boldsymbol{\iota}},2j}(x) &= \begin{cases}
    - \indc{ x - \alpha_{j}^{(\ge,\bbF_p)} \ge 0} \; &\text{if} \; \iota_j\in\{1,\dots,K-1\}, \\
    -\indc{ x  - \alpha_K^{(> , \bbF_p)} \le 0 } \; &\text{if} \; \iota_j=K. 
\end{cases}%
\end{align*}
Since $h_{\boldsymbol{\iota},j}(x)\in\{-1,0,1\}, h_{\boldsymbol{\iota},2j-1}(x)+h_{\boldsymbol{\iota},2j}(x) = \{-1,0,1\} $, we have 
$${\left|\bigoplus_{j=1}^{m}h_{{\boldsymbol{\iota}},j}(x_{\ceilZ{j/2}}) \right| \le d \le 2^p.}$$
for any $1 \le m \le 2d$. 
Therefore, {by \cref{lem:integerexact_fp}}, all operations in the computation of $\bigoplus_{j=1}^{2d}h_{{\boldsymbol{\iota}},j}(x_{\ceilZ{j/2}})$ are exact. 
Furthermore, for each $\bfx \in \bbF_p^d \cap [0,1]^d$, we have 
\begin{align*}
\bigoplus_{j=1}^{2d}h_{{\boldsymbol{\iota}},j}(x_{\ceilZ{j/2}})\begin{cases}
    =d~&\text{if}~\bfx\in\mathcal I_{\iota_1}\times\cdots\times\mathcal I_{\iota_d},\\
    <d~&\text{if}~\bfx\in([0,1]^d\cap\bbF_p^d)\setminus(\mathcal I_{\iota_1}\times\cdots\times\mathcal I_{\iota_d}).
\end{cases}
\end{align*}
Since $ \bigcupdot_{\boldsymbol{\iota}\in[K]^d}\mathcal I_{\iota_1}\times\cdots\times\mathcal I_{\iota_d}=\bbF_p^d \cap [0,1]^d$, we have 
\begin{align}
    f_\theta( \bfx ; \bbF_p ) =\round{f^*(\boldsymbol{\gamma}_{\boldsymbol{\iota}})}, \; \forall \bfx \in \mathcal I_{\iota_1}\times\cdots\times\mathcal I_{\iota_d}. \nonumber
\end{align}
Hence, for each $\boldsymbol{\iota}\in[K]^d$ and $\bfx \in \mathcal I_{\iota_1}\times\cdots\times\mathcal I_{\iota_d}$,
\begin{align*}
|f_\theta( \bfx ; \bbF_p ) - f^*(\bold{x}) | &= |\round{f^*(\boldsymbol{\gamma}_{\boldsymbol{\iota}})} - f^*(\boldsymbol{\gamma}_{\boldsymbol{\iota}}) | + |   f^*(\boldsymbol{\gamma}_{\boldsymbol{\iota}}) - f^*(\bold{x}) |\\
&\le |\round{f^*(\bfx)} - f^*(\bfx) | + \varepsilon
\end{align*}
where we use $\|\bfx-\boldsymbol{\gamma}_{\boldsymbol{\iota}}\|_\infty\le\omega_{f^*}(\varepsilon)$ for the above inequality.
Since each $h_{\boldsymbol{\iota},j}$ can be implemented by a $\step$ network of $3$ parameters, 
$\indc{\left(\bigoplus_{j=1}^{2d}h_{{\boldsymbol{\iota}},j}(x_{\ceilZ{j/2}})\right) \ominus d  \ge 0}$ can be implemented using $6d+1$ parameters. 
This implies that our $f_\theta$ can be implemented by a $\step$ network of $3$ layers and $(6d+2)K^d$ parameters.
This completes the proof.

\subsection{Proof of \cref{thm:memrelufp-temp}}\label{sec:pfthm:memrelufp-temp}

Define $\delta=\frac12\min \{ |{z}_{i,j}| : z_{i,j}\ne0,i\in[n],j\in[d] \}$. 
For each $i\in[n]$, we also define $h_{i,1},\dots,h_{i,4d}$ as follows: for each $j\in[d]$,
\begin{align*}
h_{i,4j-3}=\psi_{\theta_{1,1,t_{i,j,1}}},~h_{i,4j-2}=\psi_{\theta_{1,2,t_{i,j,1}}},~h_{i,4j-1}=-\psi_{\theta_{1,1,t_{i,j,2}}},~h_{i,4j}=-\psi_{\theta_{1,2,t_{i,j,2}}},
\end{align*}
where $t_{i,j,1}=z_{i,j}$ and $t_{i,j,2}=z_{i,j}^+$ if $z_{i,j}\ne0$, $t_{i,j,1}=-\delta$ and $t_{i,j,2}=\delta^+$ if $z_{i,j}=0$, and 
$\psi_{\theta_{1,1,z}},\psi_{\theta_{1,2,z}}$ are defined in \cref{lem:reluindc-temp}.

By \cref{lem:reluindc-temp}, we have 
\begin{align*}
    & h_{i,4j-3}(x),-h_{i,4j-2}(x), -h_{i,4j-1}(x),h_{i,4j}(x)  \in\{0\}\cup[2^p] \\
    &\bigoplus_{k=1}^{4} h_{i,4j-4+k} (x)=\begin{cases}\indc{x=z_{i,j}}~&\text{if}~z_{i,j}\ne0,\\
\indc{x \in [-\delta,\delta}  ~&\text{if}~z_{i,j}=0,
\end{cases} 
\end{align*}
for all $x\in\bbF_p$.

Let $ 0 \le m < d \le 2^p$. For each $m$, we have $ h_{i,4m+1}(x)+  h_{i,4m+2}(x) = l_{4m+1} \in \{0,1\}$ and 
\begin{align*}
      -2^{p} < -m \le &\bigoplus_{j=1}^{4m}h_{i,j}(x)  \le m  < 2^p, \;   h_{i,4m+1}(x)  \in\{0\}\cup[2^p], \\
      -2^{p} \le -m +h_{i,4m+1}(x) \le & \bigoplus_{j=1}^{4m+1}h_{i,j}(x)  \le m + h_{i,4m+1}(x) < 2^{p+1}, \;  - h_{i,4m+2}(x) \in\{0\}\cup[2^p], \\
      -2^{p}< -m + l_{4m+1} \le & \bigoplus_{j=1}^{4m+2}h_{i,j}(x) \le m + l_{4m+1} \le 2^{p}. \; 
\end{align*}
By \cref{lem:integerexact_fp}, all above operations in $ \bigoplus_{j=1}^{4m+2}h_{i,j}(x)$ are exact.

We have  $ h_{i,4m+3}(x)+  h_{i,4m+4}(x) = l_{4m+3} \in \{0,-1\},  l_{4m+1}+ l_{4m+3} \in \{0,-1\} $, and
\begin{align*}
      -2^p \le -m + l_{4m-1}  \le &\bigoplus_{j=1}^{4m+2}h_{i,j}(x) \le 2^p, \;   -h_{i,4m+3}(x)  \in\{0\}\cup[2^p], \\
      -2^{1+p} \le -m + l_{4m-1} + h_{i,4m+3}(x) \le & \bigoplus_{j=1}^{4m+3}h_{i,j}(x) \le 2^p + h_{i,4m+3}(x) , \;   h_{i,4m+4}(x) \in\{0\}\cup[2^p], \\
      -2^{p} \le -m + l_{4m-1} + l_{4m-3} \le & \bigoplus_{j=1}^{4m+4}h_{i,j}(x) < 2^p + l_{4m-3}  \le 2^{p}. \; 
\end{align*}
By \cref{lem:integerexact_fp}, all above operations in $ \bigoplus_{j=1}^{4m+2}h_{i,j}(x)$ are exact.

Therefore we conclude that all operations in $ \bigoplus_{j=1}^{4d}h_{i,j}(x)$ are exact with  $ |\bigoplus_{j=1}^{4d}h_{i,j}(x)| \in \{0\} \cup [2^p]$.

We design the target network $f_\theta$ as follows:
\begin{align}
    & f_\theta(\bold{x} ; \bbF_p ) =  \bigoplus_{i=1 }^n y_i \otimes \relu \left( \left(\bigoplus_{j=1}^{4d}h_{i,j}(x_{\ceilZ{j/4}}) \right)\ominus (d - 1) \right). \nonumber \end{align}
Since for each $k\in[n]$
\begin{align*}
  \bigoplus_{j=1}^{4d}h_{i,j}(z_{k,\ceilZ{j/4}}) \begin{cases}
      =d~&\text{if}~\bfz_k=\bfz_i,\\
      <d~&\text{if}~\bfz_k\ne \bfz_i,
  \end{cases}
\end{align*}
$f_\theta$ memorizes the target dataset. 
Since there are {5} parameters for each $h_{i,j}$, $f_\theta$ has $20dn+2n$ parameters. This completes the proof.

\subsection{Proof of \cref{thm:univrelufp-temp}}\label{sec:pfthm:univrelufp-temp}
The proof of \cref{thm:univrelufp-temp} is almost identical to that of \cref{thm:univstepfp-temp}; we define $\mathcal I_i$, $\alpha_i$, and $\boldsymbol{\gamma}_{\boldsymbol{\iota}}$ as in \cref{sec:pfthm:univstepfp-temp}.
For each $\boldsymbol{\iota}\in[K]^d$, $j \in [d]$, we also define $h_{\boldsymbol{\iota},1},\dots,h_{\boldsymbol{\iota},4d}$ as follows: 
\begin{align*}
h_{\boldsymbol{\iota},4j-3}=\psi_{\theta_{1,1,t_{\boldsymbol{\iota},j,1}}},~h_{\boldsymbol{\iota},4j-2}=\psi_{\theta_{1,2,t_{\boldsymbol{\iota},j,1}}},~h_{\boldsymbol{\iota},4j-1}=-\psi_{\theta_{1,1,t_{\boldsymbol{\iota},j,2}}},~h_{\boldsymbol{\iota},4j}=-\psi_{\theta_{1,2,t_{\boldsymbol{\iota},j,2}}},
\end{align*}
where 
\begin{align*}
    t_{\boldsymbol{\iota},j,1}&=\begin{cases}
    \alpha_{j-1}^{(\ge,\bbF_p)}~&\text{if}~\iota_j\in\{2,\dots,K\},\\
    -\alpha_{0}^{(\ge,\bbF_p)}~&\text{if}~\iota_j=1,
    \end{cases}\\
    t_{\boldsymbol{\iota},j,2}&=\begin{cases}
    \alpha_{j}^{(\ge,\bbF_p)}~&\text{if}~\iota_j\in\{1,\dots,K-1\},\\
    \alpha_{K}^{(>,\bbF_p)}~&\text{if}~\iota_j=K,
    \end{cases}
\end{align*}
and $\psi_{\theta_{1,1,z}},\psi_{\theta_{1,2,z}}$ are defined in \cref{lem:reluindc-temp}. 
Namely, we have
\begin{align*}
h_{\boldsymbol{\iota},4j-3}(x)\oplus h_{\boldsymbol{\iota},4j-2}(x)&=\begin{cases}
    \indc{x\ge\alpha_{j-1}^{(\ge,\bbF_p)}}~&\text{if}~\iota_j\in\{2,\dots,K\},\\
    \indc{x\ge-\alpha_{0}^{(\ge,\bbF_p)}}~&\text{if}~\iota_j=1,
    \end{cases}\\
h_{\boldsymbol{\iota},4j-1}(x)\oplus h_{\boldsymbol{\iota},4j}(x)&=\begin{cases}
    -\indc{x\ge\alpha_{j}^{(\ge,\bbF_p)}}~&\text{if}~\iota_j\in\{1,\dots,K-1\},\\
    -\indc{x\ge-\alpha_{K}^{(>,\bbF_p)}}~&\text{if}~\iota_j=K,
    \end{cases}
\end{align*}
for all $x\in\bbF_p$ by \cref{lem:reluindc-temp,lem:integerexact_fp}.
We design the target network $f_\theta$ as follows:
\begin{align*}
f_\theta(\bold{x} ; \bbF_p ) =  \bigoplus_{{\boldsymbol{\iota}}\in [K]^d } \round{f^*(\boldsymbol{\gamma}_{\boldsymbol{\iota}})} \otimes \relu\left(\left(\bigoplus_{j=1}^{4d}h_{{\boldsymbol{\iota}},j}(x_{\ceilZ{j/4}})\right) \ominus (d-1) \ge 0\right).
\end{align*}
By similar argument presented in the proof of \cref{thm:memrelufp-temp}, all operations in $\bigoplus_{j=1}^{4d}h_{\boldsymbol{\iota},j}(x_{\ceilZ{j/4}})$ are exact by \cref{lem:integerexact_fp}, i.e., for each $k\in[n]$
\begin{align*}
  \bigoplus_{j=1}^{4d}h_{\boldsymbol{\iota},j}(x_{\ceilZ{j/4}})\begin{cases}
      =d~&\text{if}~\bfx\in\mathcal I_{\iota_1}\times\cdots\times\mathcal I_{\iota_d},\\
      \le d-1 ~&\text{if}~([0,1]^d\cap\bbF_p^d)\setminus(\mathcal I_{\iota_1}\times\cdots\times\mathcal I_{\iota_d}),
  \end{cases}
\end{align*}
This implies that for each $\boldsymbol{\iota}\in[K]^d$ and $\bfx \in \mathcal I_{\iota_1}\times\cdots\times\mathcal I_{\iota_d}$,
\begin{align*}
|f_\theta( \bfx ; \bbF_p ) - f^*(\bold{x}) | &= |\round{f^*(\boldsymbol{\gamma}_{\boldsymbol{\iota}})} - f^*(\boldsymbol{\gamma}_{\boldsymbol{\iota}}) | + |   f^*(\boldsymbol{\gamma}_{\boldsymbol{\iota}}) - f^*(\bold{x}) |\\
&\le |\round{f^*(\bfx)} - f^*(\bfx) | + \varepsilon
\end{align*}
where we use $\|\bfx-\boldsymbol{\gamma}_{\boldsymbol{\iota}}\|_\infty\le\omega_{f^*}(\varepsilon)$ for the above inequality.
Since each $h_{\boldsymbol{\iota},j}$ can be implemented by a $\relu$ network of $3$ layers and $5$ parameters by \cref{lem:reluindc-temp}, 
$\indc{\left(\bigoplus_{j=1}^{4d}h_{{\boldsymbol{\iota}},j}(x_{\ceilZ{j/4}})\right) \ominus (d-1) \ge 0}$ can be implemented using $20d+1$ parameters. 
This implies that our $f$ can be implemented by a $\relu$ network of $4$ layers and $(20d+2)K^d$ parameters.
This completes the proof.

\section{Proofs of results under $\bbF_{p,q}$}\label{sec:pfresults-fpqr}
\subsection{Addition notations for $\bbF_{p,q}$}
In this section, we introduce notations frequently used for proving our results under $\bbF_{p,q}$.
For $x \in \bbF_{p,q}$, if $x$ is normal, we represent $x$ as 
$$x = s_x \times a_x \times 2^{e_x}, \; s_x \in \{-1,1\},  a_x = 1.x_1\cdots x_p.  $$ 
If $x$ is subnormal, it is standard to represent $x$ as 
\begin{align*}
    &x = s_x \times a_x \times 2^{e_{min}}, \; s_x \in \{-1,1\},  a_x = 0.x_1\cdots x_p, \\
    &x_1=\cdots =x_{c_x-1}=0, x_{c_x}=1.
\end{align*}
for some $1 \le c_x < p$. 
However, instead of using the above representation, we opt for a different one for the sake of convenience.
$$x = \frak{s}_x \times \frak{a}_x \times 2^{\frak{e}_{x}}, \; \frak{s}_x \in \{-1,1\},  \frak{a}_x = 1.{x}_1\cdots {x}_{p-c_x}, \frak{e}_x = e_{min}-c_x.  $$ 
It is convenient because, regardless of whether $x$ is normal or subnormal, we have the following representation for $x$:
\begin{align}
    x = \frak{s}_x \times \frak{a}_x \times 2^{\frak{e}_{x}}, \; \frak{s}_x \in \{-1,1\},  \frak{a}_x = 1.\frak{x}_1\cdots \frak{x}_{p} , \;  2^{\frak{e}_x} \le x < 2^{1+\frak{e}_x}.
\end{align}
Note that if $x$ is normal, we have $\frak{s}_x = s_x, \frak{a}_x = a_x, \frak{e}_x = e_x$. 
We define $\mu(x)$ as 
$$ \mu(x) \defeq  \inf \{  m \in \bbZ : x \times 2^{- m }  \in \bbN \}.  $$
Note that if $x \ne 0$, we can represent $x$ as
\begin{align}
    x= (n_x \times 2^{-\frak{e}_x+\mu(x)} ) \times 2^{\frak{e}_x}, \label{eq:murepr}
\end{align}
 for $n_x  = x \times 2^{-\mu(x)} \in \bbN$ with $ 2^{\frak{e}_x- \mu(x)} \le n_x < 2^{1+ \frak{e}_x- \mu(x) }$. 
We define $e_0$ as 
$$ e_0 \defeq e^{q-2} - 1. $$
\begin{definition}
    For a number $x \in \bbR$, we say $x$ is ``representable'' by $\bbF_{p,q}$ if $x \in \bbF_{p,q}$.
\end{definition}

\begin{remark}\label{remark:rep}
    For nonzero $x \in \bbR$, suppose  $x$ has a following representation
    $$     x = \frak{s}_x \times \frak{a}_x \times 2^{\frak{e}_{x}}, \; \frak{s}_x \in \{-1,1\},  1 \le \frak{a}_x < 2, \; \frak{e}_x \in \bbZ,  2^{\frak{e}_x} \le x < 2^{1+\frak{e}_x}.  $$
    Define $c_x \defeq \max \{ 0, e_{min} - \frak{e}_x \}$. Then $x$ is representable by $\bbF_{p,q}$ if  $$ -p+e_{min} \le \frak{e}_x \le e_{max}, \; \frak{a}_x \times 2^{p-c_x} \in \bbN,$$
    which leads to 
    \begin{align}
        -p+e_{min} \le \frak{e}_x \le e_{max},  \; 0 \le \frak{e}_x - \mu(x) \le p-c_x. \label{eq:rep_test}
    \end{align}
    We say representability test on $x$ is to check whether $x$ satisfies \cref{eq:rep_test}. 
\end{remark}

\subsection{Technical lemmas}\label{sec:techlem-fpqr}
We introduce technical lemmas used for proving our results under $\bbF_{p,q}$.
\begin{lemma}\label{lem:representable}
    Let $x = n \times 2^{m}$ for some $n \in \bbN, \; m \in \bbZ$. Let $c_x = \max \{ 0 , e_{min}-m \}$.  If 
    $$ 0 < n < 2^{1+p-c_x}, \;   -p + e_{min} \le m \le -p+e_{max}, $$ 
    then $x$ is representable by $\bbF_{p,q}$.
\end{lemma}
\begin{proof}
    Since $ 0 \le n < 2^{1+p-c_x}$, there exists $c_0 \in  \{ 0 \} \cup [p-c_x]$ such that $2^{p-c_x-c_0} \le n < 2^{1+p-c_x-c_0}$.   
    Note that $x$ has the following representation in $\bbF_{p,q}$,
    \begin{align*}
         &x =  (n \times 2^{-p+c_x+c_0}) \times 2^{p-c_x-c_0+m}, \\
         &\; 1 \le n \times 2^{-p+c_x+c_0} <  2,  \;  -p + e_{min} \le p-c_x-c_0+m \le e_{max}.
    \end{align*}
    Then  we express $n \times 2^{-p+c_x+c_0} $ as 
    $$  n \times 2^{-p+c_x+c_0} = 1.\underbrace{w_{1} \cdots w_{p-c_x-c_0}}_{p-c_x-c_0 \; \text{times}},  $$ for some $w_1 , \dots , w_{p-c_x-c_0} \in \{0,1\}. $ 
    Therefore $x$ is representable by $\bbF_{p,q}$. 
\end{proof}

\begin{lemma}\label{lem:exact}
Let $x,y \in \bbF_{p,q}$ and $\frak{s}_x=\frak{s}_y$ and $\frak{e}_x \le \frak{e}_y$.
If $\mu(x) \ge \frak{e}_y-p $, then $x \ominus y$ and $y \ominus x$ are exact.
In addition, if $ |x + y| \le 2^{1+\frak{e}_y }$, then $x \oplus y$ is exact. 
\end{lemma}

\begin{proof}
Let $c_y = \max \{ 0 , e_{min}-\frak{e}_y \}$. Note that if $\frak{e}_y \ge e_{min}$ (\text{i.e.} $y$ is normal), we have $c_y = 0$.  If $\frak{e}_y < e_{min}$ (\text{i.e.} $y$ is subnormal), we have $c_y = e_{min} - \frak{e}_y> 0$ and $ \mu(x) \ge -p+e_{min} = -p +  \frak{e}_y +c_y$. Let $ k \defeq \mu(x)- \frak{e}_y+p-c_y \ge 0$.
As described in \cref{eq:murepr}, we can represent $x$ and $y$ as 
\begin{align*}
    &x = (n_x \times 2^{- \frak{e}_x + \mu(x)})  \times 2^{\frak{e}_x} = n_x \times 2^{ \mu(x)} ,\;  2^{\frak{e}_x - \mu(x)} \le n_x < 2^{1+\frak{e}_x - \mu(x)}, \\
    &y = (n_y \times 2^{-p+c_y}) \times 2^{\frak{e}_y} = n_y \times 2^{-p+c_y+\frak{e}_y} ,\;  2^{p-c_y} \le n_y < 2^{1+p-c_y},
\end{align*}
for some $n_x,n_y \in \bbN$. 
Since $ k \ge 0$, we have 
$$x = (2^k n_x) \times 2^{-p + \frak{e}_y+c_y} , \;   2^{p-(\frak{e}_y -\frak{e}_x+c_y) } \le 2^k n_x < 2^{1+p-(\frak{e}_y -\frak{e}_x+c_y) }.$$
Therefore for $n'  = n_y - 2^k n_x \in \bbN$, we have
$$ y - x = n' \times 2^{-p + \frak{e}_y+c_y}, \;  2^{p-c_y} - 2^{1+p-(\frak{e}_y -\frak{e}_x+c_y) } <  n' <  2^{1+p-c_y} - 2^{p-(\frak{e}_y -\frak{e}_x + c_y) }$$
which leads to 
$$ -2^{p-c_y} <  n' <  2^{p+1-c_y} - 1. $$
Since $|n'| =0$ or $|n'| <   2^{1+p-c_y} - 1$ and $-p + e_{min} \le -p + \frak{e}_y+c_y \le  e_{max}$, by \cref{lem:representable}, $ y - x =n' \times  2^{-p + \frak{e}_y+c_y}$ is representable by $\bbF_{p,q}$. This ensures that $y \ominus x$ is exact. \\
Now suppose $| x + y | \le 2^{1+\frak{e}_y}$. Then for $n'' = n_x + n_y \in \bbN$, we have 
$$  x + y  = n'' \times 2^{-p + \frak{e}_y+c_y}, \;  2^{p-(\frak{e}_y -\frak{e}_x + c_y) } + 2^{p-c_y}  \le n'' <   2^{1+p-(\frak{e}_y -\frak{e}_x+c_y) } +2^{1+p-c_y}. $$
Since $| x + y | \le 2^{1+\frak{e}_y}$, we have $n'' \le 2^{1+p-c_y}$. If $|n''| =0$ or $|n''| <   2^{1+p-c_y} $, since $-p + e_{min} \le -p + \frak{e}_y+c_y \le  e_{max}$, by \cref{lem:representable}, $ x + y =n'' \times  2^{-p + \frak{e}_y+c_y}$ is representable by $\bbF_{p,q}$. If $|n''| = 2^{1+p-c_y}$, $x+y = 2^{1+\frak{e}_y}$ is obviously representable $\bbF_{p,q}$. Therefore, $x \oplus y$ is exact. 
\end{proof}

\begin{lemma}\label{lem:integerexact}
In $\bbF_{p,q}$, if $x , y \in [2^{1+p}]$, $ x \ominus y$ and $y \ominus x$ are exact. In addition if $x,y \in [2^{p}]$, then $x \oplus y$ is exact. 
\end{lemma}
\begin{proof}
    Without loss of generality, suppose $x \le y$ for $x,y \in [2^{1+p}]$. Since $y - x < 2^{1+p}$, $ y - x$ is representable by $\bbF_{p,q}$ by \cref{lem:representable}, ensuring that $ x \ominus y$ and $y \ominus x$ are exact. \\
    In addition, suppose $x ,y \in [2^p]$ and $x \le y$. Since $\mu(x) \ge 0$ and $\frak{e}_y \le p$, we have $\mu(x) \ge -p +\tcb{\frak{e}_y} $. Since $ | x + y | < 2^{1+p}$, $x \oplus y$ is exact by \cref{lem:exact}. 
\end{proof}

\begin{lemma}\label{lem:ignore}
Let $x,y \in \bbF_{p,q}$ and $\frak{e}_x \le \frak{e}_y$.
If $\frak{e}_x \le -2-p+ \frak{e}_y$, then $ y \oplus x =  y \ominus x = y$.  
\end{lemma}
\begin{proof}
    Since $|x| \le 2^{-2-p+\frak{e}_y}$, we have $\round{x+y}=\round{-x+y}=y$. 
\end{proof}

\begin{lemma}\label{lem:reluindcfpqr}
There exists $\relu$ networks $\psi_\theta(\cdot;\bbF_{p,q}):\bbF_{p,q}\to\bbF_{p,q}$ of $3$ layers and $5$ parameters that satisfies the following: for any $z\in\bbF_{p,q}$ satisfying $ 0 < |z| \le (2-u) \times 2^{-2-p+e_{max}}$, there exist $\theta_{1,1,z},\theta_{1,2,z},\theta_{2,1,z},\theta_{2,2,z} \in\bbF_{p,q}^{5}$ such that
\begin{align*}
    &\psi_{\theta_{1,1,z}}(x ; \bbF_{p,q}) \oplus \psi_{\theta_{1,2,z}}(x ; \bbF_{p,q}) = \indc{x\ge z}, \\ 
    &\psi_{\theta_{2,1,z}}(x ; \bbF_{p,q}) \oplus \psi_{\theta_{2,2,z}}(x ; \bbF_{p,q}) = \indc{x\le z}, 
\end{align*}
for $|x| \le (2-u) \times 2^{-2+e_0}$.
In addition, we have $ \psi_{\theta_{i,1,z}}(x;\bbF_{p,q}) \in \{0 \} \cup [2^p]$, $ -\psi_{\theta_{i,2,z}}(x;\bbF_{p,q}) \in \{0 \} \cup [2^p]$  for all $ i \in \{1,2\}$. 
\end{lemma}
\begin{proof}

In this proof, to ensure that indicator functions are realizable by $\relu$ networks in $\bbF_{p,q}$, we should check  (1) whether every parameter in the network is representable by $\bbF_{p,q}$ and (2) whether every number occurring during the intermediate calculations is also representable by $\bbF_{p,q}$. Hence, we conduct the representability tests (\cref{remark:rep}) for each condition. 

We first consider $z>0$. 
We define 
\begin{align*}
    u_0 &= \begin{cases}
    2^{-p}\; (=u) \; &\text{if} \; \frak{e}_z \ge e_{min}, \\ 
    2^{-p+c_z} \; &\text{if} \; \frak{e}_z < e_{min}, \\ 
\end{cases}, \\
k &= \begin{cases}
  -p+e_0 \; &\text{if} \;  0 < \frak{e}_z \le -2-p+e_{max}, \\ 
  3-p-e_0 - c_z \; &\text{if} \; \frak{e}_z \le 0, \\ 
\end{cases},\\
   \tilde{u}_0 &= \begin{cases}
    0  \; &\text{if} \; \frak{a}_z \ne 1, \; \text{or} \; \frak{e}_z = -p+e_{min} \\ 
     u_0  \times 2^{-1}\; &\text{if} \; \frak{a}_z = 1, \\ 
\end{cases},\\
    \tilde{c} &= \begin{cases}
  0 \; &\text{if} \; \frak{a}_z \ne 1, \; \text{or} \; \frak{e}_z = -p+e_{min} \\ 
 1 \; &\text{if} \; \frak{a}_z = 1, \\ 
\end{cases} ,
\end{align*}
where $c_z = \max \{  e_{min} - \frak{e}_z , 0 \}  \ge 0$, and recall that 
$$ e_{min} = -2^{q-1}+2 , e_{max} = 2^{q-1}-1, e_0 = 2^{q-2}-1 , 4 \le p \le 2^{q-2}+2.  $$

To construct  $\indc{x \ge z}$, we define a three-layer $\relu$ network $f_{1}(x;\bbF_{p,q})$ as follows:
\begin{align}
    f_{1}(x;\bbF_{p,q}) = \psi_{\theta_{1,1,z}}(x;\bbF_{p,q}) \oplus \psi_{\theta_{1,2,z}}(x;\bbF_{p,q}), \label{eq:f1fpqr}
\end{align}
where
\begin{align*}
 &\psi_{\theta_{1,1,z}}(x;\bbF_{p,q}) = \phi_{\theta_{1,1,z}}(x;\bbF_{p,q}), \;  \psi_{\theta_{1,2,z}}(x;\bbF_{p,q}) = \phi_{\theta_{1,2,z}}(x;\bbF_{p,q}), \\
&\phi_{\theta_{1,1,z}}(x;\bbF_{p,q})=  2^{ \tilde{c} -c_z -k} \otimes \relu\big( (-2^{p-\frak{e}_z+k} \otimes g_1(x) )  \oplus (2-\frak{a}_z - \tilde{u}_0 ) \times 2^{p+k}   \big),  \\
    & \phi_{\theta_{1,2,z}}(x;\bbF_{p,q}) = -2^{ \tilde{c}-c_z -k} \otimes \relu\big( (-2^{p-\frak{e}_z+k} \otimes g_1(x) )  \oplus (2-\frak{a}_z - u_0) \times 2^{p+k}   \big) . 
\end{align*}
Here, $g_1(x)$ is defined as
\begin{align*}
    g_1(x):= \relu( -x \oplus ( (2-u_0) \times 2^{\frak{e}_z} )).
\end{align*}
Let 
\begin{align*}
     &\zeta_{1,1}(x):= (-2^{p-\frak{e}_z+k} \otimes g_1(x) )  \oplus (2-\frak{a}_z - \tilde{u}_0 ) \times 2^{p+k},   \\
     &\zeta_{1,2}(x): = (-2^{p-\frak{e}_z+k} \otimes g_1(x) )  \oplus (2-\frak{a}_z - u_0) \times 2^{p+k}, \\
     &\phi_{\theta_{1,1,z}}(x;\bbF_{p,q}) = 2^{ \tilde{c} -c_z -k} \otimes \relu(\zeta_{1,1}(x)),\\
    &\phi_{\theta_{1,2,z}}(x;\bbF_{p,q}) = -2^{ \tilde{c}-c_z -k} \otimes  \relu(\zeta_{1,2}(x)).  
\end{align*}
Note that (1) the following are parameters in $\phi_{\theta_{1,1,z}}(x;\bbF_{p,q})$ and $\phi_{\theta_{1,2,z}}(x;\bbF_{p,q})$:
\begin{align}
    2^{\tilde{c}-c_z-k},2^{p-\frak{e}_z+k}, (2-u_0) \times 2^{\frak{e}_z}, (2-\frak{a}_z-\tilde{u}_0) \times 2^{p+k},(2-\frak{a}_z-u_0) \times 2^{p+k}. 
    \label{eq:param_psi_1_fpq}
\end{align}
The following are (2) the numbers occurring during the intermediate calculations in $\phi_{\theta_{1,1,z}}(x;\bbF_{p,q})$ and $\phi_{\theta_{1,2,z}}(x;\bbF_{p,q})$:
\begin{align}
    g_1(x), -2^{p-\frak{e}_z+k} \otimes g_1(x) , \zeta_{1,1}(x), \zeta_{1,2}(x), \phi_{\theta_{1,1,z}}(x;\bbF_{p,q}), \phi_{\theta_{1,2,z}}(x;\bbF_{p,q}).  \label{eq:nums_psi_1_fpq}
\end{align}

Now, we consider the following cases.

\paragraph{\textbf{Case 1-1: $ \frak{e}_z = -p+ e_{min} $.}} \\
In this case, we have 
\begin{align}
    &z =  2^{-p+ e_{min}}, \;  c_z = p, \; k = 3 - 2p-e_0, \\
    &  2^{\tilde{c}-c_z-k} = 2^{-3+p + e_0}, 2^{p-\frak{e}_z+k}=2^{3-e_0-e_{min}}, (2-u_0) \times 2^{\frak{e}_z} = 2^{-p+e_{min}}, \label{eq:params_case11_1} \\
    &(2-\frak{a}_z-\tilde{u}_0) \times 2^{p+k}  = 2^{3-p-e_0}, \; (2-\frak{a}_z-u_0) \times 2^{p+k} = 0, \label{eq:params_case11_2} \\
    & \phi_{\theta_{1,1,z}}(x;\bbF_{p,q}) = 2^{ -3 +  p + e_0  } \otimes \relu\big( (-2^{3-e_0-e_{min}} \otimes g_1(x) )  \oplus  2^{3-p-e_0}   \big), \nonumber \\
 & \phi_{\theta_{1,2,z}}(x;\bbF_{p,q}) =  - 2^{-3+p+e_0} \otimes \relu\big( -2^{3-e_0-e_{min}} \otimes g_1(x)     \big) , \nonumber
\end{align}
where $g_1(x) = \relu( -x  \oplus 2^{-p+e_{min}} )$. 

If $- (2-u)  \times 2^{ -2 + e_0 }  \le  x\leq 0$, we have
\begin{align*}
    &2^{-p+e_{min}} \le  \relu( -x  \oplus 2^{-p+e_{min}} ) 
 \\
 &\le \relu(  (2-u)  \times 2^{ -2+ e_0 }  \oplus 2^{-p+e_{min}} ) = (2-u)  \times 2^{ -2+ e_0 }, 
\end{align*}
where the last equation is due to Lemma \ref{lem:ignore}. 
Therefore, 
    \begin{equation}\label{eq:case1pqr_eq1}
     2^{-p+e_{min}} \le  g_1(x) \le (2-u)  \times 2^{ -2+ e_0 },
    \end{equation}
    which leads to
    \begin{equation}\label{eq:case1pqr_g0times1}
        -2^{3-e_0-e_{min}} \otimes g_1(x) \le  -2^{3-e_0-e_{min}}\otimes 2^{-p+e_{min}} = -2^{3-p-e_0},
    \end{equation}
    and
    \begin{equation}\label{eq:case1pqr_g0times2}
       -2^{3-e_0-e_{min}} \otimes g_1(x)  \ge  -(2-u)\times 2^{1-e_{min}} = -(2-u) \times 2^{ e_{max}  }.
    \end{equation}
    Hence, the following inequalities hold:
    \begin{align}
         &-(2-u)  \times 2^{e_{max} }  \le \zeta_{1,1}(x) \le 0, \label{eq:case11zeta1} \\
         &-(2-u)  \times 2^{e_{max} }  \le \zeta_{1,2}(x) \le 0. \label{eq:case11zeta2}
    \end{align}
    Therefore, 
\begin{equation}  
     \phi_{\theta_{1,1,z}}(x;\bbF_{p,q}) = \phi_{\theta_{1,2,z}} (x;\bbF_{p,q})  = 0, \label{eq:case11final}
\end{equation}
which leads to $f_{1}(x;\bbF_{p,q})=0$. \\
If  $  2^{-p+e_{min}} \le  x \le \Omega$, we have  
\begin{align}\label{eq:case1pqr_eq2}
    \relu( 2^{-p+e_{min}} \ominus x ) = 0, \; \zeta_{1,1}(x)  = 2^{3-p-e_0}, \; \zeta_{1,2}(x) = 0, 
\end{align}
leading to $\phi_{\theta_{1,1,z}}(x;\bbF_{p,q}) = 1, \phi_{\theta_{1,2,z}}(x;\bbF_{p,q})=0, f_{1}(x;\bbF_{p,q})=1$.

\textbf{Representability test for Case 1-1.}
(1) We check that the parameters in \cref{eq:params_case11_1,eq:params_case11_2} are representable by $\bbF_{p,q}$:
since
\begin{align*}
   -3 +  p + e_0  \le -3 + (2^{q-2}+2) + (2^{q-2} -1)  =  2^{q-1}-2 < e_{max},
\end{align*} 
$2^{ -3 +  p + e_0  }$ is representable by $\bbF_{p,q}$. 
Apparently, $2^{3-e_0-e_{min}}$, $2^{-p+e_{min}}$, $2^{3-p-e_0}$,$2^{3-p-e_0}$, and $0$ are representable by $\bbF_{p,q}$.
(2) We check that the intermediate numbers in \cref{eq:nums_psi_1_fpq} are representable  by $\bbF_{p,q}$, which is apparent by 
\cref{eq:case1pqr_eq1,eq:case1pqr_g0times1,eq:case1pqr_g0times2,eq:case11zeta1,eq:case11zeta2,eq:case11final,eq:case1pqr_eq2}

Therefore, we conclude 
\begin{align}
    f_{1}(x;\bbF_{p,q}) \begin{cases}
        = 0\; &\text{if} \; - (2-u)  \times 2^{ -2+ e_0 }  \le  x \le 0  \\ 
        = 1  \ \; &\text{if} \; 2^{-p+ e_{min}} \le x \le \Omega. \\ 
    \end{cases}
\end{align}

\paragraph{\textbf{Case 1-2: $ 1-p+e_{min} \le \frak{e}_z \le 0 $.}} \\

In this case, we have $k = 3 - p - e_0 - c_z$.
Let $\Omega_1 = \min \{ \Omega ,(2 - u) \times 2^{ -3+e_{max} + e_0 +\frak{e}_z + c_z }\}$. Since 
$$ -3+e_{max} + e_0 +\frak{e}_z + c_z \ge  -3+e_0 +(e_{max}+e_{min}) = -2 + e_0,$$
we have $\Omega_1 \ge (2-u) \times 2^{-2 + e_0}$.

Suppose  $-\Omega_1 \le x \le (2-2u_0) \times 2^{-1+\frak{e}_z}$.
If  $-\Omega_1 \le x \le -2^{2+p+\frak{e}_z}$, we have $(2-u_0) \times 2^{\frak{e}_z} \ominus x = -x $ by Lemma \ref{lem:ignore}. If $ -2^{2+p+\frak{e}_z} < x \le (2-2u_0) \times 2^{-1+\frak{e}_z}$, we have $ (2-u_0) \times 2^{-1+\frak{e}_z} \le  (2-u_0) \times 2^{\frak{e}_z} \ominus x < 2^{2+p+\frak{e}_z}$. Therefore, we have 
\begin{align}
    & 2^{\frak{e}_z} \le g_1(x) \le \Omega_1, \label{eq:case2pqr_eq1} 
\end{align}
Since 
$$2^{p-\frak{e}_z+k} \otimes \Omega_1 \le 2^{p-\frak{e}_z+k} \otimes (2 - u) \times 2^{ -3+e_{max} + e_0 +\frak{e}_z + c_z } = (2-u)\times 2^{e_{max}} = \Omega,$$
we have $2^{p-\frak{e}_z+k} \otimes \Omega_1\le \Omega$.
If $\Omega_1 = (2 - u) \times 2^{ -3+e_{max} + e_0 +\frak{e}_z + c_z }$, we have $ 2^{p-\frak{e}_z+k} \otimes \Omega_1 = \Omega$. Hence, $2^{p-\frak{e}_z+k} \otimes \Omega_1 \geq \Omega\ge (2-u)\times 2^{4+e_0}$. If $\Omega_1 = \Omega$, since  
\begin{align*}
     &p-\frak{e}_z+k+e_{max} = p-\frak{e}_z+(3-p-e_0-c_z) + e_{max} \\
     & = 3-e_0-(\frak{e}_z+c_z)+e_{max}
     \ge 3 -e_0 +e_{max} = 4 + e_0,
 \end{align*}
we have $ 2^{p-\frak{e}_z+k} \otimes \Omega_1 \ge (2-u) \times 2^{4+e_0}$.
Therefore, we conclude that 
\begin{align}
    (2-u) \times 2^{4+e_0}\le 2^{p-\frak{e}_z+k} \otimes \Omega_1 \le \Omega.  \label{eq:omega}
\end{align}
 By \cref{lem:exact}, $-2^{p+k } \oplus (2-\frak{a}_z-\tilde{u}_0) \times 2^{p+k}$ and $-   2^{p+k } \oplus (2-\frak{a}_z-u_0) \times 2^{p+k}$ are exact. Together with \cref{eq:case2pqr_eq1} and \cref{eq:omega}, we have
\begin{align}
&  2^{\frak{e}_z} \le g_1(x) \le \Omega_1, \nonumber \\
&- (2 - u) \times 2^{4+e_0 } \le -2^{p-\frak{e}_z+k} \otimes g_1(x) \le   -  2^{p+k }, \label{eq:case2pqr_eq2} \\
&- (2 - u) \times 2^{4+e_0 }  \le  \zeta_{1,1}(x) \le   -(\frak{a}_z-1 +\tilde{u}_0) \times  2^{p+k} < 0, \label{eq:case2pqr_eq3} \\
&- (2 - u) \times 2^{4+e_0 }  \le  \zeta_{1,2}(x) \le   -(\frak{a}_z-1 +{u}_0) \times  2^{+p+k} < 0, \label{eq:case2pqr_eq4}
\end{align}
which leads to $\phi_{\theta_{1,1,z}}(x;\bbF_{p,q})=\phi_{\theta_{1,2,z}}(x;\bbF_{p,q})=0$.

If $   2^{\frak{e}_z} \le x \le (2-2u_0) \times 2^{\frak{e}_z}$, $-x \oplus ((2-u_0) \times 2^{\frak{e}_z})$ is exact by \cref{lem:sterbenz}. 
Therefore, we have $ 2^{p-\frak{e}_z+k} \otimes g_1(x) = (n_1 \times 2^{-p+c_z}) \times 2^{p+k}$ where $n_1 = 1,2, \dots , 2^{p-c_z}-1$ which leads to 
\begin{align}
    2^{p-\frak{e}_z+k} \otimes g_1(x) \le (2-2u_0) \times 2^{-1+p+k} , \; \mu ( 2^{p-\frak{e}_z+k} \otimes g_1(x)) \ge  k. \label{eq:case2pqr_eq5}
\end{align}
 We consider the following cases.

\paragraph{\textbf{Case 1-2-1: $ \frak{a}_z \ne 1 $.}} \\
In this case,  we have  
\begin{align}
    &2^{k} \le (2-\frak{a}_z - \tilde{u}_0 ) \times 2^{p+k},(2-\frak{a}_z - u_0 ) \times 2^{p+k} < 2^{p+k}, \label{eq:case2pqr_eq6} \\
    & \mu( (2-\frak{a}_z - \tilde{u}_0 ) \times 2^{p+k} ) , \mu ( (2-\frak{a}_z - u_0 ) \times 2^{p+k} ). \ge k \nonumber
\end{align}

\paragraph{\textbf{Case 1-2-2: $ \frak{a}_z = 1$.}} \\
In this case, we have 
\begin{align}
    &(2-\frak{a}_z - \tilde{u}_0 ) \times 2^{p+k} = (2-u_0) \times 2^{-1+p+k},\label{eq:case2pqr_eq7} \\
    &(2-\frak{a}_z - u_0 ) \times 2^{p+k} = (2-2u_0) \times 2^{-1+p+k}. \label{eq:case2pqr_eq8}
\end{align}
In both cases, since $\mu ( 2^{p-\frak{e}_z+k} \otimes g_1(x)) \ge  k $, by \cref{lem:exact}, all operations in $\phi_{\theta_{1,1,z}}(x;\bbF_{p,q})$ and $\phi_{\theta_{1,2,z}}(x;\bbF_{p,q})$ are exact. Hence 
\begin{align}
    \phi_{\theta_{1,1,z}}(x;\bbF_{p,q}) &= 2^{\tilde{c}-c_z-k} \times \relu\big(2^{p-\frak{e}_z+k} \times x - ( ( \frak{a}_z - u_0 + \tilde{u}_0) \times 2^{p+k} ) \big), \label{eq:case2pqr_eq9}  \\ 
    \phi_{\theta_{1,2,z}}(x;\bbF_{p,q}) &= - 2^{\tilde{c}-c_z-k} \times \relu\big(2^{p-\frak{e}_z+k} \times x - ( \frak{a}_z  \times 2^{p+k} ) \big), \label{eq:case2pqr_eq10}
\end{align}
with $  \phi_{\theta_{1,1,z}}(x;\bbF_{p,q})  ,  -\phi_{\theta_{1,2,z}}(x;\bbF_{p,q})  \in \{ 0 \} \cup [2^p]$.
Therefore, we have
\begin{align*}
    f_{1}(x;\bbF_{p,q}) = \begin{cases}
     0 \; &\text{if} \; (2-u_0) \times  2^{-1+\frak{e}_z} \le x \le z^{-}, \\ 
    1 \; &\text{if} \; z \le x \le (2-2u_0) \times 2^{\frak{e}_z}.
    \end{cases}
\end{align*}   
 If $(2-u_0) \times 2^{\frak{e}_z} \le x \le \Omega$, we have $g_1(x) = 0$. Hence 
 \begin{align}
     &\phi_{\theta_{1,1,z}}(x;\bbF_{p,q}) = 2^{\tilde{c}-c_z-k} \otimes ( ( 2 -\frak{a}_z - \tilde{u}_0 ) \times 2^{p+k } ), \label{eq:case2pqr_eq11} \\
     & \phi_{\theta_{1,2,z}}(x;\bbF_{p,q}) = - 2^{\tilde{c}-c_z-k} \otimes ( ( 2 -\frak{a}_z - u_0 ) \times 2^{p+k } ) \label{eq:case2pqr_eq12} \\
     &f_{1}(x;\bbF_{p,q})  = \phi_{\theta_{1,1,z}}(x;\bbF_{p,q}) \oplus \phi_{\theta_{1,2,z}}(x;\bbF_{p,q})   =   2^{p+\tilde{c}-c_z} \times ( u_0 - \tilde{u}_0) = 1.  \label{eq:case2pqr_eq13}
 \end{align}
 with $ \phi_{\theta_{1,1,z}}(x;\bbF_{p,q}), -\phi_{\theta_{1,2,z}}(x;\bbF_{p,q})\in \{0 \} \cup [2^p]$.

\textbf{Representability test for Case 1-2.} \\
(1)  We check that the parameters in \cref{eq:param_psi_1_fpq} are representable by $\bbF_{p,q}$. \\
Since 
\begin{align*}
    &e_{min} < \tilde{c}-c_z-k \le 1-(3-p-e_0) = -2+p+e_0 \\
    &= -2 + (2^{q-2}+2) +  (2^{q-2}-1)  \le  2^{q-1} - 1  = e_{max}, \\
    &  e_{min} < 3-e_0  \le p-\frak{e}_z+k = p - \frak{e}_z +(3-p-e_0-c_z) \\
    & = (3-e_0) + (-\frak{e}_z-c_z) \le 3-e_0-e_{min} \\
    &= 3-(2^{q-2}-1)+(2^{q-1}-2) = 2^{q-2}+2 < e_{max}, 
\end{align*}
$2^{\tilde{c}-c_z-k}$ and $2^{p-\frak{e}_z+k}$ are representable by $\bbF_{p,q}$.   
Since 
\begin{align*}
    &  e_{min} =- 2^{q-1} + 2 \le 4-p-2^{q-2} \le 3-p-e_0  \le p+k = 3-e_0-c_z < e_{max}, \\ 
    & e_{min} < 3-p-e_0 = (-p+c_z) + ( p +k ) \le  \mu(t)  \le \frak{e}_{t} \le p+k < e_{max}, \\
    & \frak{e}_{t} - \mu( t ) \le (p+k) - ( -(p-c_z) + (p+k) ) = p-c_z,  
\end{align*}
where $\frak{a}_z \ne 2-u_0$ for $t = (2 - \frak{a}_z -\tilde{u}_0 ) \times 2^{p+k}$, $(2 - \frak{a}_z -u_0 ) \times 2^{p+k}$,  they are representable by $\bbF_{p,q}$. 
If $\frak{a}_z = 2-u_0$, it is obvious.
Apparently, $(2-u_0) \times 2^{\frak{e}_z}$ is representable by $\bbF_{p,q}$. \\
(2) We check that the intermediate numbers in \cref{eq:nums_psi_1_fpq} are representable  by $\bbF_{p,q}$, which is apparent by
\cref{eq:case2pqr_eq1,eq:case2pqr_eq2,eq:case2pqr_eq3,eq:case2pqr_eq4,eq:case2pqr_eq5,eq:case2pqr_eq6,eq:case2pqr_eq7,eq:case2pqr_eq8,eq:case2pqr_eq9,eq:case2pqr_eq10,eq:case2pqr_eq11,eq:case2pqr_eq12,eq:case2pqr_eq13}.
Therefore, we conclude  
\begin{align*}
    f_1(x;z) \begin{cases}
        = 0\; &\text{if} \; - \Omega_1  \le  x \le z^-, \\ 
        = 1  \ \; &\text{if} \; z \le x \le \Omega. \\ 
    \end{cases}
\end{align*}

\paragraph{\textbf{Case 1-3: $0 < \frak{e}_z \le -2-p+e_{max} $.}} \\
In this case, $c_z =0$, $u_0 = u$, and $k = -p+e_0$. 
Let $\Omega_2 = \min\{ \Omega, (2-u) \times 2^{e_{max}-e_0+\frak{e}_z}$. Since 
$$ e_{max}-e_0 + \frak{e}_z \ge e_{max}-e_0 + 1 = 2+e_0,$$
we have $\Omega_2 \ge (2-u) \times 2^{2+e_0}$.

Suppose $ -\Omega_2  \le  x \le (2-u) \times 2^{-1+\frak{e}_z} $. If  $-\Omega_2 \le x < -\min \{\Omega_2, 2^{2+p+\frak{e}_z} \}$, we have $(2-u) \times 2^{\frak{e}_z} \ominus x = -x $ by Lemma \ref{lem:ignore}.   If $ -\min \{\Omega_2, 2^{2+p+\frak{e}_z} \}  < x \le (2-u) \times 2^{-1+\frak{e}_z}$, we have $ (2-u) \times 2^{-1+\frak{e}_z} \le  (2-u) \times 2^{\frak{e}_z} \ominus x < 2^{2+p+\frak{e}_z}$. Therefore, we have 
\begin{align}
    & (2-u) \times 2^{-1+\frak{e}_z} \le g_1(x) \le \Omega_2, \label{eq:case3pqr_eq1} 
\end{align}
If $\Omega_2 =(2-u) \times 2^{e_{max}-e_0+\frak{e}_z} $, we have $  2^{p-\frak{e}_z+k} \otimes \Omega_2 =  \Omega$. If $\Omega_2 = \Omega$, we have $ (2-u) \times 2^{  e_0 } \le  (2-u) \times 2^{ - \frak{e}_z + e_0   + e_{max} } =  2^{p-\frak{e}_z+k} \otimes \Omega_2$. Therefore, we have 
\begin{align}
    (2-u) \times 2^{e_0}\le 2^{p-\frak{e}_z+k} \otimes \Omega_2.  \label{eq:omega_2}
\end{align}
Hence, \cref{eq:case3pqr_eq1} and \cref{eq:omega_2}   lead to 
\begin{align}
    - (2 - u) \times 2^{e_0 } \le -2^{p-\frak{e}_z+k} \otimes g_1(x) \le - (2-u) \times  2^{-1+p+k }. \label{eq:case3pqr_eq2} 
\end{align}
 By \cref{lem:exact}, $-(2-u) \times  2^{-1+p+k } \oplus (2-\frak{a}_z-\tilde{u}_0) \times 2^{p+k}$ and $-(2-u) \times  2^{-1+p+k } \oplus (2-\frak{a}_z-u) \times 2^{p+k}$ are exact. Together with \cref{eq:case3pqr_eq1} and \cref{eq:omega_2}, we have
\begin{align}
&- (2 - u) \times 2^{e_0 }  \le  \zeta_{1,1}(x) \le   -(2(\frak{a}_z-1) -u+2\tilde{u}) \times  2^{-1+p+k} \le 0, \label{eq:case3pqr_eq3} \\
&- (2 - u) \times 2^{e_0 }  \le  \zeta_{1,2}(x) \le   -(2(\frak{a}_z-1) +u) \times  2^{-1+p+k} < 0, \label{eq:case3pqr_eq4}
\end{align}
 which leads to $\phi_{\theta_{1,1,z}}(x;\bbF_{p,q})=\phi_{\theta_{1,2,z}}(x;\bbF_{p,q})=0$. 

If $  2^{\frak{e}_z} \le x \le (2-2u) \times 2^{\frak{e}_z}$, by similar arguments in \textbf{Case 1-2} and $  e_{min} <  -1 + p + k  , p + k < e_{max}$, we have the same results as those in 
\cref{eq:case2pqr_eq5,eq:case2pqr_eq6,eq:case2pqr_eq7,eq:case2pqr_eq8} from \textbf{Case 1-2}. Since $\mu( 2^{p-\frak{e}_z+k} )  \otimes g_1(x) \ge k$, by \cref{lem:exact} all operations in $\phi_{\theta_{1,1,z}}(x;\bbF_{p,q})$ and $\phi_{\theta_{1,2,z}}(x;\bbF_{p,q})$ are exact. Hence 
\begin{align}
    \phi_{\theta_{1,1,z}}(x;\bbF_{p,q}) &= 2^{\tilde{c}-c_z-k} \times \relu\big(2^{p-\frak{e}_z+k} \times x - ( ( \frak{a}_z - u + \tilde{u}) \times 2^{p+k} ) \big), \label{eq:case3pqr_eq5}  \\ 
    \phi_{\theta_{1,2,z}}(x;\bbF_{p,q}) &= -2^{\tilde{c}-c_z-k} \times \relu\big(2^{p-\frak{e}_z+k} \times x - ( \frak{a}_z  \times 2^{p+k} ) \big), \label{eq:case3pqr_eq6}
\end{align}
with $  \phi_{\theta_{1,1,z}}(x;\bbF_{p,q})  ,  -\phi_{\theta_{1,2,z}}(x;\bbF_{p,q})  \in [2^p]$.
Therefore,
\begin{align*}
    f_{1}(x;\bbF_{p,q}) = \begin{cases}
     0 \; &\text{if} \; (2-u) \times  2^{-1+\frak{e}_z} \le x \le z^{-}, \\ 
    1 \; &\text{if} \; z \le x \le (2-2u) \times 2^{\frak{e}_z}.
    \end{cases}
\end{align*}
 If $(2-u) \times 2^{\frak{e}_z} \le x \le \Omega$, we have $g_1(x) = 0$. Hence 
 \begin{align}
     &\phi_{\theta_{1,1,z}}(x;\bbF_{p,q}) = 2^{\tilde{c}-c_z-k} \otimes ( ( 2 -\frak{a}_z - \tilde{u}_0 ) \times 2^{p+k } ) \label{eq:case3pqr_eq7} \\
     & \phi_{\theta_{1,2,z}}(x;\bbF_{p,q}) =  - 2^{\tilde{c}-c_z-k} \otimes ( ( 2 -\frak{a}_z - u_0 ) \times 2^{p+k } ) \label{eq:case3pqr_eq8} \\
     &f_{1}(x;\bbF_{p,q}) =  \phi_{\theta_{1,1,z}}(x;\bbF_{p,q}) \oplus \phi_{\theta_{1,2,z}}(x;\bbF_{p,q}) =   2^{p+\tilde{c}-c_z} \times ( u_0 - \tilde{u}_0) = 1,  \label{eq:case3pqr_eq9}
 \end{align}
 with $ \phi_{\theta_{1,1,z}}(x;\bbF_{p,q})  , - \phi_{\theta_{1,2,z}}(x;\bbF_{p,q})  \in \{0 \} \cup [2^p]$.

\textbf{Representability test for Case 1-3.}
(1)  We check that the parameters in \cref{eq:param_psi_1_fpq} are representable by $\bbF_{p,q}$. \\
Since 
\begin{align*}
    &  e_{min} \le p -e_0 \le \tilde{c}-c_z-k \le 1-(-p+e_0) =1+p-e_0 < e_{max}, \\
    & -p + e_{min}  \le 1-p -e_{max} = p - ( -2-p+e_{max})  +(-p+e_0)  \\
    &\le p-\frak{e}_z+k = p - \frak{e}_z +(-p+e_0) = e_0 - \frak{e}_z  < e_{max}, \\
    &e_{min} < p+k = p+(-p+e_0) = e_0  < e_{max},
\end{align*}
$2^{\tilde{c}-c_z-k}$, $2^{p-\frak{e}_z+k}$, $(2-\frak{a}_z-\tilde{u}_0) \times 2^{p+k}$, and $(2-\frak{a}_z-u_0) \times 2^{p+k}$ are  representable by $\bbF_{p,q}$. Apparently, $(2-u_0) \times 2^{\frak{e}_z}$ is representable by $\bbF_{p,q}$. \\

(2) We check that the intermediate numbers in \cref{eq:nums_psi_1_fpq} are representable  by $\bbF_{p,q}$, which is apparent by 

\cref{eq:case3pqr_eq1,eq:case3pqr_eq2,eq:case3pqr_eq3,eq:case3pqr_eq4,eq:case3pqr_eq5,eq:case3pqr_eq6,eq:case3pqr_eq7,eq:case3pqr_eq8,eq:case3pqr_eq9}.
Hence, we conclude 
\begin{align}
    f_1(x;z) \begin{cases}
        = 0\; &\text{if} \; - \Omega_2   \le  x \le z^- \\ 
        = 1  \ \; &\text{if} \; z \le x \le \Omega \\ 
    \end{cases}
\end{align}
Finally, combining \textbf{Case 1-1}, \textbf{Case 1-2}, and \textbf{Case 1-3},
since $\Omega_1 \ge (2-u) \times 2^{-2 + e_0  }$  and $\Omega_2 \ge (2-u) \times 2^{2+e_{0}}$, we conclude that 
\begin{align}
    f_1(x;z) \begin{cases}
        = 0 \; &\text{if} \; -(2 - u) \times 2^{ -2+ e_0  }  \le  x  \le z^- \\ 
        = 1  \ \; &\text{if} \; z \le x \le \Omega, \\ 
    \end{cases}
\end{align}
for all $z$ satisfying $0 <z \le (2-u) \times 2^{-2-p+e_{max}}$.

To construct  $\bold{1}[x \le z]$, we define a three-layer $\relu$ network $f_{2}(x;\bbF_{p,q})$ as follows:
\begin{align}
    f_{2}(x;\bbF_{p,q}) = \psi_{\theta_{2,1,z}}(x;\bbF_{p,q}) \oplus \psi_{\theta_{2,2,z}}(x;\bbF_{p,q}), \label{eq:f2fpqr}
\end{align}
where 
\begin{align}
     & \psi_{\theta_{2,1,z}}(x;\bbF_{p,q}) = \phi_{\theta_{2,1,z}}(x;\bbF_{p,q}), \; \psi_{\theta_{2,2,z}}(x;\bbF_{p,q}) = \phi_{\theta_{2,2,z}}(x;\bbF_{p,q}) \nonumber \\ 
     & \phi_{\theta_{2,1},z}(x;\bbF_{p,q}) = 2^{ -c_z -k} \otimes  \relu \big( (-2^{p - \frak{e}_z + k} \otimes \relu( g_2(x) ) ) \oplus  (\frak{a}_z -1 +u_0) \times 2^{p+k}   \big),  \nonumber \\ 
     & \phi_{\theta_{2,2},z}(x;\bbF_{p,q}) =  - 2^{- c_z -k} \otimes  \relu \big( (-2^{p-\frak{e}_z+k} \otimes \relu(g_2(x) ) )  \oplus  (\frak{a}_z -1 ) \times 2^{p+k}   \big). \nonumber  
\end{align}
Here, $g_2(x)$ is defined as 
$$ g_2(x) =  x \ominus 2^{\frak{e}_z}. $$
Let 
\begin{align*}
     &\zeta_{2,1}(x):= (-2^{p - \frak{e}_z + k} \otimes \relu( g_2(x) ) ) \oplus (\frak{a}_z -1 +u_0) \times 2^{p+k} ,   \\
     &\zeta_{2,2}(x): = ( -2^{p-\frak{e}_z+k} \otimes \relu(g_2(x) ) ) \oplus (\frak{a}_z -1 ) \times 2^{p+k}, \\
     &\phi_{\theta_{2,1,z}}(x;\bbF_{p,q}) = 2^{ -c_z -k} \otimes \relu(\zeta_{2,1}(x)),\\
    &\phi_{\theta_{2,2,z}}(x;\bbF_{p,q}) = -2^{ -c_z -k} \otimes  \relu(\zeta_{2,2}(x)).  
\end{align*}
Note that (1) the following are parameters in $\phi_{\theta_{2,1,z}}(x;\bbF_{p,q})$ and $\phi_{\theta_{2,2,z}}(x;\bbF_{p,q})$:
\begin{align}
    2^{-c_z-k},2^{p-\frak{e}_z+k},  2^{\frak{e}_z} ,(\frak{a}_z -1 +u_0) \times 2^{p+k}, (\frak{a}_z -1 ) \times 2^{p+k}
    \label{eq:param_psi_2_fpq}
\end{align}
The following are (2) the numbers occurring during the intermediate calculations in $\phi_{\theta_{2,1,z}}(x;\bbF_{p,q})$ and $\phi_{\theta_{2,2,z}}(x;\bbF_{p,q})$:
\begin{align}
    g_2(x), -2^{p-\frak{e}_z+k} \otimes g_2(x) , \zeta_{2,1}(x), \zeta_{2,2}(x), \phi_{\theta_{2,1,z}}(x;\bbF_{p,q}), \phi_{\theta_{2,2,z}}(x;\bbF_{p,q}).  \label{eq:nums_psi_2_fpq}
\end{align}
We consider the following cases.

\paragraph{\textbf{Case 2-1: $ -p + e_{min} \le \frak{e}_z \le 0 $. }} \\
In this case, we have $k = 3-p-e_0-c_z$.  

If $-\Omega \le x \le 2^{\frak{e}_z}$, we have $g_2(x) = 0$. Hence we have
\begin{align}
     &\zeta_{2,1}(x)=  (\frak{a}_z -1 +u_0) \times 2^{p+k} , \; \zeta_{2,2}(x) =  (\frak{a}_z -1 ) \times 2^{p+k}, \label{eq:case21pqr_eq1} \\
     &\phi_{\theta_{2,1,z}}(x;\bbF_{p,q}) = 2^{ -c_z -k} \otimes \relu((\frak{a}_z -1 +u_0) \times 2^{p+k}) = (\frak{a}_z -1 + u_0) \times 2^{p-c_z}, \label{eq:case21pqr_eq2} \\
    &\phi_{\theta_{2,2,z}}(x;\bbF_{p,q}) = -2^{ -c_z -k} \otimes  \relu((\frak{a}_z -1 ) \times 2^{p+k}) = - (\frak{a}_z -1 ) \times 2^{p-c_z} , \label{eq:case21pqr_eq3} \\
 &f_{2}(x;\bbF_{p,q})  = \phi_{\theta_{2,1,z}}(x;\bbF_{p,q}) \oplus \phi_{\theta_{2,2,z}}(x;\bbF_{p,q})  =   2^{ -c_z -k} \times  u_0  = 1, \label{eq:case21pqr_eq4}
\end{align}
 with $  \phi_{\theta_{2,1,z}}(x;\bbF_{p,q})  ,  -\phi_{\theta_{2,2,z}}(x;\bbF_{p,q})  \in \{0 \} \cup [2^p]$.

If $ (1+u_0) \times 2^{\frak{e}_z} \le x \le (2-u_0) \times 2^{\frak{e}_z}$,  $x \ominus 2^{\frak{e}_z}$ is exact by \cref{lem:sterbenz}. Then $g_2(x)=x - 2^{\frak{e}_z} = n_2 \times 2^{-p+\frak{e}_z+c_z}$ where $n_2 \in [2^{p-c_z}-1]$, which leads to 
\begin{align}
    &0 < 2^{p-\frak{e}_z+k} \otimes \relu(g_2(x)) \le (2-u_0) \times 2^{-1+p+k}, \label{eq:case21pqr_eq5}\\
    &\mu(2^{p-\frak{e}_z+k} \otimes \relu(g_2(x))) \ge k. \label{eq:case21pqr_eq6}
\end{align}
Since 
\begin{align}
   0 \le  (\frak{a}_z -1 ) \times 2^{p+k} , (\frak{a}_z -1 +u_0) \times 2^{p+k} < 2^{1+p+k}, \label{eq:case21pqr_eq7}
\end{align}
by \cref{lem:exact}, all operations in $\phi_{\theta_{2,1,z}}(x;\bbF_{p,q})$ and $\phi_{\theta_{2,2,z}}(x;\bbF_{p,q})$ are exact. Hence, 
 \begin{align}
     &\phi_{\theta_{2,1,z}}(x;\bbF_{p,q}) = 2^{-c_z-k} \otimes \relu \big( -2^{p-\frak{e}_z+k} \times x + (\frak{a}_z + u_0) \times 2^{p+k} \big), \label{eq:case21pqr_eq8} \\
     &\phi_{\theta_{2,2,z}}(x;\bbF_{p,q}) = - 2^{-c_z-k} \otimes \relu \big( -2^{p-\frak{e}_z+k} \times x + \frak{a}_z  \times 2^{p+k} \big), \label{eq:case21pqr_eq9} 
 \end{align}
with $\phi_{\theta_{2,1,z}}(x;\bbF_{p,q}),-\phi_{\theta_{2,2,z}}(x;\bbF_{p,q}) \in \{0\} \cup [2^p]$. Therefore, we have 
\begin{align*}
    f_{2}(x;\bbF_{p,q}) = \begin{cases}
     1 \; &\text{if} \; (1+u_0) \times  2^{\frak{e}_z} \le x \le z, \\ 
    0 \; &\text{if} \; z^+ \le x \le (2-u_0) \times 2^{\frak{e}_z}.
    \end{cases}
\end{align*} 
Suppose $ 2^{1+\frak{e}_z}  \le x \le \Omega_1$.
 If $ 2^{1+\frak{e}_z} \le x <  2^{2+p+\frak{e}_z}$,  we have $ 2^{\frak{e}_z} \le x \ominus 2^{\frak{e}_z} < 2^{2+p+\frak{e}_z} $. 
If  $2^{2+p+\frak{e}_z}  \le x \le \Omega_1$, we have $x \ominus 2^{\frak{e}_z} = x $ by Lemma \ref{lem:ignore}. Therefore, we have 
\begin{align}
    &  2^{\frak{e}_z} \le g_2(x) \le \Omega_1, \label{eq:case21pqr_eq10}
\end{align}
By \cref{lem:exact}, $-2^{p+k} \oplus (\frak{a}_z -1 +u_0) \times 2^{p+k}$ and $-2^{p+k} \oplus (\frak{a}_z -1) \times 2^{p+k}$ are exact. Together with \cref{eq:omega} and \cref{eq:case21pqr_eq11}, we have 
\begin{align}
    &  2^{\frak{e}_z} \le g_2(x) \le \Omega_1, \nonumber \\
    & -(2-u) \times  2^{4+e_0} \le -2^{p-\frak{e}_z +k} \otimes g_2(x) \le - 2^{p+k} , \label{eq:case21pqr_eq11} \\
    & -(2-u) \times  2^{4+e_0} \le \zeta_{2,1}(x) \le - (2-\frak{a}_z - u_0) \times 2^{p+k} \le 0 , \label{eq:case21pqr_eq12} \\
    & -(2-u) \times  2^{4+e_0} \le \zeta_{2,2}(x) \le - (2-\frak{a}_z ) \times 2^{p+k} < 0 , \label{eq:case21pqr_eq13} 
\end{align}
which leads to $\phi_{\theta_{2,1,z}}(x;\bbF_{p,q})=\phi_{\theta_{2,2,z}}(x;\bbF_{p,q})=0$.

\textbf{Representability test for Case 2-1.} \\
(1)  We check that the parameters in \cref{eq:param_psi_2_fpq} are representable by $\bbF_{p,q}$. \\
Since $2^{-c_z-k},2^{p-\frak{e}_z}$ are shown to be representable by $\bbF_{p,q}$ in \textbf{Case 1-1} and \textbf{Case 1-2} and $2^{\frak{e}_z}$ is obvious, we only need to check $(\frak{a}_z -1 +u_0) \times 2^{p+k}$ and $(\frak{a}_z -1) \times 2^{p+k}$. 
Since 
\begin{align*}
    &  e_{min} =- 2^{q-1} + 2 \le 4-p-2^{q-2} \le 3-p-e_0  \le p+k = 3-e_0-c_z < e_{max}, \\ 
    & e_{min} < 3-p-e_0 = (-p+c_z) + ( p +k ) \le  \mu(t)  \le \frak{e}_{t} \le p+k < e_{max}, \\
    & \frak{e}_{t} - \mu( t ) \le (p+k) - ( -(p-c_z) + (p+k) ) = p-c_z,  
\end{align*}
where $\frak{a}_x \ne 1$ for $t =(\frak{a}_z -1 +u_0) \times 2^{p+k} ,(\frak{a}_z -1) \times 2^{p+k}$, they are representable by $\bbF_{p,q}$. If $\frak{a}_x = 1$, it is obvious. 

(2) We check that the intermediate numbers in \cref{eq:nums_psi_2_fpq} are representable  by $\bbF_{p,q}$, which is apparent by \cref{eq:case21pqr_eq1,eq:case21pqr_eq2,eq:case21pqr_eq3,eq:case21pqr_eq4,eq:case21pqr_eq5,eq:case21pqr_eq6,eq:case21pqr_eq7,eq:case21pqr_eq8,eq:case21pqr_eq9,eq:case21pqr_eq10,eq:case21pqr_eq11,eq:case21pqr_eq12,eq:case21pqr_eq13}.
Therefore, we conclude  
\begin{align}
    f_2(x;\bbF_{p,q}) \begin{cases}
        =  1 \; &\text{if} \;  - \Omega \le x \le z, \\ 
        =  0  \ \; &\text{if} \; z^+ \le x \le \Omega_1.
    \end{cases}
\end{align}

\paragraph{\textbf{Case 2-2: $0 < \frak{e}_z \le -2-p+e_{max} $. }} \\
In this case, $c_z = 0 , u_0 = u$, and $k=-p+e_0$. 

If $ -\Omega  \le x \le 2^{\frak{e}_z}$, we have $g_2(x) = 0$. Similar to \textbf{Case 2-1}, we have
\begin{align}
     &\zeta_{2,1}(x)=  (\frak{a}_z -1 +u) \times 2^{p+k} , \; \zeta_{2,2}(x) =  (\frak{a}_z -1 ) \times 2^{p+k}, \label{eq:case22pqr_eq1} \\
     &\phi_{\theta_{2,1,z}}(x;\bbF_{p,q}) = 2^{ -c_z -k} \otimes \relu((\frak{a}_z -1 +u) \times 2^{p+k}) = (\frak{a}_z -1 + u) \times 2^{p-c_z}, \label{eq:case22pqr_eq2} \\
    &\phi_{\theta_{2,2,z}}(x;\bbF_{p,q}) = -2^{ -c_z -k} \otimes  \relu((\frak{a}_z -1 ) \times 2^{p+k}) = - (\frak{a}_z -1 ) \times 2^{p-c_z} , \label{eq:case22pqr_eq3} \\
 &f_{2}(x;\bbF_{p,q})  = \phi_{\theta_{2,1,z}}(x;\bbF_{p,q}) \oplus \phi_{\theta_{2,2,z}}(x;\bbF_{p,q})  =   2^{ -c_z -k} \times  u  = 1, \label{eq:case22pqr_eq4}
\end{align}
 with $  \phi_{\theta_{2,1,z}}(x;\bbF_{p,q})  , - \phi_{\theta_{2,2,z}}(x;\bbF_{p,q})  \in \{0 \} \cup [2^p]$.

 If $ (1+u) \times 2^{\frak{e}_z} \le x \le (2-u) \times 2^{\frak{e}_z}$, by similar arguments in \textbf{Case 2-1} with $e_{min} \le p+k \le e_{max}$, we have the same results as those in 
 \cref{eq:case21pqr_eq5,eq:case21pqr_eq6,eq:case21pqr_eq7} from \textbf{Case 2-1}. By \cref{lem:exact}, all operations in $\phi_{\theta_{2,1,z}}(x;\bbF_{p,q})$ and $\phi_{\theta_{2,2,z}}(x;\bbF_{p,q})$ are exact. 
Hence, 
 \begin{align}
     &\phi_{\theta_{2,1,z}}(x;\bbF_{p,q}) = 2^{-c_z-k} \otimes \relu \big( -2^{p-\frak{e}_z+k} \times x + (\frak{a}_z + u) \times 2^{p+k} \big), \label{eq:case22pqr_eq5} \\
     &\phi_{\theta_{2,2,z}}(x;\bbF_{p,q}) = - 2^{-c_z-k} \otimes \relu \big( -2^{p-\frak{e}_z+k} \times x + \frak{a}_z  \times 2^{p+k} \big), \label{eq:case22pqr_eq6} 
 \end{align}
with $|\phi_{\theta_{2,1,z}}(x;\bbF_{p,q})|,|\phi_{\theta_{2,2,z}}(x;\bbF_{p,q})| \in \{0\} \cup [2^p]$. Therefore, we have 
\begin{align*}
    f_{2}(x;\bbF_{p,q}) = \begin{cases}
     1 \; &\text{if} \; (1+u) \times  2^{\frak{e}_z} \le x \le z, \\ 
    0 \; &\text{if} \; z^+ \le x \le (2-u) \times 2^{\frak{e}_z}.
    \end{cases}
\end{align*}

 \begin{align}
     &\phi_{\theta_{1,1,z}}(x;\bbF_{p,q}) = 2^{\tilde{c}-c_z-k} \otimes ( ( 2 -\frak{a}_z - \tilde{u}_0 ) \times 2^{p+k } ) \label{eq:case22pqr_eq7} \\
     & \phi_{\theta_{1,2,z}}(x;\bbF_{p,q}) =  - 2^{\tilde{c}-c_z-k} \otimes ( ( 2 -\frak{a}_z - u ) \times 2^{p+k } ) \label{eq:case22pqr_eq8} \\
     &f_{1}(x;\bbF_{p,q}) =  \phi_{\theta_{1,1,z}}(x;\bbF_{p,q}) \oplus \phi_{\theta_{1,2,z}}(x;\bbF_{p,q}) =   2^{p+\tilde{c}-c_z} \times ( u - \tilde{u}_0) = 1,  \label{eq:case22pqr_eq9}
 \end{align}
 with $  \phi_{\theta_{1,1,z}}(x;\bbF_{p,q})  ,- \phi_{\theta_{1,2,z}}(x;\bbF_{p,q})  \in \{0 \} \cup [2^p]$.

Suppose $ 2^{1+\frak{e}_z}  \le x \le \Omega_2$.
If $ 2^{1+\frak{e}_z} \le x <  \min \{ \Omega_2,2^{2+p+\frak{e}_z} \}$, we have $ 2^{\frak{e}_z} \le x \ominus 2^{\frak{e}_z} < \min \{ \Omega_2,2^{2+p+\frak{e}_z} \}   $. 
If  $\min \{ \Omega_2,2^{2+p+\frak{e}_z} \} \le x \le \Omega_2$, we have $x \ominus 2^{\frak{e}_z} = x $ by Lemma \ref{lem:ignore}. Therefore, we have 
\begin{align}
    &  2^{\frak{e}_z} \le g_2(x) \le \Omega_2, \label{eq:case22pqr_eq10}
\end{align}
By \cref{lem:exact}, $-2^{p+k} \oplus (\frak{a}_z -1 +u) \times 2^{p+k}$ and $-2^{p+k} \oplus (\frak{a}_z -1) \times 2^{p+k}$ are exact. Together with \cref{eq:omega_2} and \cref{eq:case22pqr_eq10}, we have 
\begin{align}
    &  2^{\frak{e}_z} \le g_2(x) \le \Omega_2, \nonumber \\
    & -(2-u) \times  2^{e_0} \le -2^{p-\frak{e}_z +k} \otimes g_2(x) \le - 2^{p+k} , \label{eq:case22pqr_eq11} \\
    & -(2-u) \times  2^{e_0} \le \zeta_{2,1}(x) \le - (2-\frak{a}_z - u) \times 2^{p+k} \le 0 , \label{eq:case22pqr_eq12} \\
    & -(2-u) \times  2^{e_0} \le \zeta_{2,2}(x) \le - (2-\frak{a}_z ) \times 2^{p+k} < 0 , \label{eq:case22pqr_eq13} 
\end{align}
which leads to $\phi_{\theta_{2,1,z}}(x;\bbF_{p,q})=\phi_{\theta_{2,2,z}}(x;\bbF_{p,q})=0$.

\textbf{Representability test for Case 2-2.} \\
(1)  We check that the parameters in \cref{eq:param_psi_2_fpq} are representable by $\bbF_{p,q}$. \\
Since $2^{-c_z-k},2^{p-\frak{e}_z}$ are shown to be representable by $\bbF_{p,q}$ in \textbf{Case 1-3} and $2^{\frak{e}_z}$ is obvious, we only need to check $(\frak{a}_z -1 +u) \times 2^{p+k}$ and $(\frak{a}_z -1) \times 2^{p+k}$. 
Since 
\begin{align*}
    &  e_{min} < p+k = e_0 < e_{max}, \\ 
    & e_{min} < -p + e_0 = (-p+c_z) + ( p +k ) \le  \mu(t)  \le \frak{e}_{t} \le p+k < e_{max}, \\
    & \frak{e}_{t} - \mu( t ) \le (p+k) - ( -(p-c_z) + (p+k) ) = p-c_z,  
\end{align*}
where $\frak{a}_x \ne 1$ for $t =(\frak{a}_z -1 +u) \times 2^{p+k} ,(\frak{a}_z -1) \times 2^{p+k}$, they are representable by $\bbF_{p,q}$. If $\frak{a}_x = 1$, it is obvious. 

(2) We check that the intermediate numbers in \cref{eq:nums_psi_2_fpq} are representable  by $\bbF_{p,q}$, which is apparent by \cref{eq:case22pqr_eq1,eq:case22pqr_eq2,eq:case22pqr_eq3,eq:case22pqr_eq4,eq:case22pqr_eq5,eq:case22pqr_eq6,eq:case22pqr_eq7,eq:case22pqr_eq8,eq:case22pqr_eq9,eq:case22pqr_eq10,eq:case22pqr_eq11,eq:case22pqr_eq12,eq:case22pqr_eq13}.
Therefore, we conclude  
\begin{align}
    f_2(x;\bbF_{p,q}) \begin{cases}
        =  1 \; &\text{if} \;  - \Omega \le x \le z, \\ 
        =  0  \ \; &\text{if} \; z^+ \le x \le \Omega_2.
    \end{cases}
\end{align}
Finally, combining \textbf{Case 2-1} and \textbf{Case 2-2},  since $\Omega_1 \ge (2-u) \times 2^{-2 + e_0  }$  and $\Omega_2 \ge (2-u) \times 2^{2+e_{0}}$, we conclude that
\begin{align}
    f_2(x;\bbF_{p,q}) \begin{cases}
        =  1 \; &\text{if} \;  - \Omega \le x \le z \\ 
        =  0  \ \; &\text{if} \; z^+ \le x \le (2- u) \times 2^{-2+e_0},
    \end{cases}
\end{align}
for all $z$ satisfying $0 < z \le (2-u) \times 2^{-2-p+e_{max}}$.

Now, we consider $z < 0$, we define
\begin{align*}
&f_1(x;\bbF_{p,q})=\psi_{1,1,z}(x;\bbF_{p,q}) \oplus \psi_{1,2,z}(x;\bbF_{p,q}), \; f_2(x;\bbF_{p,q})=\psi_{2,1,z}(x;\bbF_{p,q}) \oplus \psi_{2,2,z}(x;\bbF_{p,q}), \\
&\psi_{1,1,z}(x;\bbF_{p,q}) = \phi_{2,1,-z}(-x;\bbF_{p,q}), \;\psi_{1,2,z}(x;\bbF_{p,q}) = \phi_{2,2,-z}(-x;\bbF_{p,q}), \\
&\psi_{2,1,z}(x;\bbF_{p,q}) = \phi_{1,1,-z}(-x;\bbF_{p,q}), \; \psi_{2,2,z}(x;\bbF_{p,q}) = \phi_{1,2,-z}(-x;\bbF_{p,q}).
\end{align*}
Then we have 
\begin{align*}
  &f_1(x;\bbF_{p,q})=\phi_{\theta_{2,1,-z}}(-x;\bbF_{p,q}) \oplus  \phi_{\theta_{2,2,-z}}(-x;\bbF_{p,q}) =\begin{cases}
         0 \; &\text{if} \; -(2 - u) \times 2^{ -2+ e_0  }  \le  x  \le z^- \\ 
        1  \ \; &\text{if} \; z \le x \le \Omega, \\ 
    \end{cases}  , \\
&f_2(x;\bbF_{p,q})=\phi_{\theta_{1,1,-z}}(-x;\bbF_{p,q}) \oplus  \phi_{\theta_{1,2,-z}}(-x;\bbF_{p,q}) =  \begin{cases}
          1 \; &\text{if} \;  - \Omega \le x \le z, \\ 
      0  \ \; &\text{if} \; z^+ \le x \le (2- u) \times 2^{-2+e_0},
    \end{cases}
\end{align*}
for all $z$ satisfying $-(2-u) \times 2^{-2-p+e_{max}} \le z  < 0$.

Finally, we have $\psi_{\theta_{i,1,z}}(x;\bbF_{p,q}) . -\psi_{\theta_{i,2,z}}(x;\bbF_{p,q})\in \{0 \} \cup [2^p]$ for all $i\in\{1,2\}$, and
\begin{align*}
    &f_1(x;\bbF_{p,q}) = \begin{cases}
         0 \; &\text{if} \; -(2 - u) \times 2^{ -2+ e_0  }  \le  x  \le z^- \\ 
         1  \ \; &\text{if} \; z \le x \le \Omega, \\ 
    \end{cases} \\
    &f_2(x;\bbF_{p,q}) = \begin{cases}
          1 \; &\text{if} \;  - \Omega \le x \le z, \\ 
          0  \ \; &\text{if} \; z^+ \le x \le (2- u) \times 2^{-2+e_0},
    \end{cases}
\end{align*}
for all $z$ satisfying $0 < |z| \le (2-u) \times 2^{-2-p+e_{max}}$.

  Lastly, it is easy to observe that $\psi_{\theta_{i,j,z}}$ for all $i,j\in\{1,2\}$ and $z\in\bbF_{p,q}\setminus\{0\}$ share the same network architecture of $3$ layers and $5$ parameters. This completes the proof.
\end{proof}

\subsection{Proof of \cref{thm:memstepfpqr-temp}}\label{sec:pfthm:memstepfpqr-temp}
The proof is almost identical to \cref{thm:memstepfp-temp}. 

We define $f_{\theta_{\mathcal{D}}}$ as follows. 
\begin{align*}
f_{\theta_{\mathcal{D}}}(\bfx;\bbF_{p,q})&=\bigoplus_{i=1}^n\left(y_i\otimes f_{\theta_{\bfz_i,\bfz_i}}(\bfx, \bbF_{p,q})\right),\\
    f_{\theta_{\alpha,\beta}}(\bfx, \bbF_{p,q}) &= \indc{ 
    \bigoplus_{i=1}^{2d} \left( g_i(x_{\ceilZ{i/2}} ; \bbF_{p,q}) \right) \ominus 2d 
    } , \; \forall \alpha , \beta \in \bbF_{p,q}^d, \\
    g_{2j-1}(x_j ; \bbF_{p,q}) &= \indc{-x_j+\beta_j \ge 0}, \; g_{2j}(x_j,\bbF_{p,q}) = \indc{x_j \ominus \alpha_j \ge 0} \; \forall j \in [d]. 
\end{align*}
From the definition of $g_i$, for $ i \in [2d]$, one can observe that 
$$g_{2j-1}(x_j;\bbF_{p,q})+g_{2j}(x_j;\bbF_{p,q})=\begin{cases}
    2~&\text{if}~x_j\in[\alpha_j,\beta_j],\\
    1~&\text{if}~x_j\notin[\alpha_j,\beta_j].
\end{cases}$$ %
Since $\bigoplus_{i=1}^{m}g_i(x_{\ceilZ{i/2}};\bbF_{p,q}) \le m  \le 2^{1+p} - 1 $ for $1 \le m \le (2d - 1) $, and $|g_i(x_{\ceilZ{i/2}};\bbF_{p,q})| = 0,1$, by \cref{lem:integerexact}, $\bigoplus_{i=1}^{2d}g_i(x_{\ceilZ{i/2}};\bbF_{p,q})$ is exact.
Therefore, we have $\bigoplus_{i=1}^{2d}g_i(x_{\ceilZ{i/2}};\bbF_{p,q})=2d$ if $\bfx\in\prod_{i=1}^d[\alpha_i,\beta_i]$ and $\bigoplus_{i=1}^{2d}g_i(x_{\ceilZ{i/2}};\bbF_{p,q})<2d$ otherwise, i.e.,
\begin{align*}
    &f_{\theta_{\alpha,\beta}}(\bfx,\bbF_{p,q})=\indc{\bfx\in\prod_{j=1}^d[\alpha_j,\beta_j]}, \\
    &f_{\theta_{\bfz_i,\bfz_i}}(\bfx,\bbF_{p,q})=\indc{\bfx\in\prod_{j=1}^d[\bfz_{i,j},\bfz_{i,j}]} = \indc{\bfx=\bfz_i}.   \\
\end{align*}

Then, $f_{\theta_{\mathcal{D}}}(\bfz_i;\bbF_{p,q})=y_i$ for all $i\in[n]$ and $f$ can be implemented by a three-layer $\step$ network of $6dn+2n$ parameters ($4dn$ parameters for the first layer, $2dn+n$ parameters for the second layer, and $n$ parameters for the last layer).

\subsection{Proof of \cref{thm:univstepfpqr-temp}}\label{sec:pfthm:univstepfpqr-temp}

The proof is almost identical to \cref{thm:univstepfp-temp}. 

If $\omega_{f^\ast}^{-1}(\varepsilon) \ge \eta$, for $\bfx,\bfx' \in \bbF_{p,q}$ with $\| \bfx - \bfx'\|_\infty \le  \delta= \omega_{f^\ast}^{-1}(\varepsilon)$, we have $ | f^*(\bfx)- f^*(\bfx') | \le \varepsilon$.  
If $\eta > \omega_{f^\ast}^{-1}(\varepsilon) $, for $\bfx,\bfx' \in \bbF_{p,q}$ with $\| \bfx - \bfx'\|_\infty \le  \delta < \eta$, we have $ | f^*(\bfx)- f^*(\bfx') | = 0$ since $\bfx = \bfx'$. Hence for  $\bfx,\bfx' \in \bbF_{p,q}$ with $\| \bfx - \bfx'\|_\infty \le  \delta$,  we have
\begin{align}
    | f^*(\bfx)- f^*(\bfx') | \le \varepsilon. \label{eq:modulus_fpq}
\end{align}

Now, for each $i\in[K]$, we define
\begin{align*}
    \alpha_i&=\begin{cases}
        i\delta~&\text{if}~i\in\{0,1,\dots,K-1\},\\
        1~&\text{if}~i=K,
    \end{cases}\\
    \mathcal I_i&=\begin{cases}
        [\alpha_{i-1}^{(\ge,)},\alpha_i^{(<,\bbF_{p,q})}]\cap\bbF_{p,q}~&\text{if}~i\in[K-1],\\
        [\alpha_{K-1}^{(\ge,\bbF_{p,q})},\alpha_K^{(\le,\bbF_{p,q})}]\cap\bbF_{p,q}~&\text{if}~i=K.
    \end{cases}
\end{align*}
Without loss of generality, we assume that $\mathcal I_i\ne\emptyset$ for all $i\in[K]$; otherwise, we remove empty $\mathcal I_j$, decrease $K$, and re-index $\mathcal I_i$ so that $\mathcal I_i$ if non-empty for all $i\in[K]$.
We note that since $0,1\in\bbF_{p,q}$, there is at least one non-empty $\mathcal I_i$ and $K\ge1$. Additionally, unlike in $\bbF_p$, in $\bbF_{p,q}$ $\mathcal{I}_i$ is finite due to the existence of the smallest positive number.
Then, by the above definitions, it holds that $\sup\mathcal I_i-\inf\mathcal I_i\le\delta$, $\mathcal I_{i_1}$ and $\mathcal I_{i_2}$ are disjoint if $i_1\ne i_2$, and $\bigcupdot_{i\in[K]}\mathcal I_i=[0,1]\cap\bbF_{p,q}$.
For each $\boldsymbol{\iota}=(\iota_1,\dots,\iota_d)\in[K]^d$, we also define 
$$\boldsymbol{\gamma}_{\boldsymbol{\iota}}=\argmin_{\bfx\in\mathcal I_{\iota_1}\times\cdots\times\mathcal I_{\iota_d}}|f^*(\bfx)-\round{f^*(\bfx)}|,$$
which is well-defined since $ \mathcal I_{\iota_1}\times\cdots\times\mathcal I_{\iota_d}$ is non-empty and finite.

We are now ready to introduce our $\step$ network construction $f_\theta$:
\begin{align*}
    f_{\theta}(\bold{x} ; \bbF_{p,q} ) =  \bigoplus_{{\boldsymbol{\iota}}\in [K]^d } \round{f^*(\boldsymbol{\gamma}_{\boldsymbol{\iota}})} \otimes \indc{\left(\bigoplus_{j=1}^{2d}h_{{\boldsymbol{\iota}},j}(x_{\ceilZ{j/2}})\right) \ominus d \ge 0}
\end{align*}
where for each $j\in[d]$ and $\boldsymbol{\iota}=(\iota_1,\dots,\iota_d)\in[K]^d$, %
 \begin{align*}
     h_{{\boldsymbol{\iota}},2j-1}(x) &= \begin{cases}
    \indc{ x - \alpha_{j-1}^{(\ge,\bbF_{p,q})} \ge 0  }\; &\text{if} \; \iota_j\in\{2,\dots,K\}, \\
    \indc{ x + \alpha_{0}^{(\ge,\bbF_{p,q})} \ge 0  }\; &\text{if} \; \iota_j=1,
\end{cases}\\
h_{{\boldsymbol{\iota}},2j}(x) &= \begin{cases}
    - \indc{ x - \alpha_{j}^{(\ge,\bbF_{p,q})} \ge 0} \; &\text{if} \; \iota_j\in\{1,\dots,K-1\}, \\
    -\indc{ x  - \alpha_K^{(> , \bbF_{p,q})} \le 0 } \; &\text{if} \; \iota_j=K. 
\end{cases}%
\end{align*}
Since $h_{\boldsymbol{\iota},j}(x)\in\{-1,0,1\}, h_{\boldsymbol{\iota},2j-1}(x)+h_{\boldsymbol{\iota},2j}(x) = \{-1,0,1\} $, we have 
$${\left|\bigoplus_{j=1}^{m}h_{{\boldsymbol{\iota}},j}(x_{\ceilZ{j/2}}) \right| \le d \le 2^p.}$$
for any $1 \le m \le 2d$. 
Therefore, {by \cref{lem:integerexact}}, all operations in the computation of $\bigoplus_{j=1}^{2d}h_{{\boldsymbol{\iota}},j}(x_{\ceilZ{j/2}})$ are exact. 
Furthermore, for each $\bfx \in \bbF_{p,q}^d \cap [0,1]^d$, we have 
\begin{align*}
\bigoplus_{j=1}^{2d}h_{{\boldsymbol{\iota}},j}(x_{\ceilZ{j/2}})\begin{cases}
    =d~&\text{if}~\bfx\in\mathcal I_{\iota_1}\times\cdots\times\mathcal I_{\iota_d},\\
    <d~&\text{if}~\bfx\in([0,1]^d\cap\bbF_{p,q}^d)\setminus(\mathcal I_{\iota_1}\times\cdots\times\mathcal I_{\iota_d}).
\end{cases}
\end{align*}
Since $ \bigcupdot_{\boldsymbol{\iota}\in[K]^d}\mathcal I_{\iota_1}\times\cdots\times\mathcal I_{\iota_d}=\bbF_{p,q}^d \cap [0,1]^d$, we have 
\begin{align}
    f_\theta( \bfx ; \bbF_{p,q} ) =\round{f^*(\boldsymbol{\gamma}_{\boldsymbol{\iota}})}, \; \forall \bfx \in \mathcal I_{\iota_1}\times\cdots\times\mathcal I_{\iota_d}. \nonumber
\end{align}
Hence, for each $\boldsymbol{\iota}\in[K]^d$ and $\bfx \in \mathcal I_{\iota_1}\times\cdots\times\mathcal I_{\iota_d}$,
\begin{align*}
|f_\theta( \bfx ; \bbF_{p,q} ) - f^*(\bold{x}) | &= |\round{f^*(\boldsymbol{\gamma}_{\boldsymbol{\iota}})} - f^*(\boldsymbol{\gamma}_{\boldsymbol{\iota}}) | + |   f^*(\boldsymbol{\gamma}_{\boldsymbol{\iota}}) - f^*(\bold{x}) |\\
&\le |\round{f^*(\bfx)} - f^*(\bfx) | + \varepsilon
\end{align*}
where we use \cref{eq:modulus_fpq} for the above inequality.
Since each $h_{\boldsymbol{\iota},j}$ can be implemented by a $\step$ network of $3$ parameters, 
$\indc{\left(\bigoplus_{j=1}^{2d}h_{{\boldsymbol{\iota}},j}(x_{\ceilZ{j/2}})\right) \ominus d  \ge 0}$ can be implemented using $6d+1$ parameters. 
This implies that our $f_\theta$ can be implemented by a $\step$ network of $3$ layers and $(6d+2)K^d$ parameters.
This completes the proof.

\subsection{Proof of \cref{thm:memrelufpqr-temp}}\label{sec:pfthm:memrelufpqr-temp}

Recall that $\eta = 2^{-p+e_{max}}$ is the smallest positive  number in $\bbF_{p,q}$. 
For each $i\in[n]$, we also define $h_{i,1},\dots,h_{i,4d}$ as follows: for each $j\in[d]$, if $z_{i,j} \ne 0$, 
\begin{align*}
&h_{i,4j-3}=\psi_{\theta_{1,1,t_{i,j,1}}},h_{i,4j-2}=\psi_{\theta_{1,2,t_{i,j,1}}},\\
&h_{i,4j-1}=-\psi_{\theta_{1,1,t_{i,j,2}}},~h_{i,4j}=-\psi_{\theta_{1,2,t_{i,j,2}}},
\end{align*}
where $t_{i,j,1}=z_{i,j}$ and $t_{i,j,2}=z_{i,j}^+$. 
If  $z_{i,j} = 0$,
\begin{align*}
&h_{i,4j-3}=-\psi_{\theta_{2,1,t_{i,j,1}}},~h_{i,4j-2}=-\psi_{\theta_{2,2,t_{i,j,1}}},\\
&h_{i,4j-1}=-\psi_{\theta_{1,1,t_{i,j,2}}},~h_{i,4j}=-\psi_{\theta_{1,2,t_{i,j,2}}},
\end{align*}
where $t_{i,j,1}= - \eta$ and $t_{i,j,2}=\eta$, and 
$\psi_{\theta_{1,1,z}},\psi_{\theta_{1,2,z}}$ are defined in \cref{lem:reluindcfpqr}. 

From \cref{lem:reluindcfpqr}, the signs of $h_{i,4j-3}(x),h_{i,4j-2}(x)$ are different. The signs of $h_{i,4j-1}(x),h_{i,4j}(x)$ are also different. Let $\frak{s}_{4j-3} = \frak{s}_{h_{i,4j-3}(x)}, \frak{s}_{4j-1} = \frak{s}_{h_{i,4j-1}(x)}$. 
Since $|z_{i,j}| < (2-u) \times 2^{-2-p+e_{max}}$ for all $i \in [n], j \in [d]$, by \cref{lem:reluindcfpqr}, we have 
\begin{align*}
    & |h_{i,4j-3}(x)|,|h_{i,4j-2}(x)|, |h_{i,4j-1}(x)|,|h_{i,4j}(x)|  \in\{0\}\cup[2^p] \\
    &\bigoplus_{k=1}^{4} h_{i,4j-4+k} (x)=\begin{cases}\indc{x=z_{i,j}}~&\text{if}~z_{i,j}\ne0,\\
-\indc{x \le -\eta} - \indc{x \ge \eta} = \indc{x = 0} -1 ~&\text{if}~z_{i,j}=0,
\end{cases} 
\end{align*}
for all $x\in\bbF_{p,q}$ with $|x| \le (2-u) \times 2^{-3+2^{q-2}}$.

Let $ 0 \le m < d \le 2^p$. For $m$, suppose $\frak{s}_{4m+1}=1$. Then we have $ h_{i,4m+1}(x)+  h_{i,4m+2}(x) = l_{4m+1} \in \{0,1\}$ and 
\begin{align*}
      -2^{p} < -m \le &\bigoplus_{j=1}^{4m}h_{i,j}(x)  \le m  < 2^p, \;   h_{i,4m+1}(x)  \in\{0\}\cup[2^p], \\
      -2^{p} \le -m +h_{i,4m+1}(x) \le & \bigoplus_{j=1}^{4m+1}h_{i,j}(x)  \le m + h_{i,4m+1}(x) < 2^{p+1}, \;  - h_{i,4m+2}(x) \in\{0\}\cup[2^p], \\
      -2^{p}< -m + l_{4m+1} \le & \bigoplus_{j=1}^{4m+2}h_{i,j}(x) \le m + l_{4m+1} \le 2^{p}. \; 
\end{align*}
By \cref{lem:integerexact}, all above operations in $ \bigoplus_{j=1}^{4m+2}h_{i,j}(x)$ are exact.

Suppose $\frak{s}_{4m+1}=-1$. Then we have $ h_{i,4m+1}(x)+  h_{i,4m+2}(x) = l_{4m+1} \in \{0,-1\}$, and
\begin{align*}
      -2^{p} < -m \le &\bigoplus_{j=1}^{4m}h_{i,j}(x)  \le m  < 2^p, \;   -h_{i,4m+1}(x)  \in\{0\}\cup[2^p], \\
      -2^{1+p} <  -m +h_{i,4m+1}(x) \le & \bigoplus_{j=1}^{4m+1}h_{i,j}(x)  \le m + h_{i,4m+1}(x) \le 2^{p}, \;   h_{i,4m+2}(x) \in\{0\}\cup[2^p], \\
      -2^{p} \le  -m + l_{4m+1} \le & \bigoplus_{j=1}^{4m+2}h_{i,j}(x) \le m + l_{4m+1} < 2^{p}. \; 
\end{align*}
By \cref{lem:integerexact}, all above operations in $ \bigoplus_{j=1}^{4m+2}h_{i,j}(x)$ are exact.

Since $\frak{s}_{4m+3} = -1$, we have  $ h_{i,4m+3}(x)+  h_{i,4m+4}(x) = l_{4m+3} \in \{0,-1\},  l_{4m+1}+ l_{4m+3} \in \{0,-1\} $, and
\begin{align*}
      -2^p \le -m + l_{4m-1}  \le &\bigoplus_{j=1}^{4m+2}h_{i,j}(x) \le 2^p, \;   -h_{i,4m+3}(x)  \in\{0\}\cup[2^p], \\
      -2^{1+p} \le -m + l_{4m-1} + h_{i,4m+3}(x) \le & \bigoplus_{j=1}^{4m+3}h_{i,j}(x) \le 2^p + h_{i,4m+3}(x) , \;   h_{i,4m+4}(x) \in\{0\}\cup[2^p], \\
      -2^{p} \le -m + l_{4m-1} + l_{4m-3} \le & \bigoplus_{j=1}^{4m+4}h_{i,j}(x) < 2^p + l_{4m-3}  \le 2^{p}. \; 
\end{align*}
By \cref{lem:integerexact}, all above operations in $ \bigoplus_{j=1}^{4m+2}h_{i,j}(x)$ are exact.

Therefore we conclude that all operations in $ \bigoplus_{j=1}^{4d}h_{i,j}(x)$ are exact with  $ |\bigoplus_{j=1}^{4d}h_{i,j}(x)| \in \{0\} \cup [2^p]$.

We design the target network $f_\theta$ as follows:
\begin{align}
    & f_\theta(\bold{x} ; \bbF_{p,q} ) =  \bigoplus_{i=1 }^n y_i \otimes \relu \left( \left(\bigoplus_{j=1}^{4d}h_{i,j}(x_{\ceilZ{j/4}}) \right) \oplus b_i \right), \nonumber \end{align}
where $b_i= | j \in [d] : z_{i,j} = 0 | - (d-1)$. 
Since for each $k\in[n]$
\begin{align*}
  \bigoplus_{j=1}^{4d}h_{i,j}(z_{k,\ceilZ{j/4}}) \begin{cases}
      =b_i +1~&\text{if}~\bfz_k=\bfz_i,\\
      \le b_i ~&\text{if}~\bfz_k\ne \bfz_i,
  \end{cases}
\end{align*}
$f_\theta$ memorizes the target dataset. 
Since there are {5} parameters for each $h_{i,j}$, $f_\theta$ has $20dn+2n$ parameters. This completes the proof.

\subsection{Proof of \cref{thm:univrelufpqr-temp}}\label{sec:pfthm:univrelufpqr-temp}

The proof of \cref{thm:univrelufpqr-temp} is almost identical to that of \cref{thm:univstepfpqr-temp}; we define $\mathcal I_i$, $\alpha_i$, and $\boldsymbol{\gamma}_{\boldsymbol{\iota}}$ as in \cref{sec:pfthm:univstepfpqr-temp}.
For each $\boldsymbol{\iota}\in[K]^d$, $j \in [d]$, we also define $h_{\boldsymbol{\iota},1},\dots,h_{\boldsymbol{\iota},4d}$ as follows:

\begin{align*}
    h_{\boldsymbol{\iota},4j-3}&=\begin{cases}
    \psi_{\theta_{1,1,t_{\boldsymbol{\iota},j,1}}}~&\text{if}~\iota_j\in\{2,\dots,K\},\\
    -\psi_{\theta_{2,1,t_{\boldsymbol{\iota},j,1}}} ~&\text{if}~\iota_j=1,
    \end{cases}\\
    h_{\boldsymbol{\iota},4j-2}&=\begin{cases}
    \psi_{\theta_{1,2,t_{\boldsymbol{\iota},j,1}}}~&\text{if}~\iota_j\in\{2,\dots,K\},\\
    -\psi_{\theta_{2,2,t_{\boldsymbol{\iota},j,1}}} ~&\text{if}~\iota_j=1,
    \end{cases} \\
    h_{\boldsymbol{\iota},4j-1}&=-\psi_{\theta_{1,1,t_{\boldsymbol{\iota},j,2}}},\\
    h_{\boldsymbol{\iota},4j}&=-\psi_{\theta_{1,2,t_{\boldsymbol{\iota},j,2}}}, \\    
    t_{\boldsymbol{\iota},j,1}&=\begin{cases}
    -\eta ~&\text{if}~\iota_j=1, \\
    \alpha_{j-1}^{(\ge,\bbF_{p,q})}~&\text{if}~\iota_j\in\{2,\dots,K\},
    \end{cases} \\
    t_{\boldsymbol{\iota},j,2}&=\begin{cases}
    \alpha_{j}^{(\ge,\bbF_{p,q})}~&\text{if}~\iota_j\in\{1,\dots,K-1\},\\
    \alpha_{K}^{(>,\bbF_{p,q})}~&\text{if}~\iota_j=K.
    \end{cases}
\end{align*}
and $\psi_{\theta_{1,1,z}},\psi_{\theta_{1,2,z}}$ are defined in \cref{lem:reluindcfpqr}. 

Namely, we have
\begin{align*}
&h_{\boldsymbol{\iota},4j-3}(x)\oplus h_{\boldsymbol{\iota},4j-2}(x)=\begin{cases}
    -\indc{x \le -\eta}~&\text{if}~\iota_j=1, \\
    \indc{x\ge\alpha_{j-1}^{(\ge,\bbF_{p,q})}}~&\text{if}~\iota_j\in\{2,\dots,K\},
    \end{cases}\\
&h_{\boldsymbol{\iota},4j-1}(x)\oplus h_{\boldsymbol{\iota},4j}(x)=\begin{cases}
    -\indc{x\ge \alpha_{j}^{(\ge,\bbF_{p,q})}}~&\text{if}~\iota_j\in\{1,\dots,K-1\},\\
    -\indc{x\ge \alpha_{K}^{(>,\bbF_{p,q})}}~&\text{if}~\iota_j=K,
    \end{cases} \\
&\bigoplus_{k=1}^{4} h_{\boldsymbol{\iota},4j-4+k} (x)=
\begin{cases}
 \indc{0 \le x < \alpha_{1}^{(\ge, \bbF_{p,q})}    } -1 ~&\text{if}~\iota_j=1, \\ 
\indc{\alpha_{j-1}^{(\ge, \bbF_{p,q})} \le x < \alpha_{j}^{(\ge, \bbF_{p,q})}  }~&\text{if}~\iota_j\in\{2,\dots,K-1\},
\end{cases}
\end{align*}
for all $x\in\bbF_{p,q}$ with $|x| \le (2-u) \times 2^{-3+2^{q-2}}$,
 by \cref{lem:reluindcfpqr,lem:integerexact}.
We design the target network $f$ as follows:
\begin{align*}
f_\theta(\bold{x} ; \bbF_{p,q} ) =  \bigoplus_{{\boldsymbol{\iota}}\in [K]^d } \round{f^*(\boldsymbol{\gamma}_{\boldsymbol{\iota}})} \otimes \indc{\left(\bigoplus_{j=1}^{4d}h_{{\boldsymbol{\iota}},j}(x_{\ceilZ{j/4}})\right) \oplus b_i \ge 0}.
\end{align*}
where $b_i= | j \in [d] : z_{i,j} = 0 | - (d-1)$. 

By similar argument presented in the proof of \cref{thm:memrelufpqr-temp}, all operations in $\bigoplus_{j=1}^{4d}h_{\boldsymbol{\iota},j}(x_{\ceilZ{j/4}})$ are exact by \cref{lem:integerexact}, i.e., for each $k\in[n]$
\begin{align*}
  \bigoplus_{j=1}^{4d}h_{\boldsymbol{\iota},j}(x_{\ceilZ{j/4}})
  \begin{cases}
      =-b_i+1~&\text{if}~\bfx\in\mathcal I_{\iota_1}\times\cdots\times\mathcal I_{\iota_d},\\
      \le -b_i~&\text{if}~([0,1]^d\cap\bbF_{p,q}^d)\setminus(\mathcal I_{\iota_1}\times\cdots\times\mathcal I_{\iota_d}),
  \end{cases}
\end{align*}
This implies that for each $\boldsymbol{\iota}\in[K]^d$ and $\bfx \in \mathcal I_{\iota_1}\times\cdots\times\mathcal I_{\iota_d}$,
\begin{align*}
|f_\theta( \bfx ; \bbF_{p,q} ) - f^*(\bold{x}) | &= |\round{f^*(\boldsymbol{\gamma}_{\boldsymbol{\iota}})} - f^*(\boldsymbol{\gamma}_{\boldsymbol{\iota}}) | + |   f^*(\boldsymbol{\gamma}_{\boldsymbol{\iota}}) - f^*(\bold{x}) |\\
&\le |\round{f^*(\bfx)} - f^*(\bfx) | + \varepsilon
\end{align*}
where we use \cref{eq:modulus_fpq} for the above inequality.
Since each $h_{\boldsymbol{\iota},j}$ can be implemented by a $\relu$ network of $3$ layers and $5$ parameters by \cref{lem:reluindcfpqr}, 
$\indc{\left(\bigoplus_{j=1}^{4d}h_{{\boldsymbol{\iota}},j}(x_{\ceilZ{j/4}})\right) \oplus b_i \ge 0}$ can be implemented using $20d+1$ parameters. 
This implies that our $f$ can be implemented by a $\relu$ network of $4$ layers and $(20d+2)K^d$ parameters.
This completes the proof.

\section{Discussions}
In this section, we discuss possible extensions of our research.
One possible direction is to explore general activation functions, which may require more cautious inspections of floating-point operations.
For example, in our $\relu$ network constructions, we approximate a $\step$ function via a $\relu$ network of size $O(1)$ using the \emph{flat region} of $\relu$, i.e., $\relu(x)=0$ for all $x\le0$. 
However, even most piecewise linear activation functions, such as Leaky-ReLU, do not have such a flat region, i.e., our $\relu$ network constructions do not easily generalize. %

We expect that networks using general smooth functions require more caution.
This is because there are various different ways to implement each function in floating-point numbers.
Consider the exponential function $e^x$, for instance. Mathematically, it can be computed via the Taylor series expansion as $e^x = 1 \oplus x \oplus (x\otimes x \otimes 2^{-1}) \oplus \cdots$. However, the result of floating-point operations can vary depending on their operation order, e.g.,
$((\frac{1}{6} \otimes x) \otimes x) \otimes x \neq \frac{1}{6}\otimes (x\otimes \left( x\otimes x\right))$. 
Additionally, a standard approach to implement the exponential function is to divide the input range into several regions and calculate the ouput differently in each region \cite{muller16}.
Such a complicated method is used to guarantee a small rounding error in the output;
for instance, the GNU C library \cite{loosemore2012gnu} claims that its implementation of various mathematical functions (e.g., $\exp,\log,\sin$) has a maximum error of 10 ulps\footnote{%
The ulp of $x \in \bbR$ is commonly defined as
$\text{ulp}(x)=2^{ \max\{ e_{min}, \lfloor{\log_2 |x|}\rfloor \} -p }$ \cite[Definition 2.6]{muller18},
and the ulp error of $y \in \bbR$ from $x$ is defined as $\text{err}_\text{ulp}(x,y) = |x-y| / \text{ulp}(x)$ \cite{goldberg91}.} for most inputs.
Therefore, when discussing the universal approximation property of a specific activation function in the floating-point setting, we should specify not just the activation itself but also its computation protocol and its guaranteed error bounds.
It would be an interesting research topic to analyze the approximation capacity of each activation computing protocol.

We can also explore another arithmetic approach, integer arithmetic, which is utilized in quantized neural networks. 
Additionally, it will be an interesting research topic to investigate the expressive power of more general network architectures such as residual networks, and convolutional neural networks.

\section{Conclusion}\label{sec:conclusion}

In this work, we investigate the expressive power of neural networks under the fully floating-point setting; all parameters of neural networks, inputs, and intermediate values are floating-point, and neural networks are represented as compositions of floating-point operations. 
Under unbounded exponent floating-point arithmetic  $\bbF_p$, we first demonstrate that $\step$ network has memorization universal approximation properties (\cref{thm:memstepfp-temp,thm:univstepfp-temp}). Following this, we establish that $\relu$ network also exhibits memorization universal approximation properties by constructing the indicator function using $\relu$ network (\cref{thm:memrelufp-temp,thm:univrelufp-temp}). Additionally, under bounded exponent floating-point arithmetic $\bbF_{p,q}$, we also demonstrate that $\step$ and $\relu$ network possess memorization and universal approximation properties (\cref{thm:memstepfpqr-temp,thm:univstepfpqr-temp,thm:memrelufpqr-temp,thm:univrelufpqr-temp}). Because underflow and overflow in $\bbF_{p,q}$ presents challenges in constructing the network, we develope several technical lemmas to address this (\cref{lem:representable,lem:exact,lem:integerexact,lem:ignore,lem:reluindcfpqr}). 

On the other hand, most prior works regarding universal approximation assume real-valued inputs and parameters and/or exact mathematical operations, which cannot be simulated by modern computers that can only represent a tiny subset of the real numbers and apply inexact operations.
Considering that almost all neural networks are indeed implemented under floating-point machinery, it is important to theoretically analyze the properties of floating-point neural networks.
To the best of our knowledge, this is the first work to tackle the universal approximation property under the floating-point setting.
We believe that our results and analyses under floating-point operations would contribute to a better understanding of the performance of modern deep and narrow networks that are executed on actual computers. %

\section*{Acknowledgements}
YP was supported by a KIAS Individual Grant [AP090301] via the Center for AI and Natural Sciences at Korea Institute for Advanced Study. GH was supported by a KIAS Individual Grant [AP092801] via the Center for AI and Natural Sciences at Korea Institute for Advanced Study. 
SP was supported by Institute of Information \& communications Technology Planning \& Evaluation (IITP) grant funded by the Korea government (MSIT) (No. 2019-0-00079, Artificial Intelligence Graduate
School Program, Korea University) and Basic Science Research Program through the National Research Foundation of Korea (NRF) funded by the Ministry of Education (2022R1F1A1076180).

\bibliography{reference}
\bibliographystyle{abbrv}

\end{document}